\documentclass[twocolumn]{svjour3}          

\smartqed  
\pdfoutput=1

\usepackage{amsmath}
\usepackage{amssymb}
\usepackage[ruled,linesnumbered]{algorithm2e}
\usepackage{multirow}
\usepackage{graphicx}
\usepackage{epstopdf}
\usepackage{epsfig}
\usepackage{fix2col}

\begin{document}

\title{ROML: A Robust Feature Correspondence Approach for Matching Objects in A Set of Images}


\author{Kui~Jia \and
        Tsung-Han~Chan \and
        Zinan~Zeng \and
        Shenghua~Gao \\
        Gang~Wang \and
        Tianzhu~Zhang \and
        Yi~Ma
        }


\institute{Kui~Jia \and Tsung-Han~Chan \and Zinan~Zeng \and Shenghua~Gao \and Tianzhu~Zhang \at
              Advanced Digital Sciences Center, 1 Fusionopolis Way, Singapore \\
              \email{kuijia@gmail.com; chantsunghan@gmail.com; edwin.zeng@adsc.com.sg; gaosheeis@gmail.com; tz.zhang@adsc.com.sg}
           \and
           Gang~Wang \at
              School of Electrical and Electronic Engineering, Nanyang Technological University, Singapore \\
              \email{wanggang@ntu.edu.sg}
           \and
           Yi~Ma \at
              Department of Electrical and Computer Engineering, University of Illinois at Urbana-Champaign, Urbana, IL, USA \\
              \email{yima@illinois.edu} \\
              School of Information Science and Technology, ShanghaiTech University, China \\
              \email{mayi@shanghaitech.edu.cn}
           }

\date{}

\maketitle

\begin{abstract}

Feature-based object matching is a fundamental problem for many applications in computer vision, such as object recognition, 3D reconstruction, tracking, and motion segmentation. In this work, we consider simultaneously matching object instances in a set of images, where both inlier and outlier features are extracted. The task is to identify the inlier features and establish their {\it consistent} correspondences across the image set. This is a challenging combinatorial problem, and the problem complexity grows exponentially with the image number. To this end, we propose a novel framework, termed ROML, to address this problem. ROML optimizes simultaneously a partial permutation matrix (PPM) for each image, and feature correspondences are established by the obtained PPMs. Two of our key contributions are summarized as follows. (1) We formulate the problem as rank and sparsity minimization for PPM optimization, and treat simultaneous optimization of multiple PPMs as a regularized consensus problem in the context of distributed optimization. (2) We use the ADMM method to solve the thus formulated ROML problem, in which a subproblem associated with a single PPM optimization appears to be a difficult integer quadratic program (IQP). We prove that under wildly applicable conditions, this IQP is equivalent to a linear sum assignment problem (LSAP), which can be efficiently solved to an exact solution. Extensive experiments on rigid/non-rigid object matching, matching instances of a common object category, and common object localization show the efficacy of our proposed method.


\end{abstract}

\section{Introduction}
\label{IntroSec}

Object matching is a fundamental problem in computer vision. Given a pair or a set of images that contain common object instances, or an object captured under varying poses, it involves establishing correspondences between the parts or features of the objects contained in the images. Accurate, robust, and consistent matching across images is a key ingredient in a wide range of applications such as object recognition, shape matching, 3D reconstruction, tracking, and motion segmentation.

For a pair of feature sets extracted from two images, finding inliers from them and establishing correspondences are in general a combinatorial search problem. Objects may appear in images with cluttered background, and some parts of the objects may also be occluded. The search space can further explode when a {\it globally consistent matching} across a set of images is desired. For object instances with large intra-category variations or those captured under varying poses (e.g., non-rigid objects with articulated pose changes), the matching tasks become even more difficult. All these factors make object matching a very challenging task.

In literature, a variety of strategies have been proposed for object matching. In particular, early shape matching works use point sets to represent object patterns \cite{SLH,ShapiroBrady92}. To match between a pair of point sets, they build point descriptions by modeling spatial relations of points within each point set as higher level geometric structures, e.g., lines, curves, and surfaces, or more advanced features, e.g., shape context \cite{BelongieShapeContext}. In \cite{ICP,ChuiRangarajanNewPointMatching,BelongieShapeContext}, alternating estimation of point correspondence and geometric transformation is also used for non-rigid shape matching. In general, point set based shape matching is less robust to measurement noise and outliers, with classical techniques such as RANSAC \cite{RANSAC} available to improve its robustness. The development of local invariant features \cite{SIFT,InterestPtsDetector} for discriminative description of visual appearance has brought significant progress in object matching and recognition \cite{DescriptorForObjectRecog}. For example in \cite{RussellSegmentationUOCD,ObjectnessForUOCD}, instances of a common object category from an image collection can be located and matched by exploiting the discriminative power of local feature descriptors. The popular Bag-of-Words model for object recognition is also built on matching (clustering) similar local region descriptors. However, local descriptors alone can be ambiguous for matching when there exist repetitive textures or less discriminative local appearance in images. In between of these two extremes, recent graph matching methods \cite{LeordeanuSpectralCorrespondence,TorresaniModelAndGlobalOpt,zhou2012factorized} consider both feature similarity and geometric compatibility between two sets of features, where the nodes of graphs correspond to local features and edges encode spatial relations between them. Mathematically, graph matching is formulated as a quadratic assignment problem (QAP), which is known to be NP-hard. Although intensive efforts of these methods have been focusing on devising more accurate and efficient algorithms to solve this problem, in general, they can only obtain approximate solutions for QAP, and thus suboptimal correspondences for robust object matching. 

Most of these methods focus on establishing correspondences between a pair of images. However, in practice, it is very common that when such a pair of images are available, a set of images are also available that we know a common object is present in them, such as a video sequence with a moving object, or a set of images collected from the Internet that contain instances of a generic object category. In these situations, it is desired that a globally consistent matching can be established. This is a very challenging combinatorial problem. As the number of images increases, the problem complexity explodes exponentially. A straightforward approach is to locally build correspondences between pairs of images. Obviously, pair-wise matching can only get suboptimal solutions, since matching found between pairs of images may not be globally consistent across the whole set. Compared to global matching, pair-wise matching is also less robust to outliers and occlusion of inlier features, as it cannot leverage additional constraints from other images that also contain the same object pattern of interest. In this work, we are thus interested in the following object matching problem.

{\it Problem 1:} Given a set of images with both inlier and outlier features extracted from each image, simultaneously identify a given number of inlier features from each image and establish their consistent correspondences across the image set.

In Problem 1, we consider the common scenario in object matching that there is exactly one object instance in each image. The inlier features describe appearance of the salient local regions of the object, and the rest of the features are outliers. One may think of these features as local region descriptors such as SIFT \cite{SIFT} or HOG \cite{HOG}, although other types of features can also be used, which will be deliberated in later sections. Under this setting, the object instance contained in each image is naturally represented as a set of inlier features. Without consideration of intra-category variations, the corresponding inlier features extracted from different images, which characterize the same salient local regions of different instances of the object category, would be linearly correlated. When concatenating ordered inlier features in each image as a simple feature of the object instance (thinking of concatenating several SIFT feature vectors as a single vector), and then arraying these features of different images as the columns of a large matrix, this matrix will have low-rank, and ideally rank one. In situations where intra-category variations exist, e.g., variations in inlier features of different images caused by illumination or pose changes, the matrix low-rank property can still hold by decomposing out some errors.

Motivated by these observations, we propose in this paper a novel and principled framework, termed {\it Robust Object Matching using Low-rank constraint} (ROML), for identifying and matching inlier features of object instances across a set of images. ROML leverages the aforementioned low-rank property, via minimizing rank of a matrix and sparsity of a matrix (for decomposing out sparse errors), to simultaneously optimize a partial permutation matrix (PPM) for each image, and feature correspondences are established by the obtained PPMs (cf. (\ref{EqnPPM}) for the definition of PPM). The so formulated ROML problem concerns with simultaneous optimization of multiple PPMs, which belongs to a more general class of multi-index assignment problem (MiAP) and is proven to be NP-hard \cite{AssignProbBook}. Exact solution methods are prohibitively slow for practical use. In this work, we treat simultaneous optimization of multiple PPMs involved in ROML as a regularized consensus problem in the context of distributed optimization \cite{BertsekasDistributedOptBook}. We use the Alternating Direction Method of Multipliers (ADMM) \cite{BoydADMM} to solve the ROML problem, in which a subproblem associated with a single PPM optimization appears to be a difficult integer quadratic program (IQP). We prove that under widely applicable conditions, this IQP is equivalent to a linear sum assignment problem (LSAP) \cite{AssignProbBook}, which can be efficiently solved to an exact solution using the Hungarian algorithm \cite{KuhnHungarian}. Extensive experiments on rigid/non-rigid object matching, matching instances of a common object category, and common object localization show the efficacy of our proposed method. A MATLAB implementation of our method and the data used in the experiments can be found from our project website: https://sites.google.com/site/kuijia/research/roml.

A preliminary work of this paper has appeared in \cite{ZinanECCV}. In the present paper, we have made significant improvement over \cite{ZinanECCV} in the following aspects. In addition, we have also completely rewritten the paper to present our ideas more clearly.

\begin{itemize}

\item Although \cite{ZinanECCV} proposes to optimize a set of PPMs via rank and sparsity minimization for robust feature matching, however, its solution of each PPM optimization is obtained by sequentially solving two costly subproblems: a quadratic program over the continuous-domain relaxation of PPM, followed by a binary integer programming that projects the relaxed PPM into its feasible set. In fact, the second subproblem is irrelevant to the original objective function, and consequently, the thus obtained PPM is only suboptimal. In contrast, we propose in the present paper a new method to solve the PPM optimization and prove that under wildly applicable conditions, the PPM optimization step is equivalent to an LSAP, which can be efficiently solved to an exact solution using the Hungarian algorithm. Extensive experiments in Section \ref{ExpObjMatching} show the great advantage of the proposed ROML over the method in \cite{ZinanECCV} in terms of both matching accuracy and efficiency.

\item We present mathematical analysis in this paper to show that the proposed ROML formulation belongs to the NP-hard MiAP. We also discuss the suitability of ADMM for approximately solving ROML from the perspective of distributed optimization. These analysis and discussion put ROML in a broader context, which are overlooked in \cite{ZinanECCV}.



\end{itemize}

\section{Related Works}
\label{LiteratureSec}

There is an intensive literature on object/shape matching between a pair of images \cite{ThirtyYears}. Representative works include shape context \cite{BelongieShapeContext}, graph \cite{LeordeanuSpectralCorrespondence,TorresaniModelAndGlobalOpt} and hyper-graph \cite{HyperGraphReweightedRandomWalk,ZassShashua,TensorHighOrderGraphMatching} matching. In this section, we briefly review several existing methods that use multiple images/point sets for object matching, and also the more general MiAP.

Maciel and Costeira \cite{AGlobalSolution} first proposed to use PPM to model both feature matching and outlier rejection in a set of images. They formulated optimization of PPMs as an integer constrained minimization problem. To solve this combinatorial problem, they relaxed both the objective function and integer constraints, resulting in an equivalent concave minimization problem. However, the complexity of concave minimization is still non-polynomial. Moreover, matching criteria used in the cost function of \cite{AGlobalSolution} were locally based on pair-wise similarity of features in different images. Instead, our method is based on low-rank and sparse minimization (via convex surrogate functions), whose problem size is polynomial w.r.t. the numbers of features and images, and whose cost function is also globally defined over features in all the images.

Rank constraints have been used in \cite{OliveiraRankConstraint,OliveiraAssignmentTensor} for point matching across video frames. They constructed a measurement matrix containing image coordinates of points extracted from a moving rigid object. Motivated by factorization model in shape-from-motion \cite{Rank4FactorShapeFromMotion}, they assumed this measurement matrix was low-rank, and used rank constraints to optimize PPMs for establishing point correspondences across frames. The method in \cite{OliveiraRankConstraint} is limited in several aspects: (1) an initial template of point set without outliers is assumed given; (2) every inlier point is required to be visible in all frames; (3) matching across frames is a bootstrapping process - points in a subsequent frame are to be aligned to those of previously matched frames, thus matching errors will inevitably propagate and accumulate; (4) an initial rough estimate of point correspondences for a new frame is assumed given in their algorithm, which may be only valid for slow motion objects. The aspects (1) and (3) have to some extent been alleviated in \cite{OliveiraAssignmentTensor}, but \cite{OliveiraAssignmentTensor} cannot cope with the other limiting aspects. As a globally consistent and robust matching framework, our method has no such limitations. More importantly, we note that the mechanism of rank constraints used in \cite{OliveiraRankConstraint,OliveiraAssignmentTensor} is different from that of our method. Methods \cite{OliveiraRankConstraint,OliveiraAssignmentTensor} can only apply to matching of rigid objects using image coordinates as features, while our method considers low-rank assumption on a type of generally defined features, which take image coordinates and region descriptors as instances. Consequently, our method is able to apply in more general scenarios, such as matching of objects with non-rigid deformation.

Recently, a low-dimensional embedding method was proposed in \cite{OneShot} for feature matching. Given feature points extracted from each of a set of images, it can learn an embedded feature space, which encodes information of both region descriptors and the geometric structure of points in each image. \cite{OneShot} used k-means clustering in the embedded space for feature matching. As we will show in Section \ref{ExpObjMatching}, k-means based on Euclidean distances of embedded features is not a good way to establish correspondences. There is no explicit outlier rejection mechanism in \cite{OneShot} either. Compared to \cite{OneShot}, our method uses the low-rank and sparse constraints to optimize PPMs, which integrates correspondence and outlier rejection in a single step.

As mentioned in Section \ref{IntroSec}, our ROML formulation for multi-image object matching belongs to a more general class of MiAP for data association \cite{AssignProbBook}, with other vision applications such as multi-target tracking \cite{Collins}. MiAP is proven to be NP-hard, and only implicit enumeration methods such as branch-and-bound are known to give an exact solution, which are however prohibitively slow for practical use. Classical approximate solution methods include greedy, Greedy Randomized Adaptive Search Procedure (GRASP) \cite{GRASP}, and relaxation based methods \cite{NewLagrangianRelaxation}.

Greedy approaches build a matching that has the lowest cost at each iteration, which has the obvious weakness that decisions once made, are fixed and may later be shown to be suboptimal. GRASP improves over greedy approaches by progressively constructing a list of best candidate matches and randomly selecting one from them. The process is repeated until all matches are exhausted. A final local search over the neighborhood of obtained matches may be used to further optimize the solution. In \cite{NewLagrangianRelaxation}, Poore and Robertson presented a Lagrangian relaxation method that progressively relaxes the original and intermediate recovery MiAPs to linear assignment problems, by incorporating constraints of each MiAP into its objective function via the Lagrangian. However, this method involves implicitly enumerative procedure, and is difficult to implement and analyze.

Collins \cite{Collins} recently proposed an iterated conditional modes (ICM) like method for video based multi-target tracking. His method is based on factoring the global decision variable for each target trajectory into a product of local variables defined for a target matching between each pair of adjacent frames. It then pair-wisely builds target matchings between adjacent frames by optimizing the corresponding local variables, but using a global cost function as matching criteria. However, the cost function in \cite{Collins} is defined by enumerating for every possible target trajectory a constant-velocity motion energy, and the number of candidate trajectories grows exponentially with the number of frames. Both factors make it less applicable to the feature-based object matching problem considered in this paper. Nevertheless, our ROML formulation bears some spiritually similar idea with \cite{Collins}, in the sense that we also factor the global decision variables for multi-image feature matching into separate components. The difference is that we factor these global variables as a set of PPMs, each of which is to be optimized to identify inlier features from an image and re-arrange them in a proper order. We then treat the joint optimization of these PPMs as a regularized consensus problem in the context of distributed optimization \cite{BertsekasDistributedOptBook}, and solve it using a ADMM-based method \cite{BoydADMM}. By this means the original NP-hard combinatorial problem boils down as to iteratively solve a set of independent pair-wise matching problems, which turn out to be easily solved.

In the preparation of this paper, we notice that a related method called permutation synchronization \cite{NIPS13MultiWayMatching} was recently proposed, which also addresses the MiAP by optimizing a permutation matrix for each feature set. However, \cite{NIPS13MultiWayMatching} assumes initial matchings between each pair of the feature sets be available, and can only apply in the scenarios where there exist no outliers in each feature set, which make it less useful in the considered problem of feature-based object matching across a set of images.

\section{Robust Object Matching Using Low-Rank and Sparse Constraints}
\label{MainFormSec}

Given a set of $K$ images, we present in this section our problem formulation and algorithm for robust object matching. We consider the settings as stated in Problem 1. Assume $n_{k}$ features $\{ {\bf f}_{i}^{k} \}_{i=1}^{n_{k}}$ be extracted from the $k^{th}$ image, where the feature vector ${\bf f}_{i}^{k} \in \mathbb{R}^{d}$ can be either image coordinates of the feature point, or region descriptors such as SIFT \cite{SIFT} that characterize the local appearance. It can also be a combination of them by low-dimensional embedding \cite{OneShot}. In spite of these multiple choices, for now we generally refer to them as {\it features}. Discussion of different feature types and their applicable spectrums will be presented in Section \ref{FeaDiscussSec}. These $n_{k}$ features are categorized as inliers or outliers. We assume at this moment that there are $n$ inliers in each of the $K$ images, where $n \leq n_{k}$ for $k = 1, \dots, K$. We will discuss the case of missing inliers shortly. In such a setting every $k^{th}$ image is represented as a set of $n_{k}$ features, and the contained object instance is represented as the $n$ inlier features.

\subsection{Problem Formulation}

Note that for inlier features in the $K$ images, it is the feature similarity and geometric compatibility that determine they form an object {\it pattern} and this pattern repeats across the set of images. While similar outlier features may appear in multiple images, they just accidently do so in a random, unstructured way. Our formulation for object matching is essentially motivated by these observations. Concretely, denote $\overline{\bf F}^{k} = [ {\bf f}_{1}^{k}, \dots, {\bf f}_{n}^{k} ] \in \mathbb{R}^{d\times n}$ as the matrix formed by inlier feature vectors in the $k^{th}$ image, so defined are the matrices $\{ \overline{\bf F}^{1}, \dots, \overline{\bf F}^{K} \}$ for all the $K$ images. Assume column vectors in each of these matrices are arrayed in the same order, i.e., inlier features in $\{ \overline{\bf F}^{1}, \dots, \overline{\bf F}^{K} \}$ are respectively corresponded, then the matrix formed by $\overline{\bf D} = [ \mathrm{vec}(\overline{\bf F}^{1}), \dots, \mathrm{vec}(\overline{\bf F}^{K}) ] \in \mathbb{R}^{dn \times K}$ will be approximately low-rank, ideally rank one, where $\mathrm{vec}(\cdot)$ is an operator that vectorizes a matrix by concatenating its column vectors.

Now consider the general case that there are outliers. Denote $ {\bf F}^{k} = [ {\bf f}_{1}^{k}, \dots, {\bf f}_{n_{k}}^{k} ] \in \mathbb{R}^{d\times n_{k} }$ as the matrix having all $n_{k}$ features of the $k^{th}$ image as its columns, where feature vectors are placed in a random order. The matrices $\{ {\bf F}^{1}, \dots, {\bf F}^{K} \}$ for all $K$ images are similarly defined. As aforementioned our interest for object matching is to identify the $n$ inlier feature vectors from each matrix of $\{ {\bf F}^{1}, \dots, {\bf F}^{K} \}$, and establish correspondences among them. For any $k^{th}$ image, this can be realized by the {\it partial permutation matrix} (PPM) defined by
\begin{multline}\label{EqnPPM} {\cal{P}}^{k} = \{ {\bf P}^{k} \in \mathbb{R}^{n_{k}\times n} \big| p_{i,j}^{k} \in \{0, 1\}, \sum_{i}p_{i,j}^{k} = 1 \\  \forall j = 1, \dots, n, \sum_{j}p_{i,j}^{k} \leq 1 \ \forall i = 1, \dots, n_{k}  \} , \end{multline}
where $p_{i,j}^k$ denotes an entry of the PPM $\mathbf{P}^{k}$ at the $i^{th}$ row and $j^{th}$ column. Thus, there exist PPMs $ \{ {\bf P}^{k} \in {\cal{P}}^{k} \}_{k=1}^{K}$ such that inlier feature vectors are selected and corresponded in $\{ {\bf F}^{k}{\bf P}^{k} \in \mathbb{R}^{d\times n} \}_{k=1}^{K}$, i.e., the matrix $[ \mathrm{vec}({\bf F}^{1}{\bf P}^{1}), \dots, \mathrm{vec}({\bf F}^{K}{\bf P}^{K}) ] \in \mathbb{R}^{dn\times K}$ is rank deficient. In the following of this paper, we also use
\begin{equation}\label{EqnDMatrix}
{\bf D}( \{\mathbf{P}^k \}_{k=1}^K ) = [ \mathrm{vec}({\bf F}^{1}{\bf P}^{1}), \dots, \mathrm{vec}({\bf F}^{K}{\bf P}^{K}) ] ,
\end{equation}
to simplify writings of equations. In the context where the values of $\{\mathbf{P}^k \}_{k=1}^K$ are determined, we also write ${\bf D} = [ \mathrm{vec}({\bf F}^{1}{\bf P}^{1}), \dots, \mathrm{vec}({\bf F}^{K}{\bf P}^{K}) ]$. Based on this low-rank assumption, feature correspondence can be formulated as the following problem to optimize $\{{\bf P}^{k} \}_{k=1}^{K}$
\begin{displaymath}
\quad\quad \min_{ \{ {\bf P}^{k} \in {\cal{P}}^{k} \}_{k=1}^{K} } \mathrm{rank}\left( {\bf D}( \{\mathbf{P}^k \}_{k=1}^K ) \right) ,
\end{displaymath}
where $\mathrm{rank}(\cdot)$ is a function to measure matrix rank. By introducing an auxiliary variable $\mathbf{L}$ (to facilitate the development of a solving algorithm), the above problem can also be written as the following equivalent problem
\begin{eqnarray}\label{EqnLowRankForm}
\quad\quad  \min_{ \{ {\bf P}^{k} \in {\cal{P}}^{k} \}_{k=1}^{K} , {\bf L} } \mathrm{rank}({\bf L}) \nonumber \quad\quad \\ \mathrm{s.t.} \ {\bf D}( \{\mathbf{P}^k \}_{k=1}^K ) = {\bf L}  .
\end{eqnarray}

In practice, however, some inlier features characterizing the same local appearance information of object instances in different images could vary due to illumination change, object pose change, or other intra-category object variations. Some inlier features could also be missing due to partial occlusion of object instances. Thus the low-rank assumption used in (\ref{EqnLowRankForm}) cannot be fully satisfied. To improve the robustness, we introduce a sparse error term into (\ref{EqnLowRankForm}) to model all these contaminations, and modify the formulation (\ref{EqnLowRankForm}) as
\begin{eqnarray}\label{EqnLowRankSparseForm}
\quad \min_{ \{ {\bf P}^{k} \in {\cal{P}}^{k} \}_{k=1}^{K} , {\bf L}, {\bf E} } \mathrm{rank}({\bf L}) + \lambda \| {\bf E} \|_{0} \nonumber \\ \mathrm{s.t.} \ {\bf D}( \{\mathbf{P}^k \}_{k=1}^K ) = {\bf L} + {\bf E} ,
\end{eqnarray}
where $\| \cdot \|_{0}$ is $\ell_{0}$-norm counting the number of nonzero entries, and $\lambda > 0$ is a parameter controlling the trade-off between rank of ${\bf L}$ and sparsity of ${\bf E}$.

\subsection{The Algorithm}

The optimization problem (\ref{EqnLowRankSparseForm}) is not directly tractable due to the following aspects: (1) both $\mathrm{rank}(\cdot)$ and $\|\cdot \|_{0}$ are non-convex, discrete-valued functions, minimization of which is NP-hard; (2) entries of $\{ {\bf P}^{k} \}_{k=1}^{K}$ are constrained to be binary, resulting in a difficult nonlinear integer programming problem. To make it tractable, we first consider the recent convention of replacing $\mathrm{rank}(\cdot)$ and $\|\cdot \|_{0}$ with their convex surrogates $\| \cdot \|_{*}$ and $\|\cdot \|_{1}$ respectively \cite{RPCA}, where $\| \cdot \|_{*}$ denotes nuclear norm (sum of the singular values) and $\|\cdot \|_{1}$ is $\ell_{1}$-norm. Applying the same relaxation to (\ref{EqnLowRankSparseForm}) yields
\begin{eqnarray}\label{EqnNuclearL1Form}
\min_{ \{ {\bf P}^{k} \in {\cal{P}}^{k} \}_{k=1}^{K} , {\bf L}, {\bf E} } \| {\bf L} \|_{*} + \lambda \| {\bf E} \|_{1} \nonumber \quad\quad\quad \\ \mathrm{s.t.} \quad {\bf D}( \{\mathbf{P}^k \}_{k=1}^K ) = {\bf L} + {\bf E}, \nonumber \quad\quad \\ {\cal{P}}^{k} = \big\{ {\bf P}^{k} \in \{0, 1\}^{n_{k}\times n} \big| {\bf 1}_{n_{k}}^{\top}{\bf P}^{k} = {\bf 1}_{n}^{\top},  \nonumber \quad \\ {\bf P}^{k}{\bf 1}_{n} \leq {\bf 1}_{n_{k}}  \big\}, \ \forall \ k = 1, \dots, K , \quad
\end{eqnarray}
where we have written the constraints of $\{ {\cal{P}}^{k} \}_{k=1}^{K}$ in matrix form, and ${\bf 1}_{n_{k}}$ (or ${\bf 1}_{n}$) denotes a column vector of length $n_{k}$ (or $n$) with all entry values of $1$. We refer to the problem (\ref{EqnNuclearL1Form}) as our framework of {\it Robust Object Matching using Low-rank (and sparse) constraints} (ROML). Without mentioning we always set the parameter $\lambda = 5/\sqrt{dn}$, which is suggested by the closely related work of Robust Principal Component Analysis (RPCA) \cite{RPCA} \footnote{Suppose we have a data matrix $\mathbf{D} \in \mathbb{R}^{m_1\times m_2}$, which is formed by superposition of a low-rank matrix $\mathbf{L}$ and a sparse matrix $\mathbf{E}$, i.e., $\mathbf{D} = \mathbf{L} + \mathbf{E}$. Assume the low-rank matrix $\mathbf{L}$ is not sparse, and the sparse matrix $\mathbf{E}$ is not low-rank (e.g., the support pattern of $\mathbf{E}$ may be selected uniformly at random).  RPCA proves that the matrices $\mathbf{L}$ and $\mathbf{E}$ can be recovered {\it exactly} via a convex program called Principal Component Pursuit: $\min_{\mathbf{L}_{rpca}, \mathbf{E}_{rpca}} \| \mathbf{L}_{rpca} \|_* + \lambda_{rpca} \| \mathbf{E}_{rpca} \|_1 \ \mathrm{s.t.} \ \mathbf{D} = \mathbf{L}_{rpca} + \mathbf{E}_{rpca}$, provided that the rank of $\mathbf{L}$ is not too large, and that $\mathbf{E}$ is reasonably sparse. Under these assumptions, a {\it universal} choice of the parameter $\lambda_{rpca} = 1/\sqrt{\max(m_1, m_2)}$ is identified in the theoretical analysis of RPCA. In practical data, assumptions used in the theoretical proof of RPCA are not generally satisfied. RPCA suggests setting $\lambda_{rpca} = C/\sqrt{\max(m_1, m_2)}$, where $C$ is a constant which can be adjusted properly to improve performance on practical data.}.

The problem (\ref{EqnNuclearL1Form}) involves jointly optimizing a set of $K$ PPMs. As reviewed in Section \ref{LiteratureSec}, it is an instance of MiAP and proved to be NP-hard, for which approximate solution methods are practically used. To solve (\ref{EqnNuclearL1Form}), note that it is a formulation of regularized consensus problem, where the local variables $\mathrm{vec}({\bf F}^{k}{\bf P}^{k})$ (function of ${\bf P}^{k}$), $k = 1, \dots, K$, in ${\bf D}( \{\mathbf{P}^k \}_{k=1}^K )$ are constrained to be equal to components (column vectors) of the global variable ${\bf L} + {\bf E}$, which is further regularized in the objective function. In literature, consensus problems are popularly solved using ADMM method in the context of distributed optimization \cite{BoydADMM,Bertsekas,BertsekasDistributedOptBook}. The general ADMM method decomposes a global problem into local subproblems that can be readily solved. For consensus problems such as (\ref{EqnNuclearL1Form}), ADMM decomposes optimization of ${\bf L}$, ${\bf E}$, and $\{ {\bf P}^{k} \}_{k=1}^{K}$ into subproblems that update ${\bf L}$, ${\bf E}$, and each of $\{ {\bf P}^{k} \}_{k=1}^{K}$ respectively. Thus joint optimization over $\{ {\bf P}^{k} \}_{k=1}^{K}$ boils down as independent optimization of individual ${\bf P}^{k}$, $k = 1, \dots, K$, in each ADMM iteration. However, the subproblem to update each ${\bf P}^{k}$ concerns with nonlinear integer programming. It is essential to understand the convergence property of ADMM under this condition, which we will discuss in Section \ref{ADMMConvergeAnalysisSec} after presentation of our algorithmic procedure.

We first write the augmented Lagrangian of (\ref{EqnNuclearL1Form}) as
\begin{multline}\label{EqnMainLagrangian}
{\cal{L}}_{\rho}({\bf L}, {\bf E}, \{ {\bf P}^{k} \in {\cal{P}}^{k} \}_{k=1}^{K}, {\bf Y}) = \| {\bf L} \|_{*} + \lambda \| {\bf E} \|_{1} + \\ \frac{\rho}{2} \| {\bf L} + {\bf E} - {\bf D}( \{\mathbf{P}^k \}_{k=1}^K ) + \frac{1}{\rho}{\bf Y} \|_{F}^{2} ,
\end{multline}
where ${\bf Y} \in \mathbb{R}^{dn\times K}$ is a matrix of Lagrange multipliers, $\rho$ is a positive scalar, and $\| \cdot \|_{F}$ denotes the Frobenius norm. The ADMM algorithm iteratively estimates one of the matrices ${\bf L}$, ${\bf E}$, $\{ {\bf P}^{k} \}_{k=1}^{K}$, and the Lagrange multiplier ${\bf Y}$ by minimizing (\ref{EqnMainLagrangian}), while keeping the others fixed. More specifically, our ADMM procedure consists of the following iterations \begin{eqnarray}\label{EqnADMMMainIters} {\bf L}_{t+1} & = & \arg\min_{{\bf L}} {\cal{L}}_{\rho}\big({\bf L}, {\bf E}_{t}, \{ {\bf P}_{t}^{k} \}_{k=1}^{K}, {\bf Y}_{t} \big), \label{EqnMainLagUpdateL} \\ {\bf E}_{t+1} & = & \arg\min_{{\bf E}} {\cal{L}}_{\rho}\big({\bf L}_{t+1}, {\bf E}, \{ {\bf P}_{t}^{k} \}_{k=1}^{K}, {\bf Y}_{t} \big), \label{EqnMainLagUpdateE} \\ \{ {\bf P}_{t+1}^{k} \}_{k=1}^{K} \!\!\!\!\!\!\! & = & \!\!\!\!\!\!\! \mathop{\arg\min}_{ \{ {\bf P}^{k} \in {\cal{P}}^{k} \}_{k=1}^{K} } \!\!\!\!\! {\cal{L}}_{\rho}\big({\bf L}_{t+1}, {\bf E}_{t+1}, \{ {\bf P}^{k} \}_{k=1}^{K}, {\bf Y}_{t} \big), \label{EqnMainLagUpdateP} \\ {\bf Y}_{t+1} & = & {\bf Y}_{t} + \rho\big({\bf L}_{t+1} + {\bf E}_{t+1} - {\bf D}_{t+1}\big) , \label{EqnMainLagUpdateY} \end{eqnarray} where $t$ denotes the iteration number and we compute ${\bf D}_{t+1} = [ \mathrm{vec}({\bf F}^{1}{\bf P}_{t+1}^{1}), \dots, \mathrm{vec}({\bf F}^{K}{\bf P}_{t+1}^{K}) ]$ after step (\ref{EqnMainLagUpdateP}).

The problems (\ref{EqnMainLagUpdateL}) and (\ref{EqnMainLagUpdateE}) for updating the global variables ${\bf L}$ and ${\bf E}$ are both convex programs. They can be explicitly written as the forms of the proximal operators associated with a nuclear norm or an $\ell_{1}$-norm respectively \cite{ZhouchenALM}. To spell out the solutions, define the soft-thresholding operator for scalars as ${\cal{T}_{\tau}}[x] = \mathrm{sign}(x)\cdot \max\{|x|-\tau , 0\}$, with $\tau > 0$. When applied to vectors/matrices, it operates element-wisely. The optimal solution to (\ref{EqnMainLagUpdateL}) and (\ref{EqnMainLagUpdateE}) can be written as \begin{eqnarray}\label{EqnMainLagUpdateLSolution} ({\bf U}, {\bf S}, {\bf V}) & = & \mathrm{svd}\big({\bf D}_{t} - {\bf E}_{t} - \frac{1}{\rho}{\bf Y}_{t}\big), \nonumber \\ {\bf L}_{t+1} & = & {\bf U}{\cal{T}}_{\frac{1}{\rho}}\big[{\bf S}\big]{\bf V}^{\top} , \end{eqnarray} \begin{equation}\label{EqnMainLagUpdateESolution} {\bf E}_{t+1} = {\cal{T}}_{\frac{\lambda}{\rho}}\big[{\bf D}_{t} - {\bf L}_{t+1} - \frac{1}{\rho}{\bf Y}_{t}\big] . \end{equation}
Derivations of the above solution to problems (\ref{EqnMainLagUpdateL}) and (\ref{EqnMainLagUpdateE}) can be found in Appendix \ref{appendix_LESolutionDerivation}.

Optimization of (\ref{EqnMainLagUpdateP}) is more involved than those of (\ref{EqnMainLagUpdateL}) and (\ref{EqnMainLagUpdateE}), mostly because of the binary constraints enforced on the entries of $\{ {\bf P}^{k} \}_{k=1}^{K}$. Given updated variables $\mathbf{L}_{t+1}$, $\mathbf{E}_{t+1}$, and $\mathbf{Y}_t$, we explicitly write the problem (\ref{EqnMainLagUpdateP}) as
\begin{multline}\label{EqnUpdatePExpand} \min_{ \{ {\bf P}^{k} \in {\cal{P}}^{k} \}_{k=1}^{K} } \frac{\rho}{2} \big\| {\bf L}_{t+1} + {\bf E}_{t+1} + \frac{1}{\rho} {\bf Y}_{t} - \\ [ \mathrm{vec}({\bf F}^{1}{\bf P}^{1}), \dots, \mathrm{vec}({\bf F}^{K}{\bf P}^{K}) ] \big\|_{F}^{2} . \end{multline}
We observe that (\ref{EqnUpdatePExpand}) can be decoupled into $K$ independent subproblems, each of which concerns with optimization of one of the local variables $\{ {\bf P}^{k} \}_{k=1}^{K}$. The $k^{th}$ subproblem to update ${\bf P}^{k}$ is written as
\begin{equation}\label{EqnUpdatePExpandSubProblem} \min_{ {\bf P}^{k} \in {\cal{P}}^{k} } \frac{\rho}{2} \big\| \big( {\bf L}_{t+1} + {\bf E}_{t+1} + \frac{1}{\rho} {\bf Y}_{t} \big) {\bf e}_{k} - \mathrm{vec}({\bf F}^{k}{\bf P}^{k}) \big\|_{2}^{2}  , \end{equation}
where ${\bf e}_{k}$ denotes a unit column vector with all entries set to $0$ except the $k^{th}$ one, which is set to $1$. Denote $\theta^{k} = \mathrm{vec}({\bf P}^{k}) \in \mathbb{R}^{nn_{k}}$, ${\bf G}^{k} = {\bf I}_{n} \otimes {\bf F}^{k} \in \mathbb{R}^{dn\times nn_{k}}$, ${\bf J}^{k} = {\bf I}_{n}\otimes {\bf 1}_{n_{k}}^{\top} \in \mathbb{R}^{n\times n n_{k}}$, ${\bf H}^{k} = {\bf 1}_{n}^{\top} \otimes {\bf I}_{n_{k}} \in \mathbb{R}^{n_{k}\times nn_{k}}$, $\otimes$ is the Kronecker product, and ${\bf I}_{n}$ (or ${\bf I}_{n_{k}}$) is the identity matrix of size $n\times n$ (or $n_{k}\times n_{k}$). Using the fact $\mathrm{vec}({\bf XYZ}) = ({\bf Z}^{\top}\otimes {\bf X})\mathrm{vec}({\bf Y})$, we can rewrite (\ref{EqnUpdatePExpandSubProblem}) as the following equivalent problem to update $\theta^{k}$
\begin{multline}\label{EqnThetaSubProbBinary} \min_{\theta^{k}} \frac{\rho}{2}\theta^{k\top}{\bf G}^{k\top}{\bf G}^{k}\theta^{k} - {\bf e}_{k}^{\top}\big[{\bf Y}_{t}^{\top} + \rho\big({\bf L}_{t+1} + {\bf E}_{t+1}\big)^{\top} \big] {\bf G}^{k}\theta^{k} \quad\quad\quad\quad \\ \mathrm{s.t.} \ {\bf J}^{k}\theta^{k} = {\bf 1}_{n}, \ {\bf H}^{k}\theta^{k} \leq {\bf 1}_{n_{k}}, \ \theta^{k} \in \{0, 1\}^{nn_{k}}. \end{multline}
(\ref{EqnThetaSubProbBinary}) appears to be a difficult integer constrained quadratic program. To solve it, a common approach is to relax the constraint set of (\ref{EqnThetaSubProbBinary}) into its convex hull, and then project back the attained continuous-domain results by either thresholding or more complicated methods, which, however, cannot guarantee to get the optimal solution \cite{leordeanu2009integer}. For the ROML problem (\ref{EqnNuclearL1Form}), we assume that distinctive information of each column vector in any ${\bf F}^{k}$ of $\{ {\bf F}^{k} \}_{k=1}^{K}$ is represented by the relative values of its elements, rather than their absolute magnitude. In other words, multiplying each feature vector by a scaling factor does not change the pattern of each feature. Based on this assumption, we prove that (\ref{EqnThetaSubProbBinary}) is equivalent to a linear sum assignment problem \cite{AssignProbBook}.

\begin{theorem}\label{ExactRelaxTheorem}
For the ROML problem (\ref{EqnNuclearL1Form}), assuming distinctive information of each column vector in any ${\bf F}^{k}$ of $\{ {\bf F}^{k} \}_{k=1}^{K}$ is represented by the relative values of its elements, (\ref{EqnThetaSubProbBinary}) is always equivalent to the following formulation of linear sum assignment problem \begin{eqnarray}\label{EqnThetaSubProbRelaxLP} \min_{\theta^{k}} - {\bf e}_{k}^{\top}\big[{\bf Y}_{t}^{\top} + \rho\big({\bf L}_{t+1} + {\bf E}_{t+1}\big)^{\top} \big] {\bf G}^{k}\theta^{k} \nonumber \quad \\ \mathrm{s.t.} \ {\bf J}^{k}\theta^{k} = {\bf 1}_{n}, \ {\bf H}^{k}\theta^{k} \leq {\bf 1}_{n_{k}}, \ \theta^{k} \in \{0, 1\}^{nn_{k}} . \end{eqnarray}
\end{theorem}

\begin{proof}
We prove the equivalence by showing that, under the considered assumption for the ROML problem (\ref{EqnNuclearL1Form}), the objective function of (\ref{EqnThetaSubProbBinary}) is equivalent to a linear function, as written in (\ref{EqnThetaSubProbRelaxLP}), which together with the constraints of (\ref{EqnThetaSubProbRelaxLP}), turns out to be a formulation of LSAP. Denote ${\bf p}_{j}^{k} \in \mathbb{R}^{n_{k}}$, $j = 1, \dots, n$, as columns of PPM ${\bf P}^{k}$. From the definitions of ${\bf G}^{k}$ and $\theta^{k}$, it is straightforward to show that \begin{equation}\label{EqnExactRelaxTheoremEqn1} {\bf G}^{k}\theta^{k} = \mathrm{vec}({\bf F}^{k}{\bf P}^{k}) = \big[({\bf F}^{k}{\bf p}_{1}^{k})^{\top}, \dots, ({\bf F}^{k}{\bf p}_{n}^{k})^{\top}\big]^{\top} \end{equation} Since ${\bf F}^{k} = [{\bf f}_{1}^{k}, \dots, {\bf f}_{n_{k}}^{k}] \in \mathbb{R}^{d\times n_{k}}$, and from the constraints of ${\bf P}^{k}$ (explicitly stated in (\ref{EqnNuclearL1Form})), it is clear that each subvector ${\bf F}^{k}{\bf p}_{j}^{k}$, $j = 1, \dots, n$, of (\ref{EqnExactRelaxTheoremEqn1}) selects one column feature vector from ${\bf F}^{k}$, with a unique index from the set $\{1, \dots, n_{k} \}$. From (\ref{EqnExactRelaxTheoremEqn1}) we also have \begin{equation}\label{EqnExactRelaxTheoremEqn2} \theta^{k\top}{\bf G}^{k\top}{\bf G}^{k}\theta^{k} = \| {\bf G}^{k}\theta^{k} \|_{2}^{2} = \sum_{j=1}^{n}\| {\bf F}^{k}{\bf p}_{j}^{k}\|_{2}^{2} .\end{equation} In case of $n_{k} = n$, i.e., there exist no outliers in the considered feature-based object matching, (\ref{EqnExactRelaxTheoremEqn2}) is equal to a constant value no matter what feasible ${\bf P}^{k}$ or $\theta^{k}$ is used. In the more general case of $n_{k} > n$, since information of each of the feature vectors ${\bf f}_{i}^{k}$, $i = 1, \dots, n_{k}$, is preserved by relative values of its elements, we can always normalize them so that they have an equal Euclidean norm, i.e., $\| {\bf f}_{1}^{k} \|_{2} = \cdots = \| {\bf f}_{n_{k}}^{k} \|_{2} = c_{k}$. And (\ref{EqnExactRelaxTheoremEqn2}) is again equal to a constant value no matter what feasible ${\bf P}^{k}$ or $\theta^{k}$ is used. We thus finish the proof. \qed
\end{proof}

The LSAP (\ref{EqnThetaSubProbRelaxLP}) can be exactly and efficiently solved using a rectangular-matrix variant of the Hungarian algorithm \cite{AssignProbBook}. After solving $K$ (\ref{EqnThetaSubProbRelaxLP})-like problems for $k = 1, \dots, K$, we get the updates of $\{ \theta_{t+1}^{k} \}_{k=1}^{K}$ and compute ${\bf D}_{t+1} = \big[ {\bf G}^{1}\theta_{t+1}^{1}, \dots, {\bf G}^{K}\theta_{t+1}^{K}\big]$. The Lagrange multiplier matrix ${\bf Y}_{t+1}$ is then updated using (\ref{EqnMainLagUpdateY}). Our ADMM procedure iteratively performs the steps (\ref{EqnMainLagUpdateL}), (\ref{EqnMainLagUpdateE}), (\ref{EqnMainLagUpdateP}), and (\ref{EqnMainLagUpdateY}), until a specified stopping condition is satisfied. Normally, the primal and dual residuals can be used as the stopping criteria \footnote{For the ROML problem (\ref{EqnNuclearL1Form}), the primal residual is ${\cal{R}}_{pri.}^{t+1} = {\bf L}^{t+1} + {\bf E}^{t+1} - {\bf D}^{t+1}$, and the dual residuals are ${\cal{R}}_{dual,{\bf L}}^{t+1} = \rho ( {\bf E}^{t} + {\bf D}^{t} - {\bf E}^{t+1} - {\bf D}^{t+1}) $ (w.r.t. the variable ${\bf L}$) and ${\cal{R}}_{dual,{\bf E}}^{t+1} = \rho ({\bf D}^{t} - {\bf D}^{t+1})$ (w.r.t. the variable ${\bf E}$). }. To improve the convergence, a common practice is to use a monotonically increasing sequence of $\{ \rho_{t} \}$. We also adopt this strategy. The pseudocode of our algorithm is summarized in Algorithm \ref{MainAlgm} \footnote{We note that the obtained solution $\{ {\bf P}^{k} \}_{k=1}^{K}$ by solving the ROML problem (\ref{EqnNuclearL1Form}) using Algorithm \ref{MainAlgm} belongs to a group of equivalent solutions, since there is no constraint on the order of columns in any of $\{ {\bf P}^{k} \}_{k=1}^{K}$. It is always easy to transform $\{ {\bf P}^{k} \}_{k=1}^{K}$ to some canonical form, e.g., by permuting columns of each of $\{ {\bf P}^{k} \}_{k=1}^{K}$ according to sorted image coordinates of feature points in any of the $K$ images. Without loss of generality we assume the solution $\{ {\bf P}^{k} \}_{k=1}^{K}$ given by Algorithm \ref{MainAlgm} has been transformed to some canonical form for ease of evaluation.}.

\begin{algorithm}[t]

{\footnotesize

\SetKwInOut{Input}{input}
\SetKwInOut{Output}{output}

\Input{Feature vectors ${\bf F}^{k} = [{\bf f}_{1}^{k}, \dots, {\bf f}_{n_{k}}^{k}] \in \mathbb{R}^{d\times n_{k}}$ (normalized as $\| {\bf f}_{1}^{k} \|_{2} = \cdots = \| {\bf f}_{n_{k}}^{k} \|_{2} = c_{k}$ when there exist outliers), $k = 1, \dots, K$, the number $n$ of inliers, weight $\lambda > 0$, and initialization of $\{ {\bf P}_{0}^{k} \in {\cal{P}}^{k} \}_{k=1}^{K}$, ${\bf L}_{0} = 0$, ${\bf E}_{0} = 0$, ${\bf Y}_{0} = 0$, and $\rho_{0} > 0$. }

\While{not converged}{

$({\bf U}, {\bf S}, {\bf V}) = \mathrm{svd}\big({\bf D}_{t} - {\bf E}_{t} - \frac{1}{\rho_{t}}{\bf Y}_{t}\big)$.

${\bf L}_{t+1} = {\bf U}{\cal{T}}_{\frac{1}{\rho_{t}}}\big[{\bf S}\big]{\bf V}^{\top}$.

${\bf E}_{t+1} = {\cal{T}}_{\frac{\lambda}{\rho_{t}}}\big[{\bf D}_{t} - {\bf L}_{t+1} - \frac{1}{\rho_{t}}{\bf Y}_{t}\big]$.

\For{each $k$}{

let $\theta_{t}^{k} = \mathrm{vec}({\bf P}_{t}^{k})$, ${\bf G}^{k} = {\bf I}_{n} \otimes {\bf F}^{k}$, ${\bf J}^{k} = {\bf I}_{n}\otimes {\bf 1}_{n_{k}}^{\top}$, ${\bf H}^{k} = {\bf 1}_{n}^{\top} \otimes {\bf I}_{n_{k}}$, solve the LSAP problem (\ref{EqnThetaSubProbRelaxLP}) to get the update $\theta_{t+1}^{k}$.

}

${\bf D}_{t+1} = \big[ {\bf G}^{1}\theta_{t+1}^{1}, \dots, {\bf G}^{K}\theta_{t+1}^{K}\big]$.

${\bf Y}_{t+1} = {\bf Y}_{t} + \rho_{t}\big({\bf L}_{t+1} + {\bf E}_{t+1} - {\bf D}_{t+1}\big)$.

$\rho_{t+1} \leftarrow \rho_{t}$.

$t \leftarrow t+1$.

}

\Output{ solution $\{ {\bf P}_{t}^{k} \}_{k=1}^{K}$, ${\bf L}_{t}$, and ${\bf E}_{t}$.  }

\caption{ Solving ROML by ADMM } \label{MainAlgm}

}
\end{algorithm}

\subsubsection{Discussion of Solving ROML Using ADMM}
\label{ADMM4ROMLDiscussSec}

Solving the ROML formulation (\ref{EqnNuclearL1Form}) establishes $n$ sets of consistent feature correspondences across the given $K$ images. In other words, it aims to find $n$ ``good'' ones out of the total $\frac{1}{n!}( \frac{n_{k}!}{ (n_{k}-n)!} )^K$ feasible solutions, assuming $n_{1} = \cdots = n_{k} = \cdots = n_{K}$. As reviewed in Section \ref{LiteratureSec}, ROML belongs to the more general class of MiAP. To see how ROML relates to MiAP, we write the standard MiAP formulation \cite{AssignProbBook} for the considered feature-based object matching problem as
\begin{eqnarray}\label{EqnGeneralMiAPForm} \min_{ \{ z_{ i_{1},i_{2},\dots,i_{K} } \} } \sum_{i_{1}=1}^{n_{1}}\sum_{i_{2}=1}^{n_{2}}\cdots\sum_{i_{K}=1}^{n_{K}} a_{ i_{1},i_{2},\dots,i_{K} } z_{ i_{1},i_{2},\dots,i_{K} } \nonumber \\ \mathrm{s.t.} \sum_{i_{2}=1}^{n_{2}}\sum_{i_{3}=1}^{n_{3}}\cdots\sum_{i_{K}=1}^{n_{K}} z_{ i_{1},i_{2},\dots,i_{K} } \leq 1 \ , \ i_{1} = 1, \dots, n_{1} \ \ \nonumber \\  \sum_{i_{1}=1}^{n_{1}}\sum_{i_{3}=1}^{n_{3}}\cdots\sum_{i_{K}=1}^{n_{K}} z_{ i_{1},i_{2},\dots,i_{K} } \leq 1 \ , \ i_{2} = 1, \dots, n_{2} \ \ \nonumber \\  \vdots \qquad\qquad\qquad\qquad \ \ \nonumber \\ \sum_{i_{1}=1}^{n_{1}}\sum_{i_{2}=1}^{n_{2}}\cdots\sum_{i_{K-1}=1}^{n_{K-1}} z_{ i_{1},i_{2},\dots,i_{K} } \leq 1 \ , \ i_{K} = 1, \dots, n_{K} \nonumber \\ \sum_{i_{1}=1}^{n_{1}}\sum_{i_{2}=1}^{n_{2}}\cdots\sum_{i_{K}=1}^{n_{K}} z_{ i_{1},i_{2},\dots,i_{K} } = n \ , \ z_{ i_{1},i_{2},\dots,i_{K} } \in \{0, 1\} ,  \end{eqnarray}
where $i_{k} \in \{1, \dots, n_{k}\}$ indexes the $n_{k}$ feature vectors extracted from the $k^{th}$ image, $k = 1, \dots, K$. The global decision variable $z_{ i_{1},i_{2},\dots,i_{K} }$ is equal to $1$ when the corresponding feature points are matched across the $K$ images, with each feature from one of the $K$ images, and $a_{ i_{1},i_{2},\dots,i_{K} }$ denotes the cost of this matching. By factoring/reformulating the set of global decision variables $\{ z_{ i_{1},i_{2},\dots,i_{K} } \}$ as PPMs $\{ {\bf P}^{k} \}_{k=1}^{K}$ defined by (\ref{EqnPPM}), we get the following equivalent problems
\begin{eqnarray}\label{EqnMiAP2ROMLForm1} \min_{ \{ {\bf P}^{k} \in {\cal{P}}^{k} \}_{k=1}^{K} } \sum_{j=1}^{n} \sum_{i_{1}=1}^{n_{1}}\cdots\sum_{i_{K}=1}^{n_{K}} a_{ i_{1},\dots,i_{K} } \prod_{k=1}^{K}p_{i_{k},j}^{k} , \end{eqnarray}
\begin{eqnarray}\label{EqnMiAP2ROMLForm2} \min_{ \{ {\bf P}^{k} \in {\cal{P}}^{k} \}_{k=1}^{K} } \sum_{j=1}^{n} \sum_{i_{1}=1}^{n_{1}}\cdots\sum_{i_{K}=1}^{n_{K}} \big\| [ {\bf f}_{i_{1}}^{1}, \dots, {\bf f}_{i_{K}}^{K} ] \big\|_{*} \prod_{k=1}^{K}p_{i_{k},j}^{k} , \end{eqnarray}
\begin{eqnarray}\label{EqnMiAP2ROMLForm3} \min_{ \{ {\bf P}^{k} \in {\cal{P}}^{k} \}_{k=1}^{K} } \sum_{j=1}^{n} \big\| [ {\bf F}^{1}{\bf p}_{j}^{1}, \dots, {\bf F}^{K}{\bf p}_{j}^{K} ] \big\|_{*} , \end{eqnarray}
where $j \in \{1, \dots, n\}$ indexes the $n$ inlier matches, and $p_{i_{k}, j}^{k}$ and ${\bf p}_{j}^{k}$ denote the $(i_{k}, j)$ entry and $j^{th}$ column of ${\bf P}^{k}$ respectively. In (\ref{EqnMiAP2ROMLForm2}) and (\ref{EqnMiAP2ROMLForm3}), we have used nuclear norm of the matrix formed by a candidate match of $K$ feature vectors as the cost coefficient $a_{ i_{1},\dots,i_{K} }$. As an instance of MiAP, jointly optimizing the set of PPMs in the above equivalent problems is NP-hard. Approximate methods are thus important to get practically meaningful solutions. In fact, due to inevitable noise in cost coefficients, e.g., that generated by various variations of object instances in different images, it is often sufficient to find suboptimal solutions that are within the noise level of the optimal one.

To understand how we have developed an approximate method in preceding sections, we slightly modify (\ref{EqnMiAP2ROMLForm3}) by vertically arraying the $n$ matrices $[ {\bf F}^{1}{\bf p}_{j}^{1}, \dots,$ ${\bf F}^{K}{\bf p}_{j}^{K} ]$, $j = 1, \dots, n$, as a bigger matrix, resulting in the following problem to optimize $\{ {\bf P}^{k} \}_{k=1}^{K}$
\begin{eqnarray}\label{EqnMiAP2ROMLForm4} \min_{ \{ {\bf P}^{k} \in {\cal{P}}^{k} \}_{k=1}^{K} } \big\| [ \mathrm{vec}({\bf F}^{1}{\bf P}^{1}), \dots, \mathrm{vec}({\bf F}^{K}{\bf P}^{K}) ] \big\|_{*} , \end{eqnarray}
which turns out to be equivalent to a nuclear norm relaxed version of (\ref{EqnLowRankForm}). Indeed, by introducing the global variable ${\bf L}$ in (\ref{EqnLowRankForm}), and also the global variables ${\bf L}$ and ${\bf E}$ in (\ref{EqnLowRankSparseForm}) and (\ref{EqnNuclearL1Form}) for a robust extension, we essentially formulate the multi-image object matching version of MiAP as a regularized consensus problem \cite{BertsekasDistributedOptBook}, with $\| {\bf L} \|_{*} + \lambda\| {\bf E} \|_{1}$ as the regularization term. It becomes well suited to be solved using distributed optimization methods such as ADMM. As in the presented ADMM procedure (\ref{EqnMainLagUpdateL})-(\ref{EqnMainLagUpdateY}), the ``fusion'' steps (\ref{EqnMainLagUpdateL}) and (\ref{EqnMainLagUpdateE}) collect information of the $t^{th}$ iteration $\{ {\bf P}_{t}^{k} \}_{k=1}^{K}$ to update ${\bf L}$ and ${\bf E}$, and the ``broadcast'' step (\ref{EqnMainLagUpdateP}) independently updates each ${\bf P}^{k}$ of $\{ {\bf P}^{k} \}_{k=1}^{K}$, using the updated fusion centers ${\bf L}_{t+1}$ and ${\bf E}_{t+1}$. Our proposed ADMM method thus belongs to a strategy of ``fusion-and-broadcast'', for ROML and the more general MiAP.

\subsubsection{Convergence Analysis}
\label{ADMMConvergeAnalysisSec}

The ADMM method is proven to converge to global optimum under some mild conditions for linearly constrained convex problem whose objective function is separable into two individual convex functions with non-overlapping variables (see \cite{ConvergRef3} and references therein). In our case, the ROML problem (\ref{EqnNuclearL1Form}) is nonconvex, due to the binary constraint associated with $\{ {\bf P}^{k} \}_{k=1}^{K}$. The convergence property of ADMM for nonconvex problems is still an open question in theory. However, it is not uncommon to see that ADMM has served a powerful heuristic for some nonconvex problems in practice \cite{ConvergRef6,ConvergRef7}. In the following, we present simulated experiments that demonstrate the excellent convergence property of ADMM for ROML.

Specifically, we generated synthetically $K = 30$ groups of vectors simulating extracted feature vectors from $K$ images, with dimension of each vector ${\bf f}$ as $d = 50$. There were both inlier and outlier feature vectors in each group. The inliers were produced by randomly generating $d$-dimensional vectors whose entries were drawn from i.i.d. normal distribution, and were shared in each of the $K$ groups. The outliers were similarly produced by randomly generating $d$-dimensional vectors following i.i.d. normal distribution, but were independently generated for each group. We then added sparse errors of large magnitude to both inlier and outlier vectors. For each vector ${\bf f}$, the error values were uniformly drawn from the range $[-2\max (\mathrm{abs} ({\bf f})), 2\max (\mathrm{abs} ({\bf f}))]$. Finally, we normalized all vectors to constant $\ell_{2}$-norm to fit with our algorithmic settings. We fixed the number of inliers in each group as $n = 10$, and investigated the convergence and recovery properties of our algorithm under varying numbers of outliers and ratios of sparse errors. The number of outliers in each group was ranged in $[0, 40]$, and the ratio of sparse errors in each vector was ranged in $[0, 0.8]$. Denote the ground truth PPMs of any test setting as $\{ {{\bf P}^{k}}^{*} \}_{k=1}^{K}$, and the recovered PPMs as $\{ {\bf P}^{k} \}_{k=1}^{K}$. The recovery rate is computed as $\sum_{k=1}^{K} \| {\bf P}^{k} \circ {{\bf P}^{k}}^{*} \|_{0} \big/ \sum_{k=1}^{K} \| {{\bf P}^{k}}^{*} \|_{0}$, where $\circ$ is Hadamard product. For each setting of outlier number and sparse error ratio, we run $5$ random trials and averaged the results. Figure \ref{ConvergSimuFigs}-(b) reports the recovery rates under different settings, which shows that our algorithm works perfectly in a large range of outlier numbers and ratios of sparse errors. For one of them (the outlier number is $32$ and sparse error ratio is $0.4$), we plot in Figure \ref{ConvergSimuFigs}-(a) its convergence curves of $5$ random trials in terms of the primal residual ($\| {\bf L} + {\bf E} - {\bf D} \|_{F}$) and objective function ($ \| {\bf L} \|_{*} + \lambda \| {\bf E} \|_{1}  $). Convergence properties under other settings are similar to Figure \ref{ConvergSimuFigs}-(a).

\begin{figure}[t]
\centering

\includegraphics[scale=0.22]{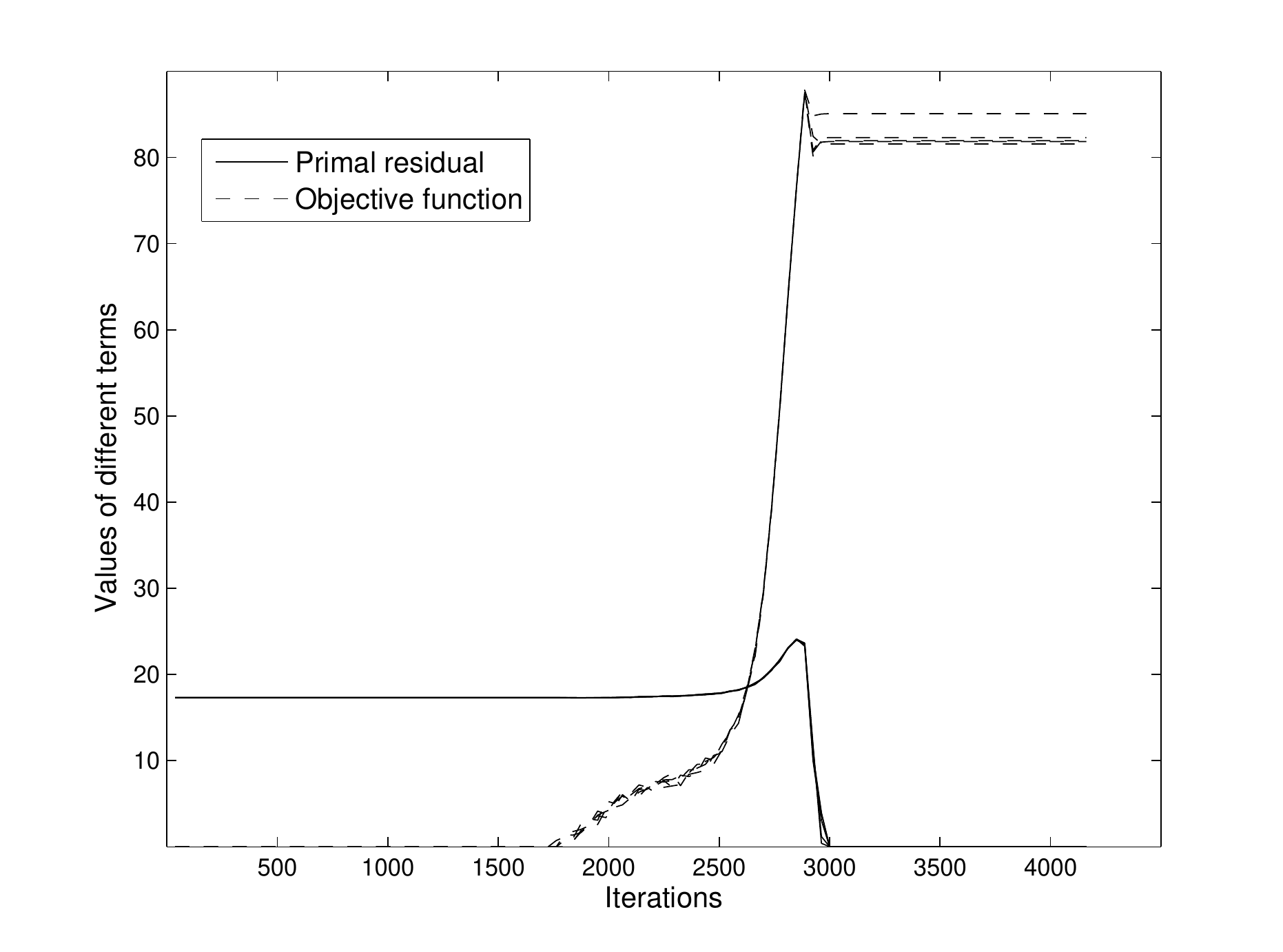}
\includegraphics[scale=0.22]{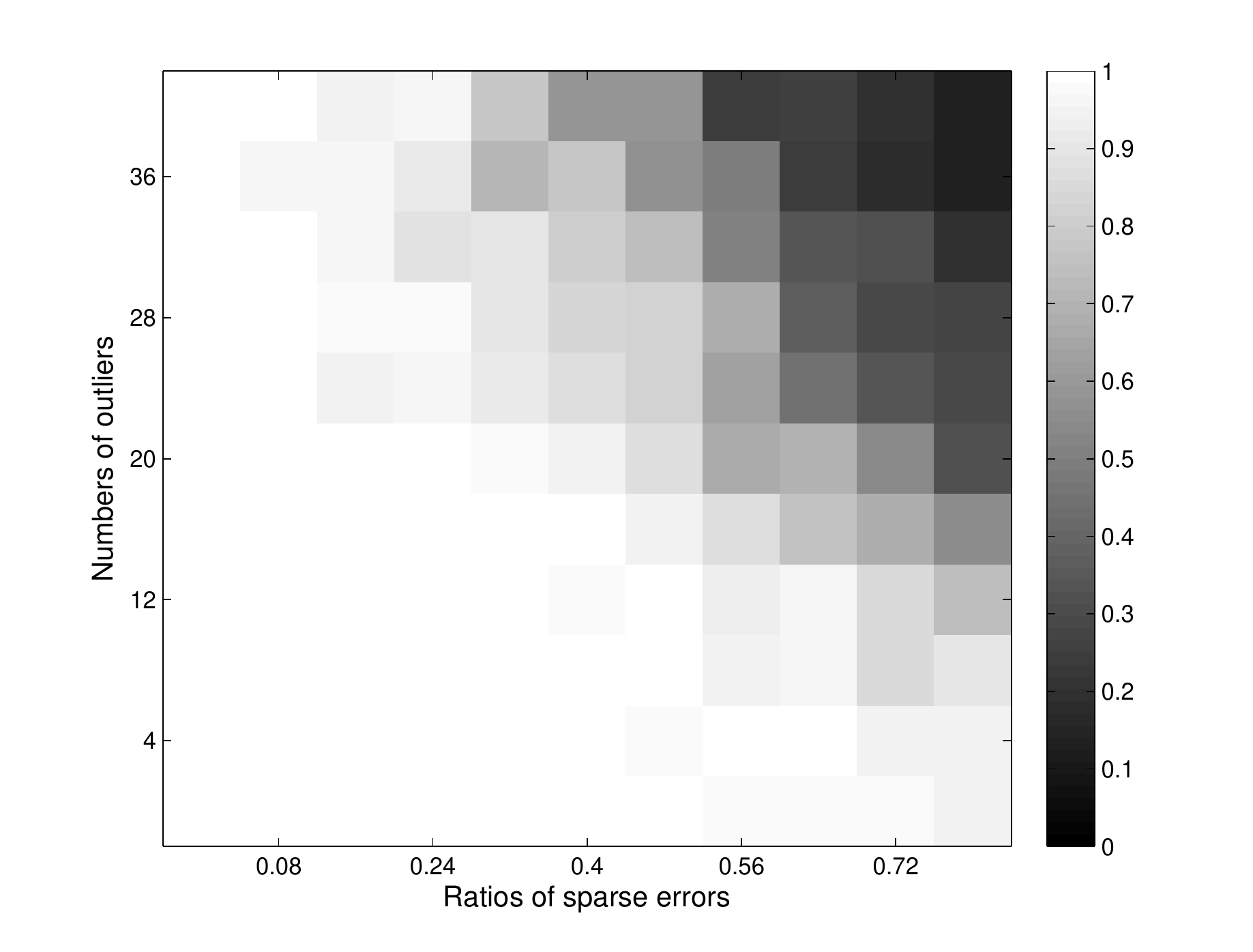} \\

\hfill\hfill (a) \hfill\hfill (b)  \hfill\hfill\hfill

\caption{ {\small Simulation of Algorithm \ref{MainAlgm}, using $30$ groups of synthetic feature vectors. The number of inliers in each group is fixed as $10$. The number of outliers in each group is ranged in $[0, 40]$, and the ratio of sparse errors in each vector is ranged in $[0, 0.8]$. Each setting of outlier number and sparse error ratio is tested with $5$ random trials. (a) convergence plot in terms of the primal residual and objective function for $5$ random trials of one test setting (the outlier number is $32$ and sparse error ratio is $0.4$); (b) recovery rates under different settings of outlier number and sparse error ratio, obtained by averaging over $5$ random trials for each test setting.}  }  \label{ConvergSimuFigs}
\end{figure}

\subsubsection{Computational Complexity}
\label{ComplexityAnalysisSec}

For ease of analysis we assume here $n_{1} = \cdots = n_{k} = \cdots = n_{K} > n$. Using an efficient implementation of the Hungarian algorithm \cite{AssignProbBook}, the complexity for solving the LSAP is ${\cal{O}}\big( n_{k}^{3} \big)$. The overall complexity for each iteration of Algorithm \ref{MainAlgm} is ${\cal{O}}\big( Kn_{k}^{3} + Kdn^{2}n_{k} + K^{2}dn \big)$. The number of iterations for Algorithm \ref{MainAlgm} to converge depends on the initial value of $\rho_{0}$ and the factor at which $\rho_{t}$ increases after each iteration. If $\rho_{t}$ increases too fast, it has the risk of converging to worse local optima \cite{ZhouchenALM}. In our experiments, without mentioning we always set $\rho_{0} = 1\mathrm{e}^{-4}$ and increase it iteratively with a factor of $1.001$. Under this setting, it normally takes about $3000$ iterations for Algorithm \ref{MainAlgm} to converge. In Section \ref{ExpObjMatching}, we also report practical computation time of our method and compare with that of competing ones.

\subsubsection{Estimating the Number of Inliers}
\label{InlierNumEstSec}


Up to now we have assumed that the number of inliers is known for a given image set. This might be a strong assumption. In order to investigate how performance of ROML is influenced by this prior knowledge, we conducted simulated experiments that take different values of $n$, assuming different numbers of inliers, as inputs of Algorithm \ref{MainAlgm}. More specifically, we generated $K = 30$ groups of synthetic feature vectors simulating extracted features from $K$ images, similarly as did in Section \ref{ADMMConvergeAnalysisSec}. There were $n_k = 30$, $k = 1, \dots, K$, feature vectors including both inliers and outliers in every $k^{th}$ group, with dimension of each vector ${\bf f}$ as $d = 50$. The ground truth number of inliers was $10$ in each group. The inlier feature vectors were generated by randomly drawing as vector entries from i.i.d. normal distribution, and were shared in each of the $K$ groups. The outliers were similarly produced, but were independently generated for each group. We then added sparse errors of large magnitude to both inlier and outlier vectors. For each vector ${\bf f}$, the error values were uniformly drawn from the range $[-2\max (\mathrm{abs} ({\bf f})), 2\max (\mathrm{abs} ({\bf f}))]$. We finally normalized all vectors to constant $\ell_{2}$-norm. In this investigation, we considered the test settings that the ratio of sparse errors in each feature vector was either $0.2$ or $0.4$.

By setting integer values of $n$ increasingly from $n = 1$ to $n = n_k$, we compared in each of the $n_k$ cases the identified and matched feature vectors from the $K$ groups by ROML (via Algorithm \ref{MainAlgm}), with the ground truth inlier features distributed in these $K$ groups. We introduce two measures to quantify these comparisons, namely {\it Match Ratio} and {\it Identification Ratio} \footnote{From any given $K$ groups of feature vectors, we can enumerate $\frac{K!}{(K-2)!2!}$ pairs of groups. For each group pair, assume ROML (or other feature matching methods) identify $n$ pair-wise correspondences between feature vectors of these two groups. Denote the number of {\it ground truth correspondences} between inlier features of these two groups as $n^*$, and the number of {\it identified ground truth correspondences} by ROML (or other feature matching methods) as $\bar{n}$. We define {\it Match Ratio} as $\sum \bar{n} / \sum n$, and {\it Identification Ratio} as $\sum \bar{n} / \sum n^*$, where $\sum$ sums over all the $\frac{K!}{(K-2)!2!}$ group pairs. Note that when $n = n^*$, i.e., setting value of $n$ as the ground truth number of inliers, the two introduced measures give equal results. Our definitions of the two measures using notions of pair-wise matching are to facilitate comparisons between ROML and traditional feature matching methods, such as graph matching \cite{TorresaniModelAndGlobalOpt,GraphReweightedRandomWalk,LeordeanuSpectralCorrespondence}, which work in the setting of matching between a pair of images/groups. Section \ref{ExpObjMatching} presents such comparisons.}, which are similar to the measures of {\it precision} and {\it recall} in information retrieval. Figure \ref{InlierEstSimuFigs}-(a) reports results of Match Ratio and Identification Ratio, which were obtained by averaging over $5$ random trials for each setting of sparse error ratio. Figure \ref{InlierEstSimuFigs}-(a) suggests that when the value of $n$ is close to the ground truth number of inliers, ROML performs better in terms of giving a balanced result of Match Ratio and Identification Ratio, which verifies the importance of knowing the true number of inliers when applying ROML to feature-based object matching. Unfortunately, in many practical applications such as object recognition or 3D reconstruction, this information is usually unknown for a given image set. It becomes essential to develop a mechanism to estimate the number of inliers, in order to apply ROML to these practical problems.

\begin{figure}[t]
\centering

\includegraphics[scale=0.22]{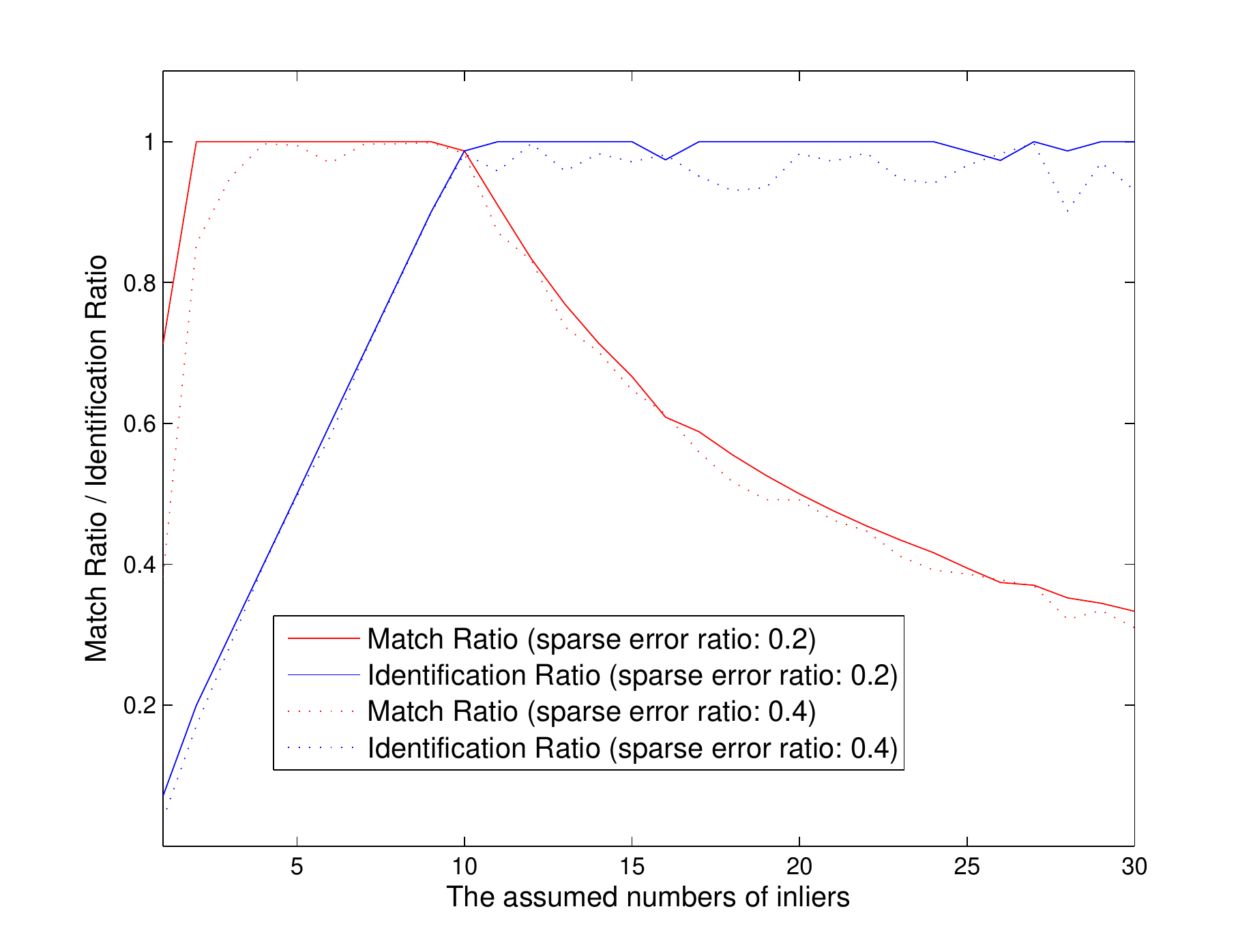}
\includegraphics[scale=0.22]{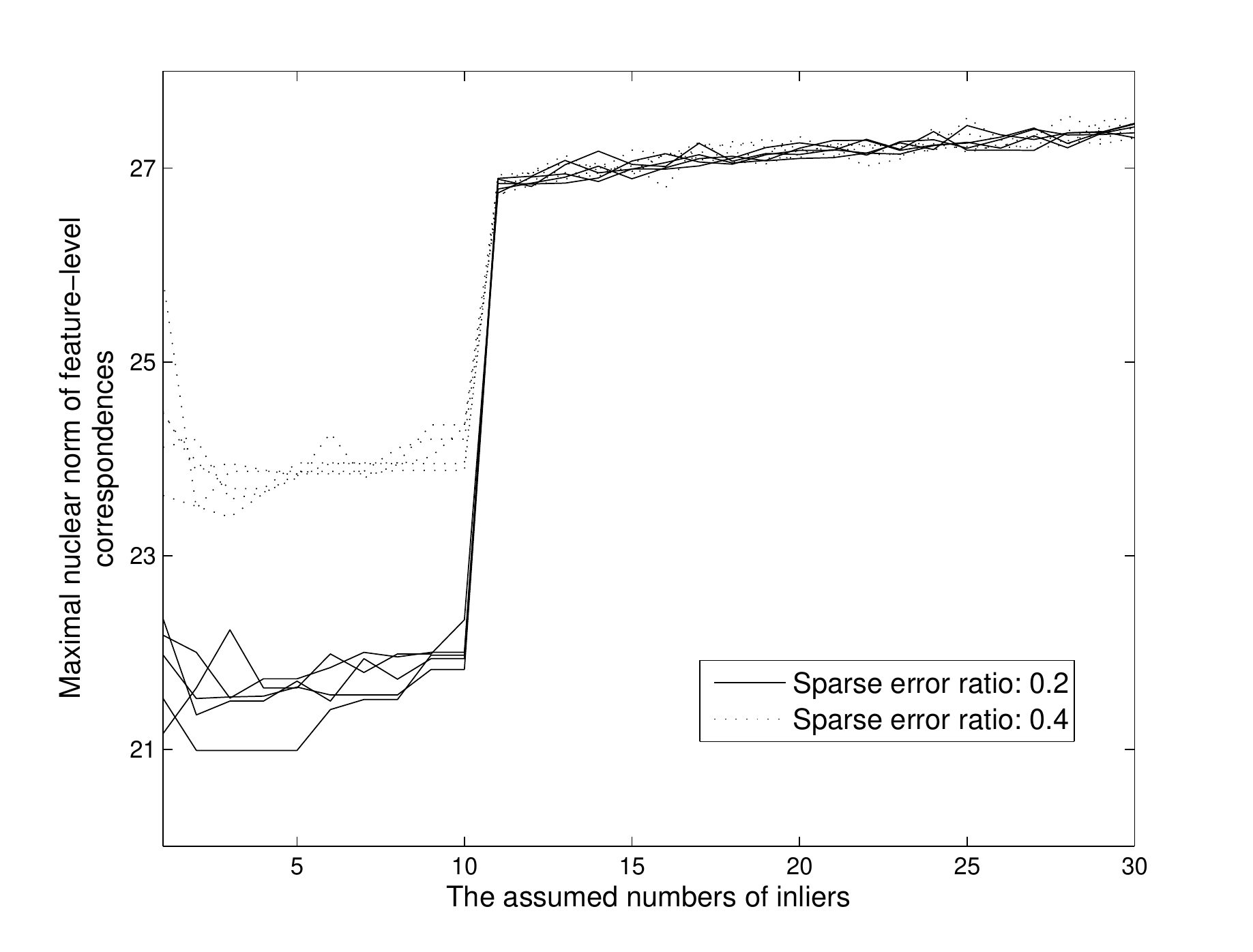} \\

\hfill\hfill (a) \hfill\hfill (b)  \hfill\hfill\hfill

\caption{ {\small Simulation for estimating the number of inliers, where two experiment settings are considered (sparse error ratios are $0.2$ and $0.4$ respectively), and each setting is tested with $5$ random trials. The ground truth number of inliers is $10$ in all these experiments. (a) Performance of ROML against varying assumed numbers of inliers (varying values of $n$ as inputs of Algorithm \ref{MainAlgm}), measured by Match Ratio and Identification Ratio. Each curve is obtained by averaging over $5$ random trials. (b) Maximal nuclear norm of the established $n$ feature-level correspondences by ROML (the $\gamma_n$ in equation (\ref{EqnMaxPerFeaNuclearNorm})) against varying assumed numbers of inliers (varying values of $n$ as inputs of Algorithm \ref{MainAlgm}). Each curve corresponds to one of the total $10$ experiments. }  }  \label{InlierEstSimuFigs}
\end{figure}

Our consideration of estimating the number of inliers is motivated by the exact low-rank property that ROML leverages for feature matching. Given a value of $n$, ROML establishes $n$ feature-level correspondences across the $K$ groups by solving Algorithm \ref{MainAlgm}, and formes a matrix ${\bf D} \in \mathbb{R}^{dn\times K}$, defined as equation (\ref{EqnDMatrix}), using the obtained PPMs $\{ \mathbf{P}^k \}_{k=1}^K$. Each of the $n$ feature-level correspondences consists of $K$ feature vectors respectively from the $K$ groups, and corresponds to a submatrix of ${\bf D}$, which we denote as ${\bf D}_j \in \mathbb{R}^{d\times K}$ with $j \in \{1, \dots, n\}$. Equation (\ref{EqnDMatrix}) indicates ${\bf D} = [ {\bf D}_1^{\top}, \dots, {\bf D}_n^{\top} ]^{\top}$. When any $j^{th}$ correspondence is formed by inlier features, the corresponding ${\bf D}_j$ would be rank deficient, and numerically would have lower nuclear norm. Conversely, the nuclear norm would be higher. We thus consider using the measure of $\| {\bf D}_j \|_*$, $j = 1, \dots, n$, for estimating the number of inliers. More specifically, given a data set of $K$ groups of feature vectors, we solve a series of ROML problems (via Algorithm \ref{MainAlgm}) with values of $n$ set increasingly from $n = 1$ to $n = n_k$. For any specific value of $n$, a matrix ${\bf D} \in \mathbb{R}^{dn\times K}$ can be obtained when Algorithm \ref{MainAlgm} converges. Define {\it maximal nuclear norm of feature-level correspondences} as
\begin{equation}\label{EqnMaxPerFeaNuclearNorm}
\gamma_n = \max_{j=1, \dots, n} \| {\bf D}_j \|_* ,
\end{equation}
which computes maximum of the nuclear norms of the established $n$ feature-level correspondences $\{ {\bf D}_j \}_{j=1}^n$. Given the series $\{ \gamma_1, \dots, \gamma_n \}$ obtained by increasingly setting $n = 1$ to the current $n$ value, we further define $\bar{\gamma}_n = \sum_{i = 1}^n \gamma_i  / n$, which averages this series. Based on the notions of $\gamma_n$ and $\bar{\gamma}_n$, we propose a simple scheme that estimates the number of inliers as the current value of $n$ if
\begin{equation}\label{EqnInlierNumEst}
(\gamma_{n+1} - \bar{\gamma}_n) / \bar{\gamma}_n > \delta ,
\end{equation}
where $\delta$ is a scalar parameter. To investigate the efficacy of this proposed scheme, we plot in Figure \ref{InlierEstSimuFigs}-(b) the curves of $\gamma_n$ values from $n = 1$ to $n = n_k$, for those simulated experiments reported in Figure \ref{InlierEstSimuFigs}-(a). Figure \ref{InlierEstSimuFigs}-(b) shows that for any curve corresponding to one of these simulated experiments, there is a clear stepping of $\gamma_n$ values when $n$ is set as the ground truth number of inliers, indicating that $\gamma_n$ is a good measure for estimation of inlier numbers. Our use of $\bar{\gamma}_n$ in (\ref{EqnInlierNumEst}) is to improve the estimation robustness. We also observe from Figure \ref{InlierEstSimuFigs}-(b) that the gap of curve stepping becomes narrower when the ratio of sparse errors is increased in each feature vector, suggesting less efficacy of the proposed scheme (\ref{EqnInlierNumEst}). This is reasonable since when errors in feature vectors increase, it becomes less distinctive between inlier and outlier features.

Simply put, our proposed scheme starts from setting $n = 1$, and solves a series of ROML problems with increasing values of $n$, until the condition (\ref{EqnInlierNumEst}) is satisfied. The ratio computed in (\ref{EqnInlierNumEst}) is numerically stable w.r.t. varying values of $d$ (feature dimension) and $K$ (number of groups/images). We set the parameter $\delta = 0.05$ throughout this paper. It gives perfect estimation of the true numbers of inliers for all the simulated experiments reported in this section. In Section \ref{ExpCaltech101MSRCSec}, we report experiments when applying this scheme to practical data.

\subsubsection{Detection of the True Inliers}
\label{InlierDetectSec}

By solving the problem (\ref{EqnNuclearL1Form}), ROML is able to identify $n$ features from each of a set of $K$ images and establish their correspondences. However, these $nK$ features are not necessarily all the true inliers. This is particularly the case when some true inlier features are contaminated with noise, e.g., due to object appearance variations caused by illumination or pose changes, and when some inliers are missing, e.g., due to partial occlusion of object instances. Detecting the true inliers out of the $nk$ features is practically useful for applications such as 3D reconstruction and object recognition.

We have introduced an error term $\mathbf{E} \in \mathbb{R}^{dn\times K}$ in (\ref{EqnNuclearL1Form}) to improve the robustness of ROML against the aforementioned contaminations of inlier features. One may think of using $\mathbf{E}$ for detection of the true inliers, e.g., by thresholding $\ell_1$-norms of the $nK$ $d$-dimensional subvectors in $\mathbf{E}$ ($n$ subvectors per column). However, the ROML formulation (\ref{EqnNuclearL1Form}) is a non-convex problem. It decomposes out sparse errors into $\mathbf{E}$ mainly for obtaining better PPMs $\{ \mathbf{P}^k \}_{k=1}^K$, so that potential inliers from the $K$ images can be identified and matched in ${\bf D} = [ \mathrm{vec}({\bf F}^{1}{\bf P}^{1}), \dots, \mathrm{vec}({\bf F}^{K}{\bf P}^{K}) ]$. By the non-convex problem nature of (\ref{EqnNuclearL1Form}), the obtained $\mathbf{E}$ is not guaranteed to be consistent with those noise contaminating inlier features, especially when inlier features are heavily contaminated and become less distinctive from outliers. This is different from the RPCA problem \cite{RPCA} where an exact solution of the low-rank matrix and sparse error matrix can be obtained with theoretical guarantee. To detect the true inliers from the selected features by ROML, we propose an alternative scheme that first solves a RPCA problem: $\min_{\mathbf{L}_{rpca}, \mathbf{E}_{rpca}} \| \mathbf{L}_{rpca} \|_* + \lambda_{rpca} \| \mathbf{E}_{rpca} \|_1 \ \mathrm{s.t.} \ \mathbf{D} = \mathbf{L}_{rpca} + \mathbf{E}_{rpca}$, using the obtained $\mathbf{D}$ from ROML as the input data, where $\lambda_{rpca} = 1/ \sqrt{dn}$ as theoretically derived in RPCA \cite{RPCA}. The proposed scheme then thresholds $\ell_1$-norms of the $nK$ $d$-dimensional error subvectors in $\mathbf{E}_{rpca} \in \mathbb{R}^{dn\times K}$, which are the decomposed error vectors respectively for the selected $nK$ features by ROML. For any $j^{th}$ selected feature from the $k^{th}$ image, we determine it as a true inlier if
\begin{equation}\label{EqnTrueInlierDetect}
\left\| [\mathbf{E}_{rpca}]_{jd-d+1:jd, k} \right\|_1 < \xi ,
\end{equation}
where $[\mathbf{E}_{rpca}]_{jd-d+1:jd, k}$ denotes the error subvector with $j \in \{1, \dots, n\}$ and $k \in \{1, \dots, K\}$, and $\xi$ is a scalar parameter. The scheme (\ref{EqnTrueInlierDetect}) is practically useful if magnitude of $\mathbf{E}_{rpca}$ is stable when applying ROML (and the subsequent RPCA problem) to various data applications, so that $\xi$ is less concerned with tuning. Note that magnitudes of entries in $\mathbf{E}_{rpca}$ are controlled by two factors: magnitudes of entries in the feature vectors $\{\mathbf{F}^k \}_{k=1}^K$, and levels of noise contaminating these feature vectors. Since all the feature vectors in $\{\mathbf{F}^k \}_{k=1}^K$ have been normalized in Algorithm \ref{MainAlgm}, magnitude of $\mathbf{E}_{rpca}$ is less influenced by the first factor. To investigate how the scheme (\ref{EqnTrueInlierDetect}) performs w.r.t. the second factor, we conducted simulated experiments using varying levels of noise and ratios of missing true inliers. Given unit $\ell_2$-norm of normalized feature vectors in $\{\mathbf{F}^k \}_{k=1}^K$, we set $\xi = 4$ for all relevant experiments of inlier detection reported in this paper.

More specifically, we randomly generated synthetic data similarly as did in Section \ref{InlierNumEstSec}. We set $K = 30$ and $n_k = 30$ for $k = 1, \dots, K$. Dimension of each feature vector was set as $d = 50$. Sparse errors were added to feature vectors in the same way as in Section \ref{InlierNumEstSec}. The number of inliers was set as $n = 10$, which gives a total of $nK = 300$ true inliers from the $K$ groups/images. After inlier features were generated, we randomly replaced a certain number of them with additionally generated outlier features, simulating the situation of missing inliers. In this investigation, we considered the test settings where the ratios of sparse errors in each feature vector were ranged in [$0.1$, $0.5$], simulating increased levels of noise, and the ratios of missing inliers for each group were ranged in [$0.05$, $0.5$]. Under each setting, performance of the scheme (\ref{EqnTrueInlierDetect}) is measured by {\it precision} and {\it recall}. Precision is computed as the number of detected true inliers divided by the total number of detected features, and recall is computed as the number of detected true inliers divided by the total number of true inliers contained in the $K$ groups. We run $5$ random trials and averaged the results under each setting. Table \ref{Table-DetectActualInliers-Simu} reports these simulated experiments, which suggests that precision scores of detecting the true inliers by the scheme (\ref{EqnTrueInlierDetect}) are very high for the considered test settings (in fact perfect precision for most of the settings). When the levels of noise (ratios of sparse errors) contaminating feature vectors increase, magnitudes of the decomposed error subvectors in $\mathbf{E}_{rpca}$ for the true inliers also increase, and become less distinctive from those of outliers, resulting in slightly reduced recall scores. Nevertheless, results in Table \ref{Table-DetectActualInliers-Simu} show that in a wide range of noise and missing inlier settings, the proposed scheme (\ref{EqnTrueInlierDetect}) is effective for detecting the true inliers. In Section \ref{ExpCaltech101MSRCSec}, we also report experiments of applying (\ref{EqnTrueInlierDetect}) to practical data for detection of true inliers.

\begin{table}[t]
\caption{ {\small Simulation for detection of true inliers, using the proposed scheme (\ref{EqnTrueInlierDetect}). Synthetic data are generated under $12$ test settings: the ratios of sparse errors in each feature vector are $10 \%$ (Noise Level I), $30 \%$ (Noise Level II), or $50 \%$ (Noise Level III), and the ratios of missing inliers in each group are $5 \%$, $10 \%$, $30 \%$, or $50 \%$. Results for each setting are obtained by averaging over $5$ random trials and presented in the format of {\it Precision/Recall}. } } \label{Table-DetectActualInliers-Simu}
\begin{center}
\begin{tabular}{ccccc}

\hline {\scriptsize Ratios of}       &        &         &         &          \\
       {\scriptsize miss. inliers} & $5 \%$ & $10 \%$ & $30 \%$ & $50 \%$  \\

\hline {\scriptsize Noise L. I} &  $1.00/1.00$  &  $1.00/1.00$  &  $1.00/1.00$  &  $1.00/1.00$   \\

\hline {\scriptsize Noise L. II} &  $1.00/1.00$  &  $1.00/1.00$  &  $1.00/0.99$  &  $0.99/0.98$  \\

\hline {\scriptsize Noise L. III} &  $1.00/0.99$  &  $1.00/0.96$  &  $1.00/0.95$ &   $0.99/0.92$   \\

\hline

\end{tabular}

\end{center}
\end{table}


\section{Choices of Feature Types and Their Applicable Spectrums}
\label{FeaDiscussSec}

In the previous sections, we have represented an image as a set of features, where features generally refer to vectors characterizing image points and local regions centered on them. The task of object matching is then posed as Problem 1. Depending on different applications, these features can be chosen as either image coordinates, local region descriptors, or combination of them encoding both spatially structural and local appearance information. In the following, we present details of different choices of feature types and their applicable spectrums for robust object matching.

\subsection{Image Coordinates}
\label{ImageCoordAsFeaSec}

Given a set of points in an image, their coordinates can be directly used as features. In fact, coordinates of a set of inlier points in an image encode geometric relations among them, and it is the geometric structure of these points that determines the object pattern, and also provides a constraint for use in object matching. Image coordinates based features have been intensively used in early shape matching works \cite{SLH,ShapiroBrady92,ICP,ChuiRangarajanNewPointMatching,BelongieShapeContext}.

For a moving rigid object in a video sequence or images of a rigid object captured from different viewpoints, denote ${\bf f}_{i}^{k} = [x_{i}^{k}, y_{i}^{k}]^{\top} \in \mathbb{R}^{2}$, $i = 1, \dots, n_{k}$, as image coordinates based $n_{k}$ features extracted from the $k^{th}$ image. Let $ {\bf F}^{k} = [ {\bf f}_{1}^{k}, \dots, {\bf f}_{n_{k}}^{k} ] \in \mathbb{R}^{2\times n_{k} }$. It has been shown in \cite{Rank4FactorShapeFromMotion} that the matrix, defined by \begin{equation}\label{EqnCoordDRank4Form} {\bf D}'( \{ \mathbf{P}^k \}_{k=1}^K ) = \big[({\bf F}^{1}{\bf P}^{1})^{\top}, \dots, ({\bf F}^{K}{\bf P}^{K})^{\top}\big]^{\top} \in \mathbb{R}^{2K\times n},\end{equation} is highly rank deficient (at most rank $4$ when considering translation and there is no measurement noise), if correct PPMs $\{ {\bf P}^{k} \}_{k=1}^{K}$ are used so that $n$ inlier points can be selected from each of $\{ {\bf F}^{k} \}_{k=1}^{K}$ and corresponding points $\{ {\bf f}_{j}^{k} \}_{k=1}^{K}$, $j \in \{1, \dots, n \}$, can be aligned in the same column of ${\bf D}'$. (\ref{EqnCoordDRank4Form}) is different from the formation of ${\bf D}(\{ \mathbf{P}^k \}_{k=1}^K)$ in (\ref{EqnDMatrix}). By applying the same low-rank and sparse constraints as in (\ref{EqnNuclearL1Form}), we will show in Section \ref{ExpHotelSeqSec} that image coordinates based features are very useful for matching rigid objects.

\subsection{Local Region Descriptors}

It is also straightforward to use region descriptors characterizing locally visual appearance information as features. These include SIFT \cite{SIFT}, HOG \cite{HOG}, Geometric Blur \cite{GeometricBlur,BergShapeMatchObjectRecog}, GIST \cite{GIST}, or even raw pixels of local patches. In general, these feature descriptors have the properties of invariance and distinctiveness. The invariance property makes it possible to match salient features extracted from images under geometric transformation or illumination change, while feature distinctiveness is important to differentiate between different salient regions. Features of such kind can be used in scenarios where they are discriminative enough for matching, or geometric constraints between feature points are not available, such as common object localization \cite{ObjectnessForUOCD,bMCL}. In Section \ref{ExpCOLSec}, we present how ROML can be applied to this application.

\subsection{Combination of Image Coordinates and Region Descriptors}
\label{CoordDescriptorCombSec}

Local region descriptors alone could be ambiguous for feature matching when there exist repetitive textures or less discriminative local appearance in images. To improve the matching accuracy, it is necessary to exploit the geometric structure of inlier points that consistently appears in each of the set of images. In literature, there are many ways to exploit such geometric constraints, such as pair-wise compatibility of feature correspondences used in graph matching \cite{LeordeanuSpectralCorrespondence,TorresaniModelAndGlobalOpt}, or linear-form constraints benefiting from a template image \cite{HongshengLiConvexRelax,HaoJiangLinearSolution}. In this work, we consider a simple method introduced in \cite{OneShot}. For any interest point in each of the set of images, this method learns a low-dimensional embedded feature vector that combines information of both the local appearance and the spatial relations of this point relative to other points in the image. We present details of how to compute this type of learned embedded features in Appendix \ref{appendix_LowDimEmbedFeaLearning}.

As suggested by Theorem \ref{ExactRelaxTheorem}, when there are no outliers, the thus learned features can be directly used in our ROML framework. When there exist outliers in any of the set of images, we can always normalize those features to let them have constant $\ell_{2}$-norm, and our method still applies. Since this type of learned features encode both appearance and spatial layout information, our method can potentially apply in more general settings, such as matching of non-rigid, articulated objects, or instances of a same object category. Experiments in Section \ref{ExpObjMatching} show the promise.

\section{Experiments}
\label{ExpObjMatching}

In this section, we present experiments to show the effectiveness of ROML for robustly matching objects in a set of images. We consider different testing scenarios from the relatively simple rigid object matching, to the more challenging matching of object instances of a common category, and matching a non-rigid object moving in a video sequence. For these testing scenarios, we choose appropriate feature types of either image coordinates or combination of image coordinates and local region descriptors, while features of region descriptors alone will be used in Section \ref{ExpCOLSec} for the application of common object localization. In the following experiments, without mentioning we always set the penalty parameter $\lambda = 5 / \sqrt{dn}$ when solving the ROML problem (\ref{EqnNuclearL1Form}) using Algorithm \ref{MainAlgm}, where $\rho$ was initially set as $1\mathrm{e}^{-4}$ and iteratively increased with a factor of $1.001$.

\subsection{Rigid Object with 3D Motion}
\label{ExpHotelSeqSec}

\begin{table*}[t]
\caption{ {\small Results of different methods on the ``Hotel'' sequence. Accuracies are measured by the Match Ratio criteria.} } \label{Table-HotelMainComp}
\begin{center}
\begin{tabular}{ccccccccc}
\hline {\footnotesize Methods} & {\footnotesize DD \cite{TorresaniModelAndGlobalOpt}} & {\footnotesize SMAC \cite{SMAC}} & {\footnotesize LGM \cite{LearningGraphMatching}} & {\footnotesize RankCon \cite{OliveiraRankConstraint}} & {\footnotesize One-Shot \cite{OneShot}} & {\footnotesize Prev \cite{ZinanECCV}} & {\footnotesize ROML-Pair} & {\footnotesize ROML} \\

\hline {\footnotesize Accuracies} & { $99.8 \%$} & { $84 \%$} & { $90 \%$} & { $57 \%$} & { ${\bf 100 \%}$} & { $72 \%$} & { ${\bf 100 \%}$} & { ${\bf 100 \%}$}  \\

\hline

\end{tabular}

\end{center}
\end{table*}

The CMU ``Hotel'' sequence consists of $101$ frames of a toy hotel building undergoing 3D motion \cite{CMUHotel}. Each frame has been manually labelled with the same set of $30$ landmark points \cite{LearningGraphMatching}, i.e., $n = 30$. We use the ``Hotel'' sequence to show that ROML can be applied using image coordinates as features for matching rigid objects. In particular, we sampled $K = 15$ frames out of the total $101$ frames (every $7$ frames), in order to simulate the wide baseline matching scenario. Given PPMs $\{ \mathbf{P}^k \}_{k=1}^K$, image coordinates of landmark points in these $15$ frames were arranged into a matrix ${\bf D}' (\{ \mathbf{P}^k \}_{k=1}^K) \in \mathbb{R}^{2K\times n}$, as defined in (\ref{EqnCoordDRank4Form}). We used Algorithm \ref{MainAlgm} to optimize a PPM for each frame, where the penalty parameter was set as $\lambda = 5 / \sqrt{2K}$, and $\rho$ was initialized as $1\mathrm{e}^{-6}$ and iteratively increased with a factor of $1.0001$.

We compare our method with representative pair-wise graph matching methods including Dual Decomposition (DD) \cite{TorresaniModelAndGlobalOpt}, SMAC \cite{SMAC}, and Learning Graph Matching (LGM) \cite{LearningGraphMatching}, which are based on either linear or quadratic assignment formulations, and also with more related methods \cite{OliveiraRankConstraint,OneShot} that are able to simultaneously match the set of $15$ frames. For the former set of methods, matchings between a total of $105$ frame pairs need to be established. Note that although all these methods are based on image coordinates, many of them have used the advanced shape context features \cite{BelongieShapeContext}. To evaluate the performance of different methods, we use the Match Ratio criteria, which is defined in Section \ref{InlierNumEstSec}. Table \ref{Table-HotelMainComp} reports the Match Ratios of different methods, where results of SMAC and LGM are from \cite{TorresaniModelAndGlobalOpt,OneShot}, and result of RankCon \cite{OliveiraRankConstraint} is from our own code implementation of the method. Table \ref{Table-HotelMainComp} tells that ROML and One-Shot \cite{OneShot} achieve the best performance (no matching error). However, One-Shot \cite{OneShot} uses shape context feature to characterize each landmark point, and it performs a low-dimensional embedding combining information of both geometric structure and local descriptors of landmark points, while ROML just directly uses image coordinates. RankCon \cite{OliveiraRankConstraint} also exploits low-rank constraints, however, its performance is much worse than that of ROML. This is due to the inherent error-propagation nature of the method. Matchings of landmark points in RankCon \cite{OliveiraRankConstraint} are established in a frame-by-frame manner; once landmarks in the current frame are matched to those in the previous frames, the matchings cannot be corrected in the later stages, and matching errors will inevitably propagate and accumulate. Thus for RankCon \cite{OliveiraRankConstraint}, it is critical to make the matchings of the first few frames accurate. Unfortunately, the low-rank property leveraged by RankCon \cite{OliveiraRankConstraint} is weak given only landmark points in the first few frames. We also present result of our previous method \cite{ZinanECCV} (Prev \cite{ZinanECCV}) in Table \ref{Table-HotelMainComp}, which corroborates our discussions on the advantage of ROML over \cite{ZinanECCV} in Section \ref{IntroSec}. It is interesting to observe that under the setting of pair-wise matching as for methods \cite{TorresaniModelAndGlobalOpt,SMAC,LearningGraphMatching}, ROML still performs perfectly on the ``Hotel'' sequence (``ROML-Pair'' in Table \ref{Table-HotelMainComp}). In fact, since the problem size of matching each frame pair is smaller, ROML converges faster with less iterations.


Compared to \cite{OliveiraRankConstraint,OneShot}, ROML has the additional advantage of being more robust against missing inliers. To verify, we performed another experiment by removing randomly selected landmark points in each frame. For each removed landmark point, we also generated arbitrary image coordinates for it and made sure the generated coordinates were far enough away from the true ones, in order to fit with the algorithmic settings of these comparative methods. We set $\lambda = 2 / \sqrt{2K}$ in Algorithm \ref{MainAlgm}. One-Shot \cite{OneShot} uses k-means clustering to obtain feature correspondences in the learned feature space. We chose its best-performing dimensionality of learned features, and run $10$ trials of k-means clustering and averaged the results. Parameters of RankCon \cite{OliveiraRankConstraint} has also been tuned to its best performance. Table \ref{Table-HotelMissingPtsRobustTest} reports the Match Ratio results, which are computed over the non-missing points. Table \ref{Table-HotelMissingPtsRobustTest} clearly shows that ROML is less influenced when there exist missing inlier points.

\begin{table}[t]
\caption{ {\small Match Ratios of different methods on the ``Hotel'' sequence with missing landmark (inlier) points. The cases of $1$, $3$, and $5$ missing points in each frame are tested (the corresponding ratios of missing points over the total $30$ ones are $3 \%$, $10 \%$, and $17 \%$ respectively). } } \label{Table-HotelMissingPtsRobustTest}
\begin{center}
\begin{tabular}{|c|c|c|c|} \hline

{\scriptsize No. of missing} & {} & {} & {}  \\
{\scriptsize points} & {\scriptsize RankCon \cite{OliveiraRankConstraint}} & {\scriptsize One-Shot \cite{OneShot}} & {\scriptsize ROML}  \\

\hline { $1$} & { $43 \%$} & { $76 \%$} & { ${\bf 95 \%}$}  \\

\hline { $3$} & { $26 \%$} & { $64 \%$} & { ${\bf 79 \%}$} \\

\hline { $5$} & { $23 \%$} & { $59 \%$} & { ${\bf 71 \%}$} \\

\hline

\end{tabular}

\end{center}
\end{table}

\subsection{Object Instances of a Common Category}
\label{ExpCaltech101MSRCSec}

In this section, we test how ROML performs to match object instances belonging to a same object category. We used $6$ image sets of different categories from Caltech101 \cite{Caltech101}, MSRC \cite{MSRC}, and the Internet. Numbers of images in these $6$ sets ranged from $16$ to $25$. For each image, interest points were detected by SIFT: the numbers of detected interest points per image were from $27$ to $174$, out of which we manually labelled inlier points as matching ground truth. When some inlier points were not detected by SIFT in some images, we also manually labelled them in order to produce consistent sets of inlier points across the sets of images. In these experiments, we chose the low-dimensional embedded feature representation \cite{OneShot}, as explained in Section \ref{CoordDescriptorCombSec}, which encodes information of both geometric structure and descriptor similarity. To learn the embedded features, we used Geometric Blur descriptors \cite{BergShapeMatchObjectRecog} to characterize local regions around interest points, and Euclidean distances between points in each image for measuring geometric relations. The parameters for embedded feature learning were set as $\sigma_{spa.} = 10$ and $\sigma_{des.} = 0.2$ (cf. Appendix \ref{appendix_LowDimEmbedFeaLearning} for the definition of $\sigma_{spa.}$ and $\sigma_{des.}$), and the dimensionality of learned features was set as $d = 60$. Feature matching was realized by solving (\ref{EqnNuclearL1Form}) using Algorithm \ref{MainAlgm}.

We compare our method with One-Shot \cite{OneShot}, and also with several recent graph matching methods including DD \cite{TorresaniModelAndGlobalOpt}, RRWM \cite{GraphReweightedRandomWalk}, and SM \cite{LeordeanuSpectralCorrespondence}, and hyper-graph matching methods including TM \cite{TensorHighOrderGraphMatching}, RRWHM \cite{HyperGraphReweightedRandomWalk}, and ProbHM \cite{ZassShashua}. Codes of DD \cite{TorresaniModelAndGlobalOpt}, RRWM \cite{GraphReweightedRandomWalk}, TM \cite{TensorHighOrderGraphMatching}, and RRWHM \cite{HyperGraphReweightedRandomWalk} are respectively from their authors' publicly available webpages.  For ProbHM \cite{ZassShashua}, we used a code implementation provided by the authors of RRWHM \cite{HyperGraphReweightedRandomWalk}. The simple methods of One-Shot \cite{OneShot} and SM \cite{LeordeanuSpectralCorrespondence} are based on our own implementations. For One-Shot, we chose its best-performing dimensionality of the learned embedded features as our method used, and run $10$ trials of k-means clustering and averaged the results. For pair-wise graph matching methods \cite{TorresaniModelAndGlobalOpt,GraphReweightedRandomWalk,LeordeanuSpectralCorrespondence,TensorHighOrderGraphMatching,HyperGraphReweightedRandomWalk,ZassShashua}, we generated a total of $\frac{K!}{(K-2)!2!}$ image pairs for each test set with $K$ images. These graph matching methods characterize interest points in each image by both their spatial relations and their respective local region descriptors. For a fair comparison, we used the same Geometric Blur region descriptors as our method used \footnote{We have also tried the SIFT descriptor \cite{SIFT} to characterize appearance of local regions around interest points. The matching accuracies using SIFT were slightly worse than those using Geometric Blur for both our method and these comparative methods.}. Their parameters were also tuned to their respective best performance on the $6$ test sets.

Table \ref{Table-Caltech101MSRCMainMatchRatioResult} reports results of different methods in terms of Match Ratio. Example feature correspondences for DD \cite{TorresaniModelAndGlobalOpt} and our method are shown in Figure \ref{Caltech101MSRCExamImgShowFigs}. Table \ref{Table-Caltech101MSRCMainMatchRatioResult} and Figure \ref{Caltech101MSRCExamImgShowFigs} suggest that for the relatively simple ``Airplane'', ``Motorbike'', and ``Face'' image sets, our method gives very good matching results. The ``Car'', ``Bus'', and ``Bank of America (BoA)'' sets are more difficult due to the cluttered background, large viewpoint changes, or intra-category variations between different instances. Our method still gives reasonably good and consistent matching results. Both One-Shot and our method can match multiple images simultaneously. Our results are much better than those of One-Shot, which shows that One-Shot cannot perform well in the presence of outliers, and also that our ROML formulation optimized by the ADMM method is very effective for multi-image feature matching. Our method greatly outperforms graph and hyper-graph matching methods. It demonstrates that leveraging more object pattern constraints (i.e., geometric and feature similarity constraints) from multiple images is very useful for feature matching. Moreover, Figure \ref{Caltech101MSRCExamImgShowFigs} suggests that our matching results across the $4$ images are more consistent than those from graph matching methods: another desired property for many computer vision applications. In Table \ref{Table-Caltech101MSRCMainMatchRatioResult}, we also compare with our previous method \cite{ZinanECCV} (Prev \cite{ZinanECCV} in Table \ref{Table-Caltech101MSRCMainMatchRatioResult}). Results of Prev \cite{ZinanECCV} are obtained using the same low-dimensional embedded features as ROML does. Table \ref{Table-Caltech101MSRCMainMatchRatioResult} tells that on all the $6$ image sets, matching accuracies of ROML are much better than those of Prev \cite{ZinanECCV}. The improved accuracy comes from the new way of PPM optimization, i.e., exactly solving an equivalent LSAP in the present paper instead of sub-optimally solving two costly subproblems in \cite{ZinanECCV}, as we have discussed in Section \ref{IntroSec}.

\begin{table*}[t]
\caption{ {\small Match Ratios of different methods on $6$ image sets of different object categories.} } \label{Table-Caltech101MSRCMainMatchRatioResult}
\begin{center}
\begin{tabular}{cccccccccc}
\hline {\scriptsize Methods} & {\scriptsize RRWM \cite{GraphReweightedRandomWalk} } & {\scriptsize SM \cite{LeordeanuSpectralCorrespondence} } & {\scriptsize TM \cite{TensorHighOrderGraphMatching} } & {\scriptsize RRWHM \cite{HyperGraphReweightedRandomWalk} } & {\scriptsize ProbHM \cite{ZassShashua} } & {\scriptsize DD \cite{TorresaniModelAndGlobalOpt} } & {\scriptsize OneShot \cite{OneShot} } & {\scriptsize Prev \cite{ZinanECCV} }  & {\scriptsize ROML } \\

\hline {\scriptsize Airplanes} & { $28 \%$} & { $54 \%$} & { $17 \%$} & { $54 \%$} & { $32 \%$} & { $70 \%$} & { $65 \%$} & { $87 \%$} & { ${\bf 95 \%}$}  \\

\hline {\scriptsize Face} & { $40 \%$} & { $57 \%$} & { $26 \%$} & { $54 \%$} & { $14 \%$} & { $64 \%$} & { $61 \%$} & { $53 \%$} & { ${\bf 89 \%}$}  \\

\hline {\scriptsize Motorbike} & { $50 \%$} & { $46 \%$} & { $23 \%$} & { $58 \%$} & { $28 \%$} & { $73 \%$} & { $68 \%$} & { $89 \%$} & { ${\bf 99 \%}$}  \\

\hline {\scriptsize Car} & { $26 \%$} & { $39 \%$} & { $12 \%$} & { $23 \%$} & { $12 \%$} & { $51 \%$} & { $50 \%$} & { $59 \%$} & { ${\bf 81 \%}$}  \\

\hline {\scriptsize Bus} & { $13 \%$} & { $25 \%$} & { $24 \%$} & { $43 \%$} & { $18 \%$} & { $52 \%$} & { $44 \%$} & { $64 \%$} & { ${\bf 79 \%}$}  \\

\hline {\scriptsize BoA} & { $7 \%$} & { $12 \%$} & { $6 \%$} & { $15 \%$} & { $7 \%$} & { $12 \%$} & { $16 \%$} & { $35 \%$} & { ${\bf 75 \%}$}  \\

\hline

\end{tabular}

\end{center}
\end{table*}

\begin{figure*}[t]
\centering

\begin{tabular}{ccc}

\includegraphics[scale=0.1125]{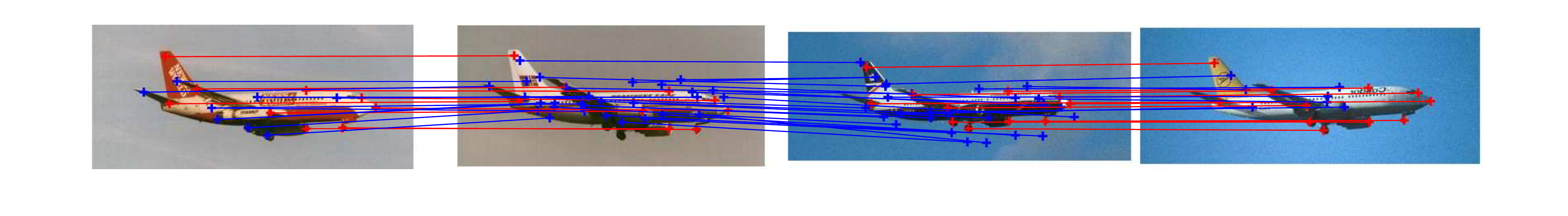} & \includegraphics[scale=0.11]{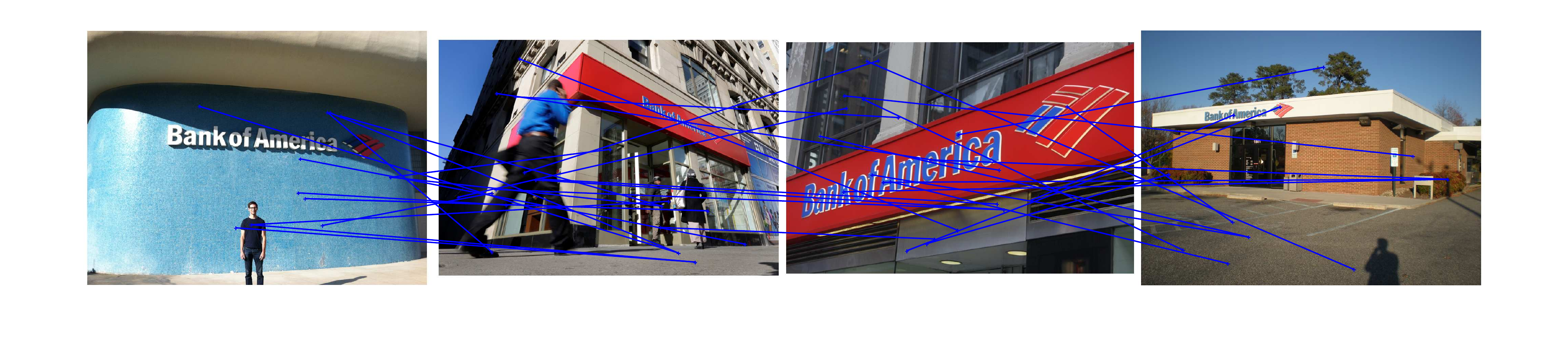} & \includegraphics[scale=0.1125]{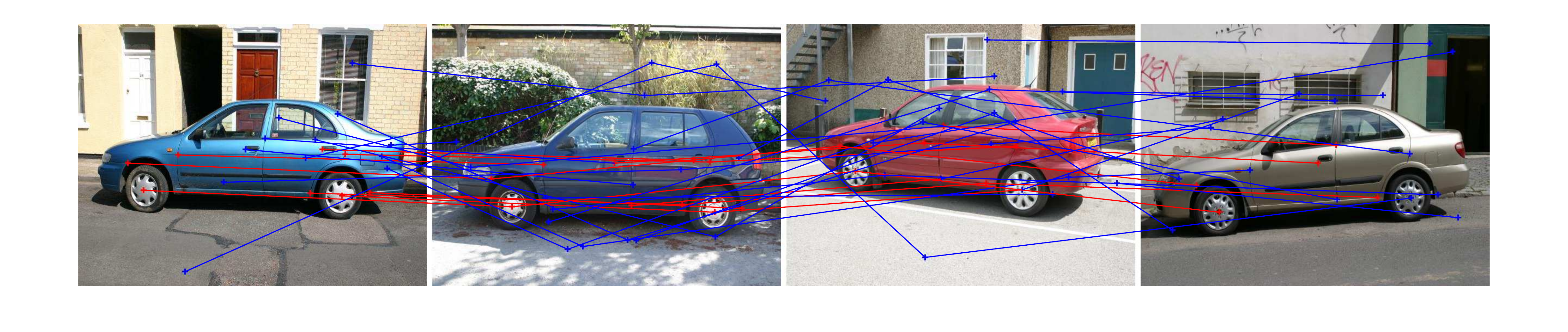} \vspace{-0.2cm} \\
\includegraphics[scale=0.11]{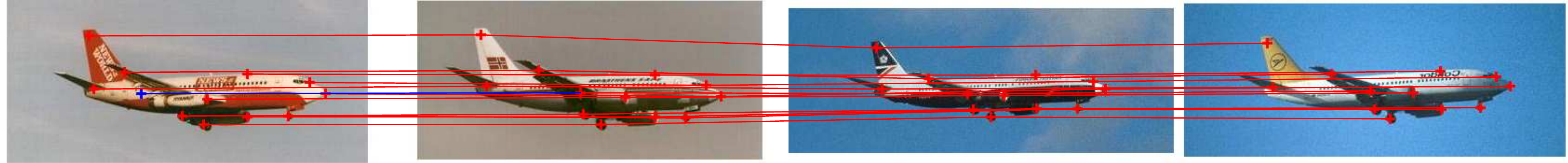} & \includegraphics[scale=0.11]{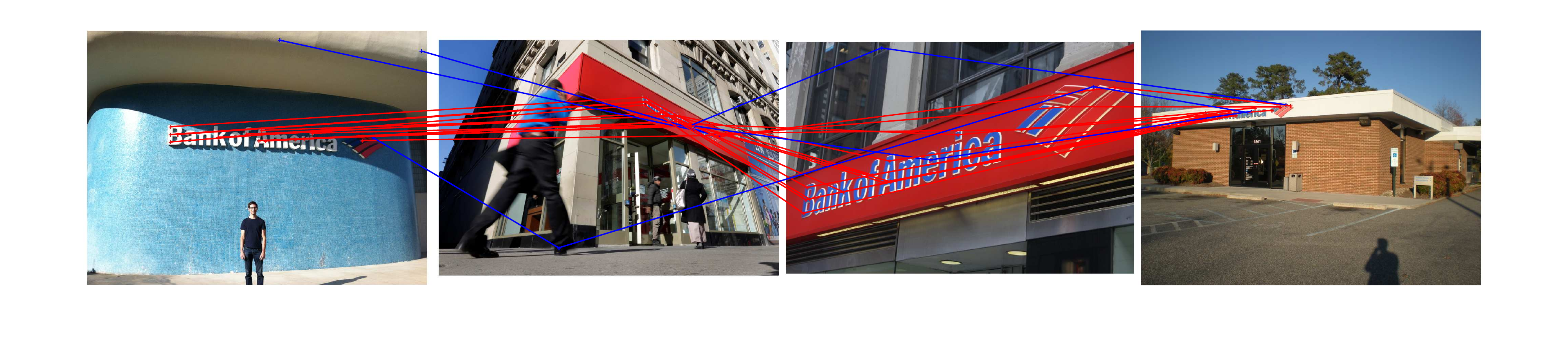} & \includegraphics[scale=0.11]{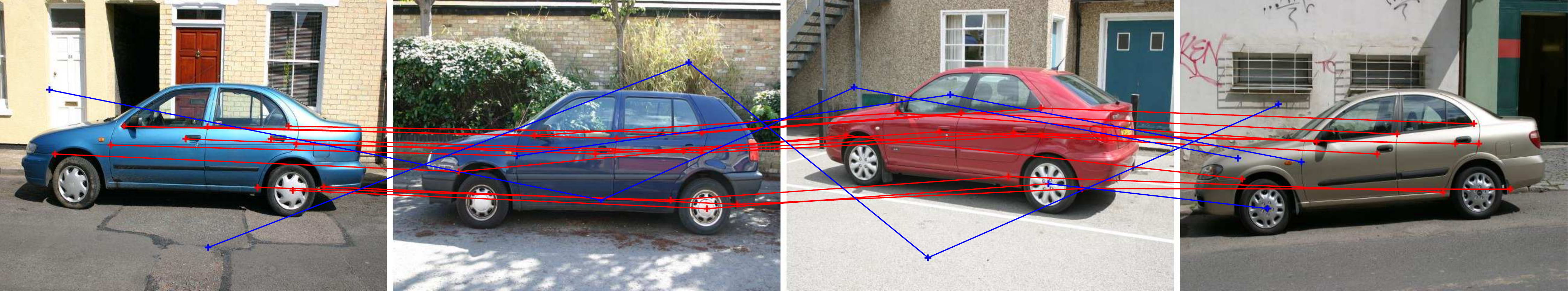} \\

\includegraphics[scale=0.15]{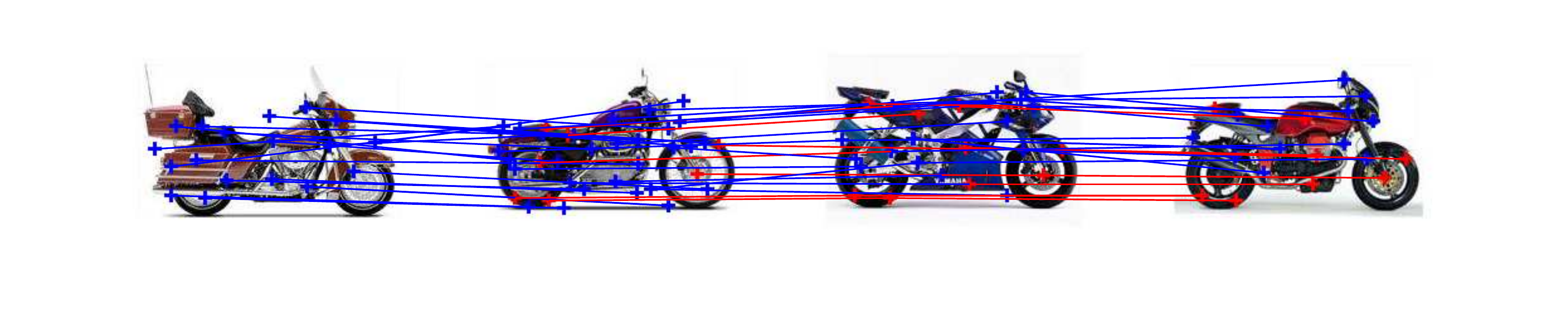} & \includegraphics[scale=0.12]{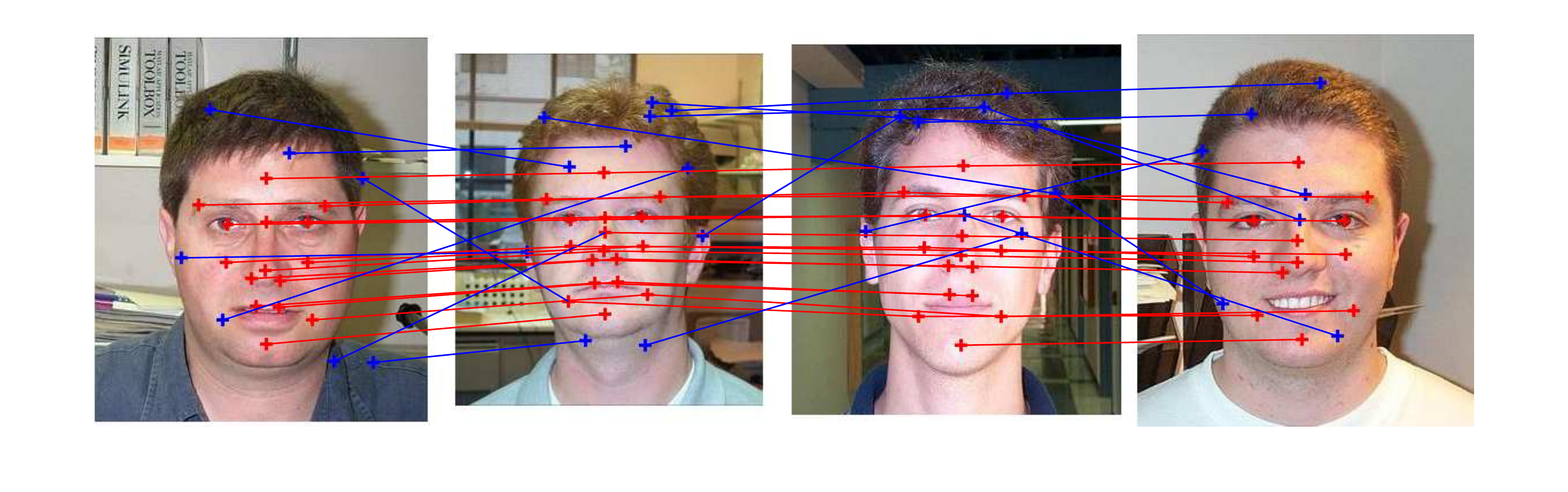} & \includegraphics[scale=0.16]{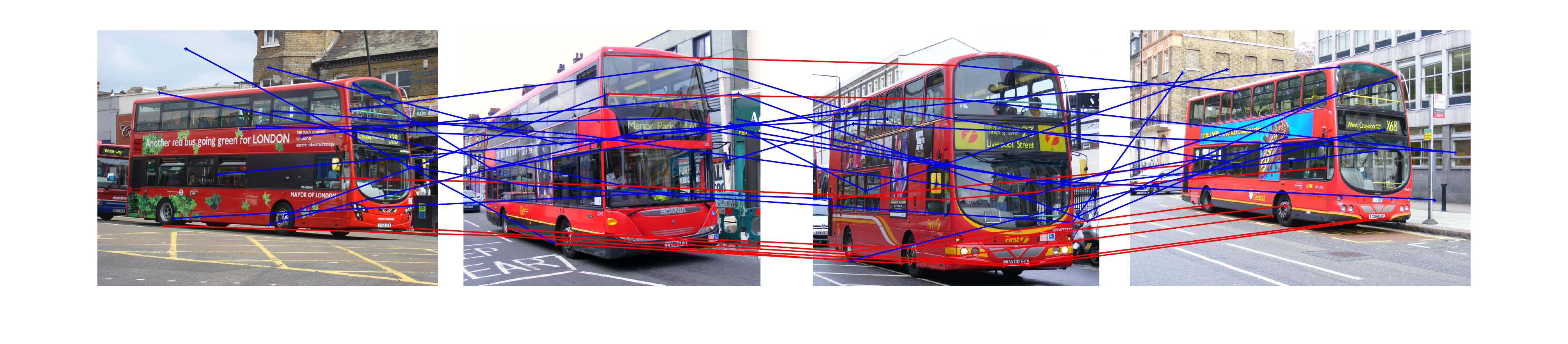} \vspace{-0.3cm} \\
\includegraphics[scale=0.15]{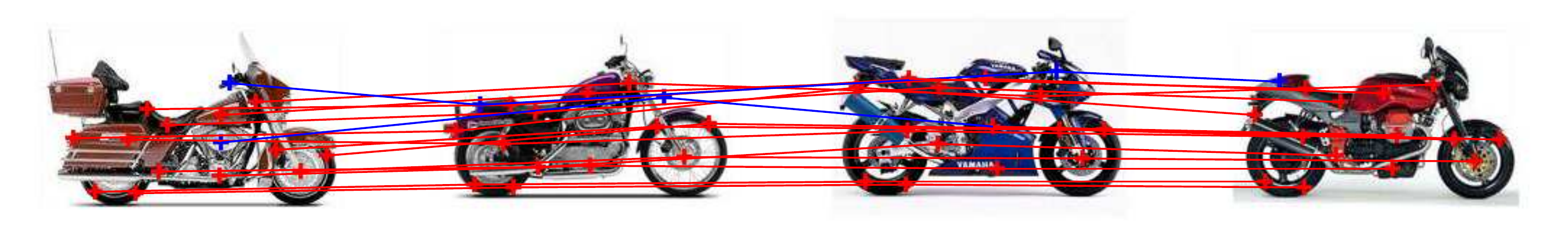} & \includegraphics[scale=0.12]{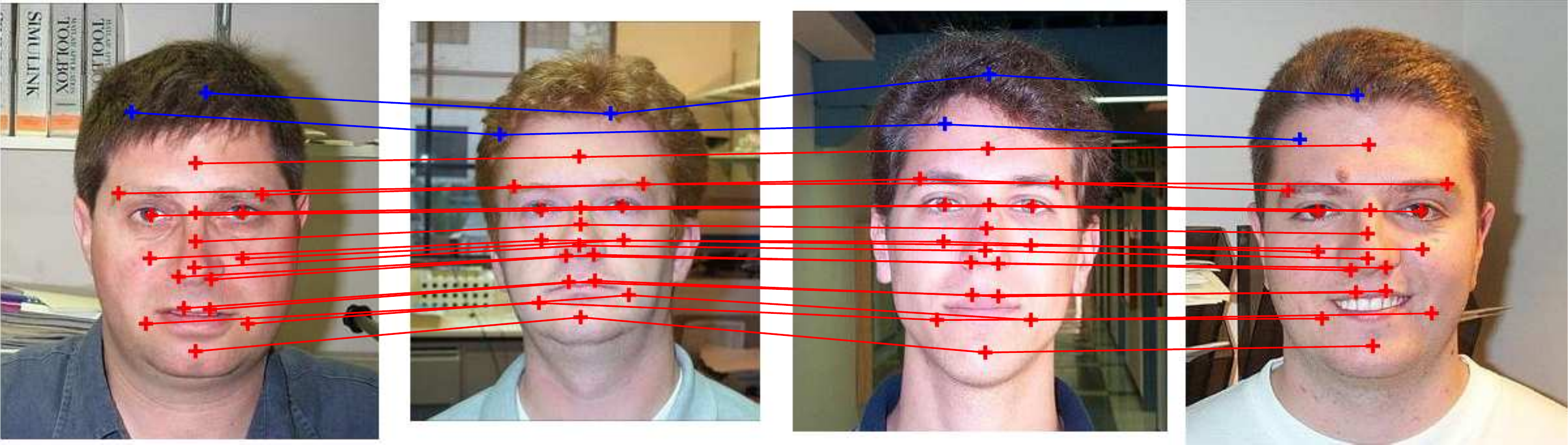} & \includegraphics[scale=0.16]{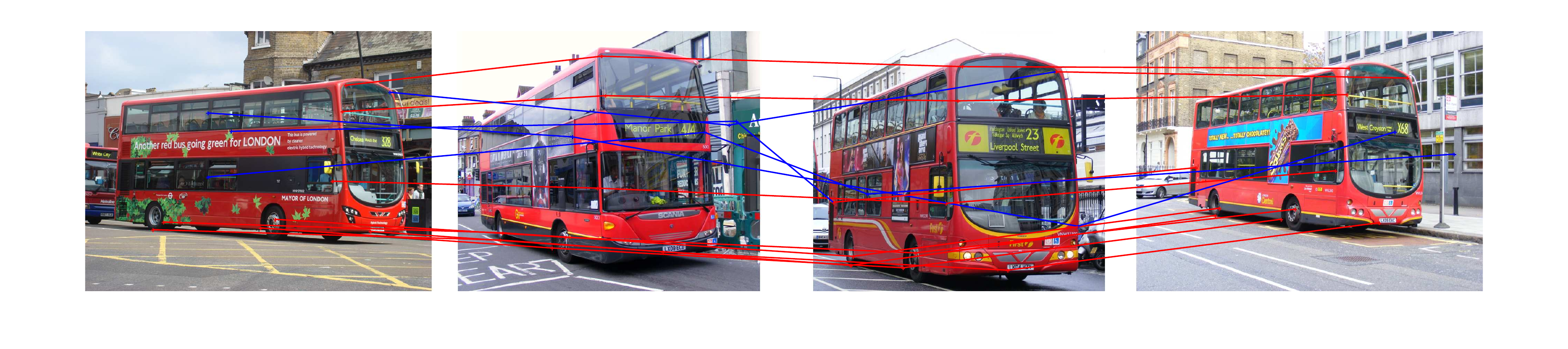}  \\

\end{tabular}

\caption{ {\small Example feature correspondences among $4$ images for different image sets. For every pair top is from DD \cite{TorresaniModelAndGlobalOpt}, and bottom is from our method. Red lines represent identified ground truth correspondences, and blue lines are for false ones.} }  \label{Caltech101MSRCExamImgShowFigs}
\end{figure*}

Different choices of dimension $d$ in low dimensional feature learning may influence our method's performance. In Figure \ref{PerformVsFeaDimAndImgNumFigs}-(a), we plot our matching accuracies with different choices of $d$ on these $6$ test sets. It shows that better results can generally be obtained when $d \geq 40$. It is expected that our method performs well only when the size of image sets (the $K$ value) is relatively large. In Figure \ref{PerformVsFeaDimAndImgNumFigs}-(b), we plot results of our method on the $6$ test sets with different choices of $K$. It shows that when $K > 10$, our method can stably get good results, which confirms that simultaneously matching a set of images is very useful for robust object matching.

\begin{figure}[t]
\centering

\includegraphics[scale=0.21]{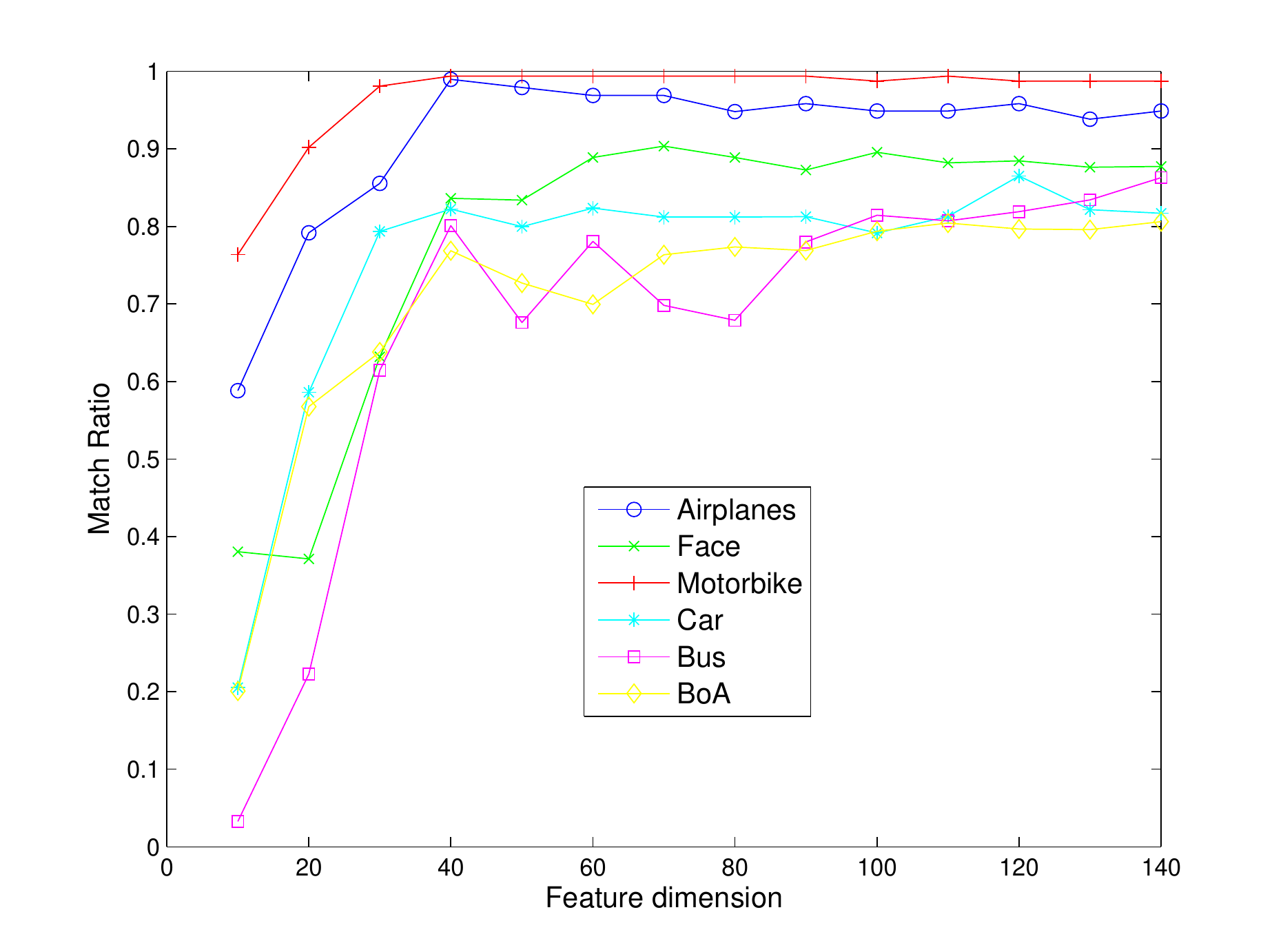}
\includegraphics[scale=0.21]{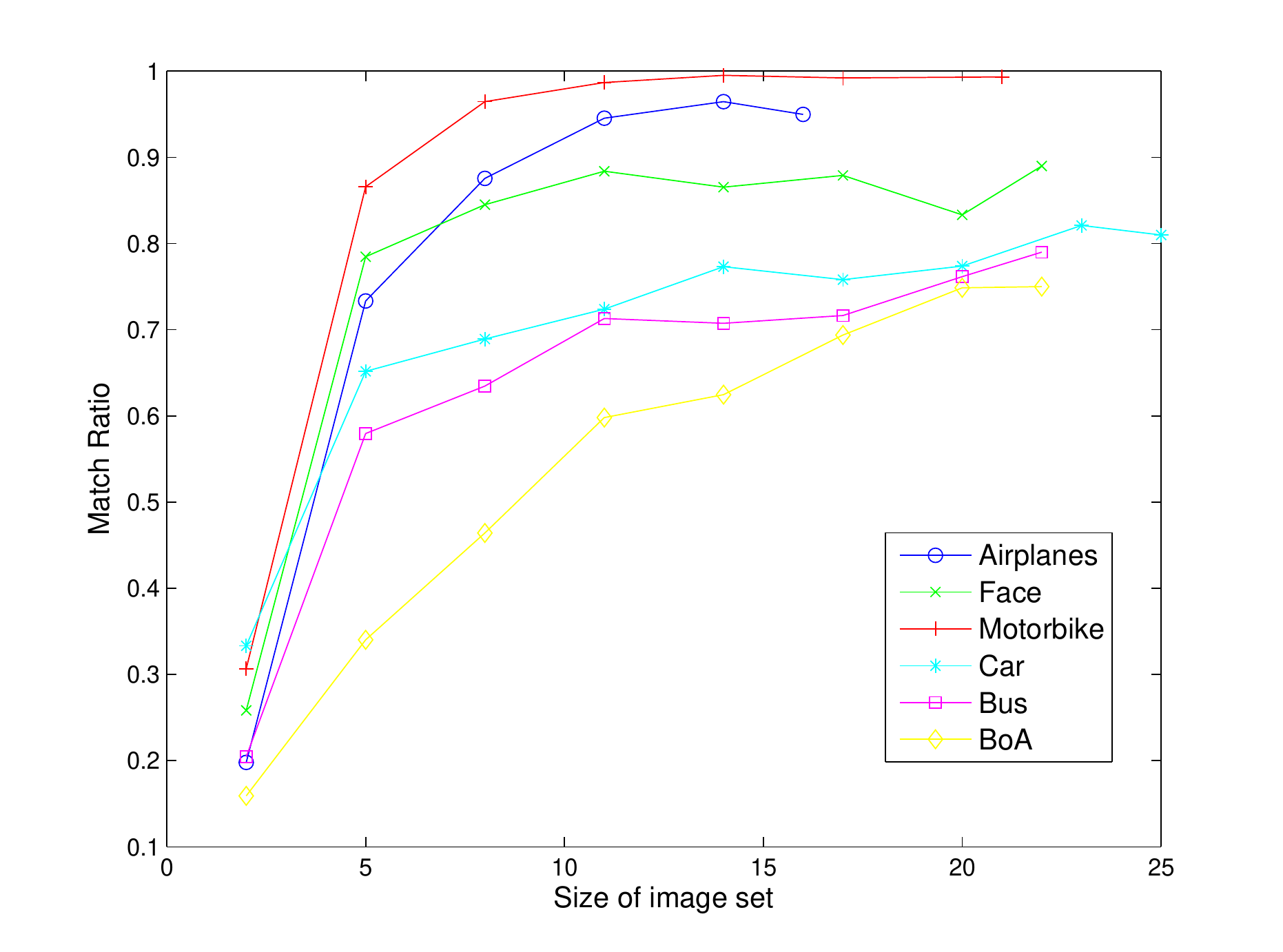} \\
\hfill (a) \hfill\hfill (b)  \hfill\hfill

\caption{ {\small Performance of ROML using (a) different choices of dimension $d$ in low-dimensional feature learning and (b) different sizes of image sets (the $K$ values).} }\label{PerformVsFeaDimAndImgNumFigs}
\end{figure}

Except for matching accuracy, one may also be interested in comparing matching efficiency of different methods. We have analyzed the computational complexity of our proposed method in Section \ref{ComplexityAnalysisSec}. In Table \ref{Table-Caltech101MSRCMainCompTimeResult}, we report practical computation time of different methods for those experiments reported in Table \ref{Table-Caltech101MSRCMainMatchRatioResult}. These experiments were conducted on an Intel Xeon CPU running at 2.8GHz, using Matlab implementation of different methods. Table \ref{Table-Caltech101MSRCMainCompTimeResult} suggests that ROML is much more efficient than the best-performing graph matching method DD \cite{TorresaniModelAndGlobalOpt}, and is slightly slower than other graph/hyper-graph matching methods. One-Shot \cite{OneShot} is very fast, however, its accuracy (reported in Table \ref{Table-Caltech101MSRCMainMatchRatioResult}) is not satisfactory. As an improved method of our previous work \cite{ZinanECCV}, ROML is much more efficient than \cite{ZinanECCV} as well. The improved efficiency is again due to the new way of PPM optimization in the present paper. In spite of this improved efficiency, most of ROML's computation is still on solving LSAPs for updating the set of $K$ PPMs (steps $5 \sim 7$ in Algorithm \ref{MainAlgm}), which concerns with $K$ independent subproblems and are fully parallelizable. When implementing the PPM optimization steps in parallel (ROML-Parallel in Table \ref{Table-Caltech101MSRCMainCompTimeResult}), efficiency of ROML is further improved.

\begin{table*}[t]
\caption{ {\small Computation time (seconds) of different methods on the $6$ image sets used in Table \ref{Table-Caltech101MSRCMainMatchRatioResult}. All experiments were conducted on an Intel Xeon CPU running at 2.8GHz, using Matlab implementation of different methods. } } \label{Table-Caltech101MSRCMainCompTimeResult}
\begin{center}
\begin{tabular}{lrrrrrrrrrrrrr}

\hline {\small Methods} & & {\small Airplanes} & & {\small Face} & & {\small Motorbike} & & {\small Car} & & {\small Bus} & & {\small BoA} & \\

\hline {\scriptsize RRWM \cite{GraphReweightedRandomWalk} }       & &  $22.92$  & &  $211.10$  & &  $53.20$  & &  $260.83$  & &  $157.30$  & &  $240.60$  & \\

\hline {\scriptsize SM \cite{LeordeanuSpectralCorrespondence} }   & &  $5.44$  & &  $61.71$  & &  $17.61$  & &  $39.31$  & &  $20.61$  & &  $43.15$  & \\

\hline {\scriptsize TM \cite{TensorHighOrderGraphMatching} }      & &  $44.01$  & &  $418.30$  & &  $145.38$  & &  $282.92$  & &  $187.39$  & &  $315.04$  & \\

\hline {\scriptsize RRWHM \cite{HyperGraphReweightedRandomWalk} } & &  $28.27$  & &  $178.34$  & &  $52.40$  & &  $255.72$  & &  $100.14$  & &  $185.97$  & \\

\hline {\scriptsize ProbHM \cite{ZassShashua} }                   & &  $39.58$  & &  $375.48$  & &  $128.45$  & &  $252.04$  & &  $168.70$  & &  $269.31$  & \\

\hline {\scriptsize DD \cite{TorresaniModelAndGlobalOpt} }        & &  $2145.56$  & &  $11714.50$  & &  $3999.30$  & &  $3688.99$  & &  $3972.35$  & &  $3827.95$  & \\

\hline {\scriptsize One-Shot \cite{OneShot} }                     & &  $0.50$  & &  $1.65$  & &  $1.15$  & &  $1.73$  & &  $1.44$  & &  $1.40$  & \\

\hline {\scriptsize Prev \cite{ZinanECCV} }                       & &  $874.19$  & &  $6641.19$  & &  $2497.25$  & &  $2220.49$  & &  $1665.51$  & &  $2640.98$  & \\

\hline {\scriptsize Prev \cite{ZinanECCV}-Parallel }              & &  $219.79$  & &  $1664.23$  & &  $626.41$  & &  $559.84$  & &  $419.05$  & &  $662.25$  & \\

\hline {\scriptsize ROML }                                        & &  $169.63$  & &  $839.95$  & &  $376.48$  & &  $305.51$  & &  $236.07$  & &  $415.51$  & \\

\hline {\scriptsize ROML-Parallel }                               & &  $52.09$  & &  $236.24$  & &  $110.90$  & &  $106.36$  & &  $74.74$  & &  $124.55$  & \\

\hline

\end{tabular}

\end{center}
\end{table*}


Results of ROML in the above experiments are obtained by setting the value of $n$ as the ground truth number of inliers for each image set. In practice, however, the true number of inliers is unknown. It is interesting to investigate how ROML performs when providing different values of $n$ to Algorithm \ref{MainAlgm}, assuming different numbers of inliers. We conducted such experiments using the same $6$ image sets. Results are plotted in Figure \ref{Figs-InlierNumEst-6ImgSets} in terms of Match Ratio and Identification Ratio (defined in Section \ref{InlierNumEstSec}). Figure \ref{Figs-InlierNumEst-6ImgSets} tells that better results can be achieved when the values of $n$ are close to the ground true ones, indicating the importance of knowing this prior knowledge. We have proposed in Section \ref{InlierNumEstSec} a scheme for estimating the true number of inliers for a given image set, i.e., the condition (\ref{EqnInlierNumEst}). Applying this scheme to the $6$ image sets gives the estimation results listed in Table \ref{Table-InlierNumEstResults-6ImgSets}. Compared with the true ones, results in Table \ref{Table-InlierNumEstResults-6ImgSets} seem biased towards smaller estimations. This may be due to large variations among different object instances that exist in these practical data. As a result, some of the corresponded inlier features across an image set become less correlated, and thus more like outliers. Nevertheless, Table \ref{Table-InlierNumEstResults-6ImgSets} and Figure \ref{Figs-InlierNumEst-6ImgSets} suggest that in these practical problems, we can still use the estimated numbers of inliers (as inputs of Algorithm \ref{MainAlgm}) to establish certain numbers of accurate feature matchings, as demonstrated by the Match Ratio results in Figure \ref{Figs-InlierNumEst-6ImgSets}.

\begin{figure}[t]
\centering

\includegraphics[scale=0.22]{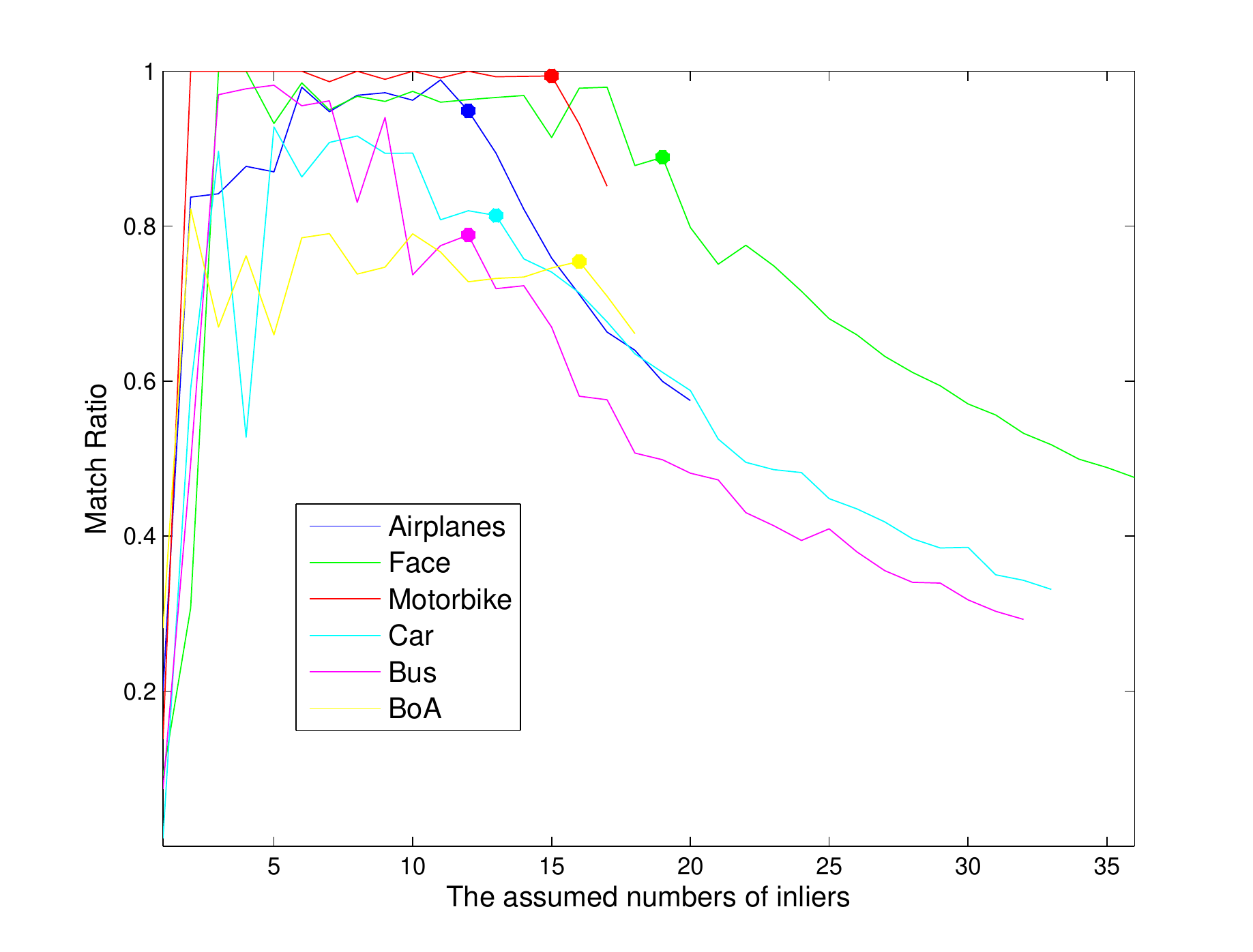}
\includegraphics[scale=0.22]{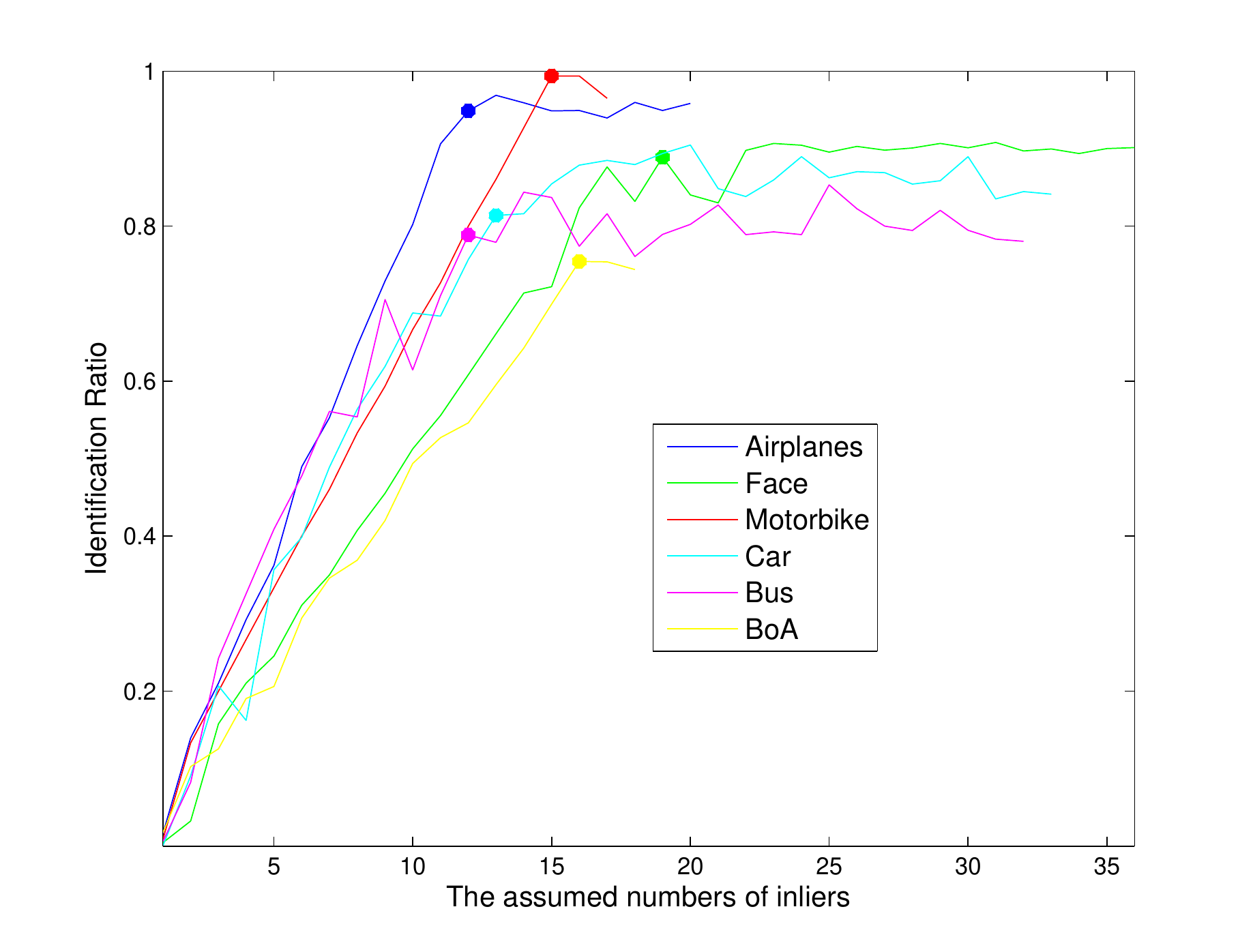} \\
\hfill (a) \hfill\hfill (b)  \hfill\hfill

\caption{ {\small Match Ratio (a) and Identification Ratio (b) of ROML on $6$ image sets of different object categories, obtained by providing different values of $n$ (the assumed number of inliers for each image set) to Algorithm \ref{MainAlgm}. The mark on each curve corresponds to the true number of inliers. } }\label{Figs-InlierNumEst-6ImgSets}
\end{figure}

\begin{table}[t]
\caption{ {\small Estimation of inlier numbers on $6$ image sets of different object categories, using the proposed scheme (\ref{EqnInlierNumEst}).} } \label{Table-InlierNumEstResults-6ImgSets}
\begin{center}
\begin{tabular}{ccc}

\hline & {\scriptsize Ground truth} & {\scriptsize Estimated} \\
       & {\scriptsize inlier number} & {\scriptsize inlier number} \\

\hline {\scriptsize Airplanes} & $12$ & $9$ \\

\hline {\scriptsize Face} & $19$ & $11$ \\

\hline {\scriptsize Motorbike} & $15$ & $13$ \\

\hline {\scriptsize Car} & $13$ & $10$ \\

\hline {\scriptsize Bus} & $12$ & $5$ \\

\hline {\scriptsize BoA} & $16$ & $11$ \\

\hline

\end{tabular}

\end{center}
\end{table}


%
%

In the above experiments, we evaluate ROML's performance by comparing the ground truth inlier features and their correspondences with those identified by ROML, e.g., the measure of Match Ratio. In practice, the ground truth information on which features are inliers is unavailable for a given image set. It is particularly interesting to detect the true inliers out of those features identified by ROML, so as to make ROML more useful for practical problems such as 3D reconstruction and object recognition. We have developed such a scheme (\ref{EqnTrueInlierDetect}) in Section \ref{InlierDetectSec}. Applying our proposed scheme to the $6$ image sets used in this section gives the results listed in Table \ref{Table-DetectActualInliers-6ImgSets}, where performance of the scheme (\ref{EqnTrueInlierDetect}) is measured by precision and recall scores (defined in Section \ref{InlierDetectSec}). These results were obtained by randomly replacing different portions of ground truth inliers with outliers for each image set, to simulate the scenarios of missing inliers. Table \ref{Table-DetectActualInliers-6ImgSets} shows that precision scores are generally high when not many true inliers are missing, and the corresponding recall scores are satisfactory to give enough numbers of true inliers for practical use. Corresponding Match Ratio results for experiments reported in Table \ref{Table-DetectActualInliers-6ImgSets} are also presented in Table \ref{Table-MatchRatioInDetectActualInliers-6ImgSets}, where Match Ratio is computed based on non-missing inliers. Table \ref{Table-MatchRatioInDetectActualInliers-6ImgSets} suggests that ROML is less influenced when only a small portion of ground truth inliers are missing.

\begin{table}[t]
\caption{ {\small Detection of true inliers on $6$ image sets of different object categories, using the proposed scheme (\ref{EqnTrueInlierDetect}). For each image set, different portions of randomly chosen ground truth inliers are replaced with outliers, to simulate the scenarios of missing inliers.  Results are presented in the format of {\it Precision/Recall}. } } \label{Table-DetectActualInliers-6ImgSets}
\begin{center}
\begin{tabular}{ccccc}

\hline {\scriptsize Ratios of}       &        &         &         &          \\
       {\scriptsize miss. inliers} & $5 \%$ & $10 \%$ & $30 \%$ & $50 \%$  \\

\hline {\scriptsize Airplanes} & $1.00/0.81$  &  $1.00/0.76$  &  $0.84/0.72$  &  $0.69/0.63$  \\

\hline {\scriptsize Face} &  $0.91/0.77$  &  $0.92/0.81$  &  $0.84/0.75$  &  $0.73/0.69$  \\

\hline {\scriptsize Motorbike} &  $1.00/0.93$  &  $0.95/0.87$  &  $0.81/0.78$  &  $0.62/0.70$  \\

\hline {\scriptsize Car} &  $0.94/0.78$  &  $0.93/0.77$  &  $0.80/0.67$  &  $0.62/0.60$  \\

\hline {\scriptsize Bus} &  $0.94/0.62$  &  $0.97/0.63$  &  $0.83/0.63$  &  $0.64/0.48$  \\

\hline {\scriptsize BoA} &  $1.00/0.73$  &  $0.90/0.60$  &  $0.63/0.42$  &  $0.53/0.39$  \\

\hline

\end{tabular}

\end{center}
\end{table}

\begin{table}[t]
\caption{ {\small Match Ratios of ROML on $6$ image sets of different object categories. For each image set, different portions of randomly chosen ground truth inliers are replaced with outliers, to simulate the scenarios of missing inliers.   } } \label{Table-MatchRatioInDetectActualInliers-6ImgSets}
\begin{center}
\begin{tabular}{ccccc}

\hline {\scriptsize Ratios of}       &        &         &         &          \\
       {\scriptsize missing inliers} & $5 \%$ & $10 \%$ & $30 \%$ & $50 \%$  \\

\hline {\scriptsize Airplanes} & $95 \%$  &  $94 \%$  &  $74 \%$  &  $69 \%$  \\

\hline {\scriptsize Face} & $80 \%$  &  $84 \%$  &  $77 \%$  &  $72 \%$   \\

\hline {\scriptsize Motorbike} &  $98 \%$  &  $93 \%$  &  $84 \%$  &  $69 \%$ \\

\hline {\scriptsize Car} &  $82 \%$  &  $79 \%$  &  $71 \%$  &  $57 \%$  \\

\hline {\scriptsize Bus} & $68 \%$  &  $72 \%$  &  $66 \%$  &  $55 \%$  \\

\hline {\scriptsize BoA} & $76 \%$  &  $59 \%$  &  $38 \%$  &  $30 \%$ \\

\hline

\end{tabular}

\end{center}
\end{table}

\subsection{Non-Rigid Object Moving in a Video Sequence}
\label{ExpTrackingSec}

Lastly, we test how ROML performs to match a non-rigid object moving in a video sequence. This could be a much harder application scenario than that in Section \ref{ExpCaltech101MSRCSec}. For example, in a video sequence with static background, some salient points in the background consistently appear across the video frames, with little change of appearance, and they can form their own ``low-rank pattern''. The corresponding salient points (inliers) on the non-rigid foreground of interest may only appear in some of the frames, possibly with changing appearance. In this section, we used a $25$-frame ``Tennis'' sequence and a $50$-frame ``Marple'' sequence \cite{BroxMotionSegData} to test how ROML performs in this challenging scenario. For the first sequence, we used KLT tracker \cite{KLT} to detect $100$ interest points in each frame. For the second one, we detected $150$ interest points in each frame. We again used the type of learned embedded features as explained in Section \ref{CoordDescriptorCombSec}. The parameter settings were the same as those used in Section \ref{ExpCaltech101MSRCSec}. Figure \ref{ROMLFailureExamShowFigs} illustrates the matching results of ROML among $4$ frames of these two sequences respectively. Clearly, most of the identified correspondences by ROML are from the background, showing the failure when applying ROML to this challenging scenario. In fact, comparative graph and hyper-graph matching methods all failed on these two sequences. This difficulty of mining/matching foreground objects from video sequences with static background poses a common challenge for these feature-based object matching methods.

\begin{figure}[t]
\centering

\includegraphics[scale=0.1875]{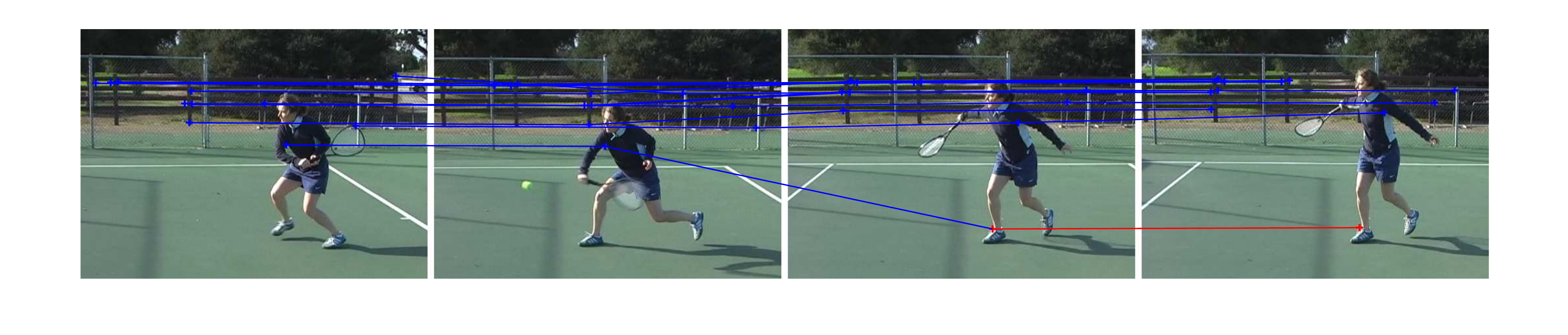} \vspace{-0.2cm} \\
\includegraphics[scale=0.22]{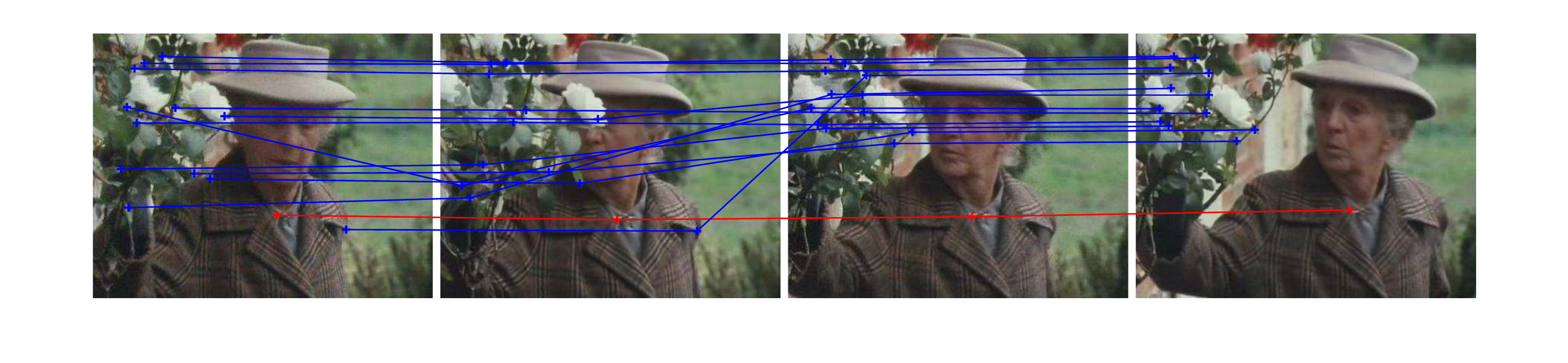} \\

\caption{ {\small Illustration of the failure of ROML on matching a non-rigid object moving in a video sequence. Most of the identified correspondences by ROML are from the background (blue lines), rather than from the foreground object of interest (red lines).} }\label{ROMLFailureExamShowFigs}
\end{figure}

The above challenge can be largely alleviated when inliers from the foreground object can be specified in the first frame of a video sequence. This resembles an object tracking scenario. In the following, we report experiments that show how ROML performs in this alleviated non-rigid object matching scenario. More specifically, given a video sequence with interest points detected in each of the total $K$ frames, we label inlier points from those detected in the first frame. The task is to match these inlier points, which are supposed to be on the object of interest, across the subsequent $K - 1$ video frames, simulating a tracking scenario. We again used the ``Tennis'' and ``Marple'' sequences. We labelled $11$ inliers from the $100$ detected interest points in the first frame of the ``Tennis'' sequence, and $14$ inliers from the $150$ detected interest points in the first frame of the ``Marple'' sequence. To adapt our method to this scenario, we simply fix ${\bf P}^{1}$ in steps $5$ to $7$ of Algorithm \ref{MainAlgm} so that it selects the $n$ inliers labelled in the first frame, while optimizing the other $K - 1$ PPMs $\{ {\bf P}^{k} \}_{k=2}^{K}$. \footnote{To make this scheme effective, we {\it emphasize} the labelled first frame by normalizing feature vectors of its interest points to have a larger value of $\ell_{2}$-norm, compared to those in the other $K-1$ frames. This trick is important since when applying ROML to a video sequence with static background, it is possible that some of the interest points in the background are selected to form a ``low-rank pattern'', while those of the true pattern labelled in the first frame are treated as outliers. Theoretical correctness of this scheme is pending for proof. In practice, we observed that it worked well, and we defer the proof of this scheme in future research.} We compare ROML with a baseline KLT tracker, and recent graph and hyper-graph matching methods \cite{TorresaniModelAndGlobalOpt,GraphReweightedRandomWalk,LeordeanuSpectralCorrespondence,TensorHighOrderGraphMatching,HyperGraphReweightedRandomWalk,ZassShashua}. Since inlier points in the first frame are labelled ground truth, for graph and hyper-graph matching methods, we generated $24$ and $49$ frame pairs for the ``Tennis'' and ``Marple'' sequences respectively, i.e., between the first frame and each of the other frames, and used them for pair-wise matching. The other settings of these methods were the same as those used to produce Table \ref{Table-Caltech101MSRCMainMatchRatioResult} in Section \ref{ExpCaltech101MSRCSec}. Parameters of these methods were also tuned to their respective best results on the two sequences. Table \ref{Table-NonRigidTrackingResult} reports the quantitative results of different methods in terms of Match Ratio. Example correspondences of interest points for DD \cite{TorresaniModelAndGlobalOpt} and our method are also shown in Figure \ref{NonRigidTrackingExamFrameShowFigs}. KLT tracker generally fails since there are abrupt motion and/or occlusion of inlier points in these two sequences. Compared to graph/hyper-graph matching methods, our method gives better results, which confirms the effectiveness of ROML for simultaneous multi-image object matching. We also compare with our previous method \cite{ZinanECCV} in Table \ref{Table-NonRigidTrackingResult}. Consistent to those results in Section \ref{ExpCaltech101MSRCSec}, ROML again improves over Prev \cite{ZinanECCV} on these two sequences.

\begin{table*}[t]
\caption{ {\small Match Ratios of different methods on the ``Tennis'' and ``Marple'' sequences \cite{BroxMotionSegData}.} } \label{Table-NonRigidTrackingResult}
\begin{center}
\begin{tabular}{cccccccccc}
\hline {\scriptsize Methods} & {\scriptsize KLT \cite{KLT} } & {\scriptsize RRWM \cite{GraphReweightedRandomWalk} } & {\scriptsize SM \cite{LeordeanuSpectralCorrespondence} } & {\scriptsize TM \cite{TensorHighOrderGraphMatching} } & {\scriptsize RRWHM \cite{HyperGraphReweightedRandomWalk} } & {\scriptsize ProbHM \cite{ZassShashua} } & {\scriptsize DD \cite{TorresaniModelAndGlobalOpt} } & {\scriptsize Prev \cite{ZinanECCV} } & {\scriptsize ROML } \\

\hline {\scriptsize Tennis} & { $ 3\%$} & { $ 23\%$} & { $ 43\%$} & { $ 13\%$} & { $ 18\%$} & { $ 16\%$} & { $ 57\%$} & { $52 \%$} & { ${\bf 73\%}$}  \\

\hline {\scriptsize Marple} & { $ 4\%$} & { $ 3\%$} & { $ 25\%$} & { $ 8\%$} & { $ 13\%$} & { $ 14\%$} & { $ 23\%$} & { $41 \%$} & { ${\bf 51\%}$}  \\

\hline

\end{tabular}

\end{center}
\end{table*}

\begin{figure}[t]
\centering

\begin{tabular}{c}
\includegraphics[scale=0.18]{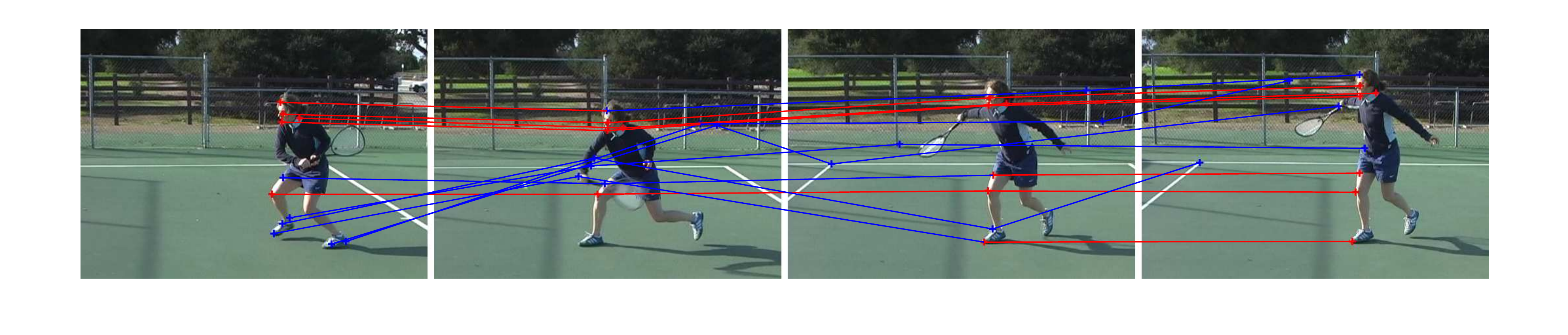} \vspace{-0.2cm} \\
\includegraphics[scale=0.1825]{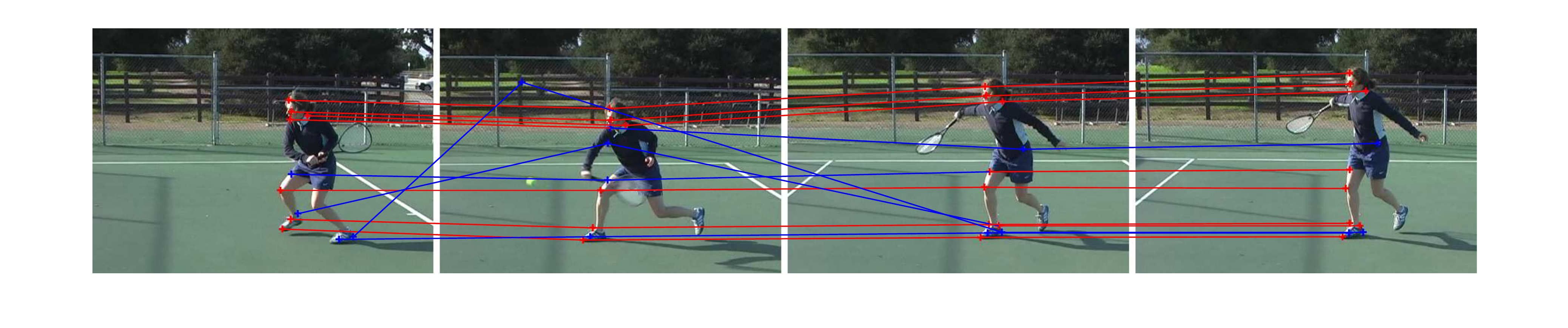} \\

\includegraphics[scale=0.22]{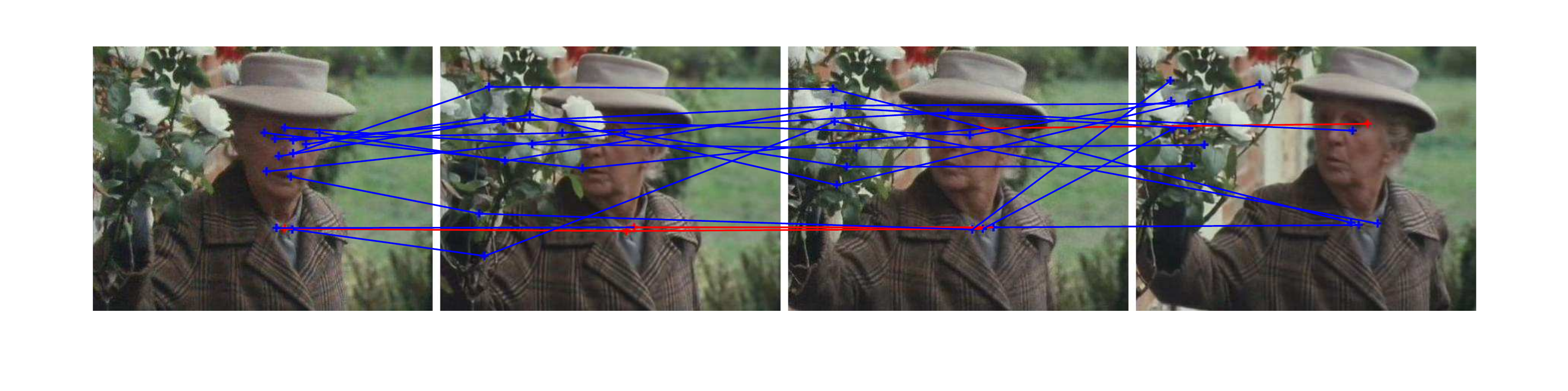} \vspace{-0.4cm} \\
\includegraphics[scale=0.22]{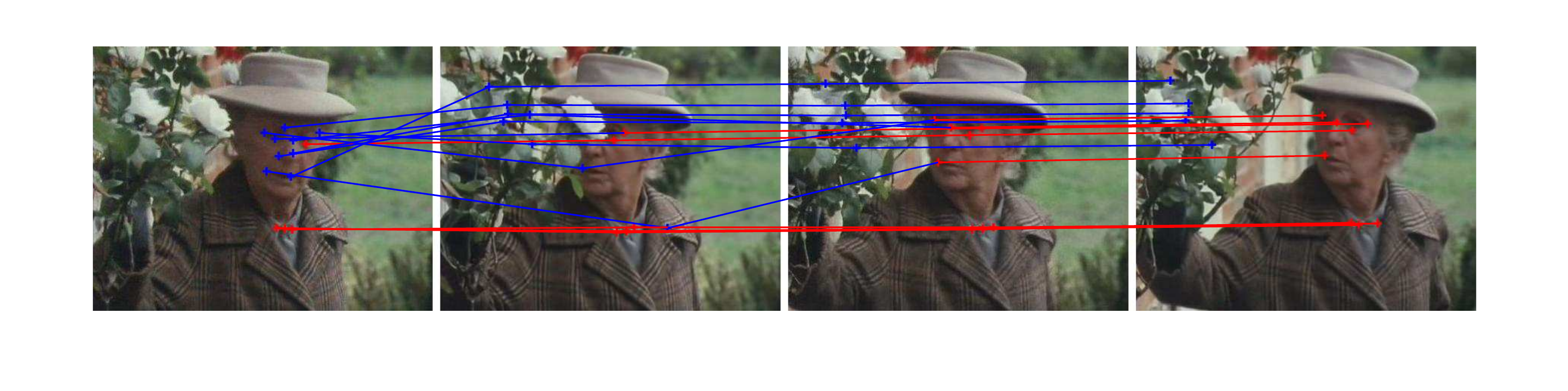}   \\
\end{tabular}

\caption{ {\small Example correspondences of interest points among $4$ frames of the ``Tennis'' and ``Marple'' sequences \cite{BroxMotionSegData} respectively. For every pair top is from DD \cite{TorresaniModelAndGlobalOpt}, and bottom is from our method. Red lines represent identified ground truth correspondences, and blue lines are for false ones.} }\label{NonRigidTrackingExamFrameShowFigs}
\end{figure}

\section{Common Object Localization}
\label{ExpCOLSec}

Learning models of object categories typically requires manually labelling a large amount of training images (e.g., up to a bounding box of the object of interest), which however, are expensive to obtain and may also suffer from unintended biases by annotators. A recently emerging research topic \cite{TuytelaarsUOCDSurvey} considers automatically discovering and learning object models from a collection of unlabelled images. Given an image collection containing object instances belonging to unknown categories, the task is to identify the categories, localize object instances in images, and learn models for them so that the learned models can be applied to novel images for object detection. This is a weakly supervised (or unsupervised) learning scenario when the image collection is known to contain object instances of a single category (or multiple categories), which is in general ill-posed. A critical component for success of learning is precise object localization inside each image. However, precise {\it common object localization} (COL) is extremely difficult given unknown object categories/models, and also large intra-category variations and cluttered background.

Many methods have been proposed for this challenging task in either weakly supervised or unsupervised settings \cite{UOCDLinkAnalysis,ForegroundFocusForUOCD,TopicModelImageCorrespUOCD,RussellSegmentationUOCD}. Among them the methods \cite{ForegroundFocusForUOCD,UOCDLinkAnalysis,TopicModelImageCorrespUOCD} explicitly take object (or its associated parts/features) localization into account. These methods normally require the objects of interest covering a large portion of the images. More recently, saliency guided object learning techniques \cite{bMCL,ObjectnessForUOCD} are proposed, which exploit generic knowledge of ``objectness'' \cite{Objectness,ObjectProposal,SalientyByComposition} obtained from low-level image cues and/or learning from other irrelevant annotated images. Consequently, they can potentially locate object instances with large scale/appearance variations in cluttered background.

In this section, we present experiments to show how ROML can be applied to this COL task using local region descriptors as features. Similar to \cite{ObjectnessForUOCD}, we also sample candidate bounding boxes from each image based on their objectness scores, and use appropriate region descriptors to characterize the appearance inside each bounding box. We then optimize (\ref{EqnNuclearL1Form}) to select a bounding box from each image, i.e., $n = 1$ for the PPMs to be optimized. Ideally the selected bounding boxes should localize object instances deemed common in the given image collection, i.e., the matrix ${\bf L}$ in (\ref{EqnNuclearL1Form}) is rank deficient. We used the PASCAL datasets \cite{PASCALVOC2006,PASCALVOC2007} for the COL experiments in both weakly supervised and unsupervised settings. For the weakly supervised case, we followed the same settings as in \cite{ObjectnessForUOCD}. In particular, we used a subset of the PASCAL06 \cite{PASCALVOC2006} train+val dataset containing all images of $6$ classes (bicycle, car, cow, horse, motorbike, sheep) from the left and right viewpoints. We conducted COL on all images of each class-viewpoint combination, which are assumed to contain object instances of the same class at a similar viewpoint. To make the problem better defined, we followed \cite{ObjectnessForUOCD} and removed images in which all objects are marked as difficult or truncated in the ground truth annotation. The PASCAL07 dataset \cite{PASCALVOC2007} is more challenging as objects vary greatly in appearance, scale, and location. We also used $6$ classes (aeroplane, bicycle, boat, bus, horse, and motorbike) of the PASCAL07 train+val dataset from the left and right viewpoints. The other settings were the same as for the PASCAL06 dataset. These classes of PASCAL06 and PASCAL07 datasets were chosen because they are the object classes on which fully supervised methods can perform reasonably well. For every image in one class-viewpoint combination, we used \cite{Objectness} to sample $100$ bounding boxes proportionally to their probability of containing an object (the objectness score). To describe the region appearance inside each bounding box, we used the GIST descriptor with the default parameters as in \cite{GIST}, which gives a $512$-dimensional feature vector. As suggested in \cite{ObjectnessForUOCD}, the shape and objectness score of a bounding box provide additional information that may help for COL. Let ${\bf f}$ be the GIST descriptor vector for a bounding box. To use shape and objectness score, we first augmented ${\bf f}$ with the aspect ratio ($width / height$) $r$ of the bounding box, and then added perturbation noise ${\bf n} \in \mathbb{R}^{513}$ whose entries were drawn from normal distribution with standard deviation set as one minus objectness score of the bounding box. We used the thus produced vector $[{\bf f}^{\top} \ \kappa_{r}r]^{\top} + \kappa_{\bf n}{\bf n}$ as the feature for each sampled bounding box, where $\kappa_{r}$ and $\kappa_{\bf n}$ are weighting parameters. We set $\kappa_{r} = 0.08$ and $\kappa_{\bf n} = 0.015$ for all the experiments reported in this section.

We measure COL performance by the percentage of correctly localized images out of all images in a class-viewpoint combination, where localization correctness in an image is based on PASCAL criteria, i.e., intersection of a bounding box with ground truth is more than half of their union. We compare with several baseline weakly supervised object localization and learning methods including MultiSeg \cite{RussellSegmentationUOCD} and Exemplar \cite{ChumUOCD}, and also with WSL-GK \cite{ObjectnessForUOCD}, which is saliency guided and performs EM-like alternation of localizing objects and learning the object class model. In the preparation of this paper, we notice that a more recent work \cite{TaoXiangSaliencyForCOL} gives better COL performance by using more advanced saliency estimation method. Since this section is mainly to show the usefulness of ROML for the COL task, we will not pursue adopting this new saliency method to further improve our results.

Table \ref{Table-COLOverallAccuracy} reports COL accuracies of different methods on the PASCAL06 and PASCAL07 datasets, which are obtained by averaging over all class-viewpoint combinations. Results of MultiSeg \cite{RussellSegmentationUOCD} and Exemplar \cite{ChumUOCD} in Table \ref{Table-COLOverallAccuracy} are from \cite{ObjectnessForUOCD}. Table \ref{Table-COLOverallAccuracy} suggests that Objectness \cite{Objectness} gives very good initial candidates of object bounding boxes. Consequently, results of both our method and WSL-GK \cite{ObjectnessForUOCD} on the PASCAL06 and PASCAL07 datasets compare favorably with those from MultiSeg \cite{RussellSegmentationUOCD} and Exemplar \cite{ChumUOCD}. For the PASCAL07 dataset, our method is comparable to WSL-GK \cite{ObjectnessForUOCD} when no iterative steps of class learning are performed in \cite{ObjectnessForUOCD}, and greatly outperforms \cite{ObjectnessForUOCD} for the PASCAL06 dataset, for which our result in fact approaches final result of \cite{ObjectnessForUOCD}, which is obtained after full steps of class learning and using richer feature representation including GIST, color information, and HOG for object shapes. Since the present paper is focusing on object matching and localization, we defer extension of our method for object class learning as future research.

\begin{table*}[t]
\caption{ {\small COL accuracies of different methods on the PASCAL06 and PASCAL07 datasets. For objectness \cite{Objectness}, sampled bounding box with the highest score in each image is considered as the estimated localization.} } \label{Table-COLOverallAccuracy}
\begin{center}

\begin{tabular}{cc}

\begin{tabular}{|c|c|c|c|c|c|}
\hline \multirow{1}{*}{ {          } } & { {\scriptsize Objectness \cite{Objectness} } } & { {\scriptsize MultiSeg \cite{RussellSegmentationUOCD} } } & { {\scriptsize Exemp. \cite{ChumUOCD} } } & { {\scriptsize WSL-GK \cite{ObjectnessForUOCD} } } & { {\scriptsize ROML } }  \\
       \multirow{1}{*}{ {  } } & { { } } & { { } } & { { } } & { {\scriptsize (No Learning) } } & { { } }  \\
\hline \multirow{1}{*}{ {\scriptsize PASCAL06 } } & { { $51 \%$} } & { { $28 \%$ } } & { { $45 \%$} } & { { $55 \%$ } } & { { ${\bf 64 \%}$ } }  \\
\hline \multirow{1}{*}{ {\scriptsize PASCAL07 } } & { { $28 \%$} } & { { $22 \%$ } } & { { $33 \%$} } & { { ${\bf 37 \%}$ } } & { { $36 \%$} }  \\
\hline
\end{tabular}

\hspace{-0.2cm}
&
\hspace{-0.2cm}

\begin{tabular}{|c|}
\hline { {\scriptsize WSL-GK \cite{ObjectnessForUOCD} } }  \\
       { {\scriptsize (With Learning) } }  \\
\hline { { $64 \%$ } }  \\
\hline { { $50 \%$} }  \\
\hline
\end{tabular}

\end{tabular}

\end{center}
\end{table*}

We also conducted COL experiments in the unsupervised setting using $4$ classes from the PASCAL06 (bicycle, car, cow, and sheep) and PASCAL07 (aeroplane, bus, horse, and motorbike) datasets respectively. Other data setups were the same as those in the above weakly supervised COL experiments. For either of the PASCAL06 and PASCAL07 datasets, we put all images of different classes from one viewpoint as an image collection, and applied ROML for object localization. Performance was again measured by the percentage of correctly localized images out of all images in a class-viewpoint combination. Table \ref{Table-COLClassSpecificAccuracy} reports detailed results of different class-viewpoint combinations, where we also list results of ROML in the weakly supervised setting. Table \ref{Table-COLClassSpecificAccuracy} tells that ROML performs consistently well in both weakly supervised and unsupervised object localization. Example images of these classes with localized bounding boxes are shown in Figure \ref{COLExamImgShowFigs}, where we also show the bounding boxes with the highest objectness score in each image and those of ROML in weakly supervised setting for comparison.

\begin{table*}
\caption{ {\small COL accuracies of ROML for different class-viewpoint combinations of the PASCAL06 and PASCAL07 datasets in both weakly supervised and unsupervised settings.} } \label{Table-COLClassSpecificAccuracy}
\begin{center}
\begin{tabular}{ccccccccccc}
\hline  & & & {\scriptsize PASCAL06} &  & & & & {\scriptsize PASCAL07} &  & \\
\hline  & & {\scriptsize Bicycle} & {\scriptsize Car} & {\scriptsize Cow} & {\scriptsize Sheep} & & {\scriptsize Aeroplane} & {\scriptsize Bus} & {\scriptsize Horse} & {\scriptsize Motorbike} \\
\hline  {\scriptsize Weakly Supervised - Left} & & $84 \%$ & $79 \%$ & $60 \%$ & $58 \%$ & & $26 \%$ & $24 \%$ & $40 \%$ & $56 \%$ \\
\hline  {\scriptsize Unsupervised - Left} & & $80 \%$ & $79 \%$ & $60 \%$ & $52 \%$ & & $30 \%$ & $29 \%$ & $35 \%$ & $56 \%$ \\
\hline  {\scriptsize Weakly Supervised - Right} & & $69 \%$ & $70 \%$ & $66 \%$ & $52 \%$ & & $38 \%$ & $61 \%$ & $35 \%$ & $65 \%$ \\
\hline  {\scriptsize Unsupervised - Right} & & $67 \%$ & $63 \%$ & $57 \%$ & $40 \%$ & & $28 \%$ & $48 \%$ & $41 \%$ & $56 \%$ \\

\hline

\end{tabular}
\end{center}
\end{table*}

\begin{figure*}[t]
\centering

{\small

\begin{tabular}{ccccccccc}

Bicycle & Car & Cow & Sheep & & Aeroplane & Bus & Horse & Motorbike \\

\includegraphics[scale=0.2]{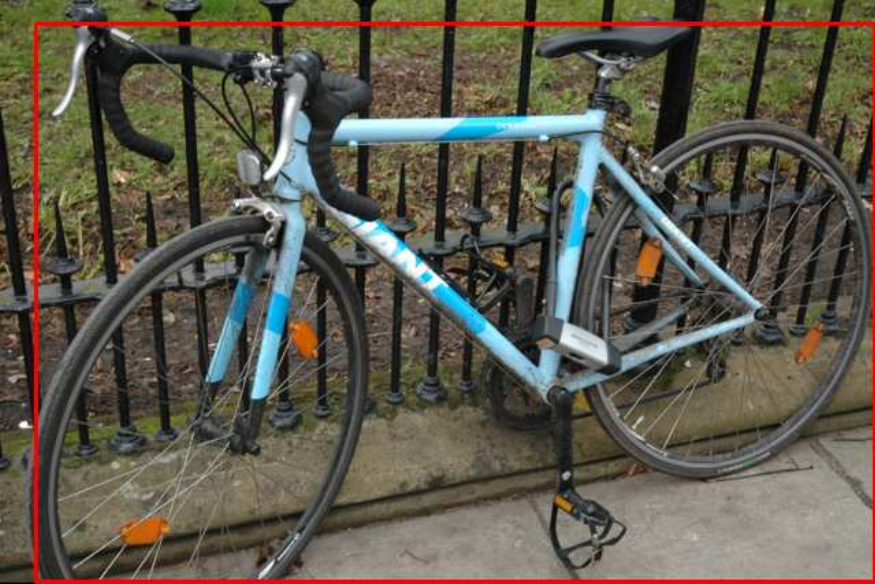} & \includegraphics[scale=0.2]{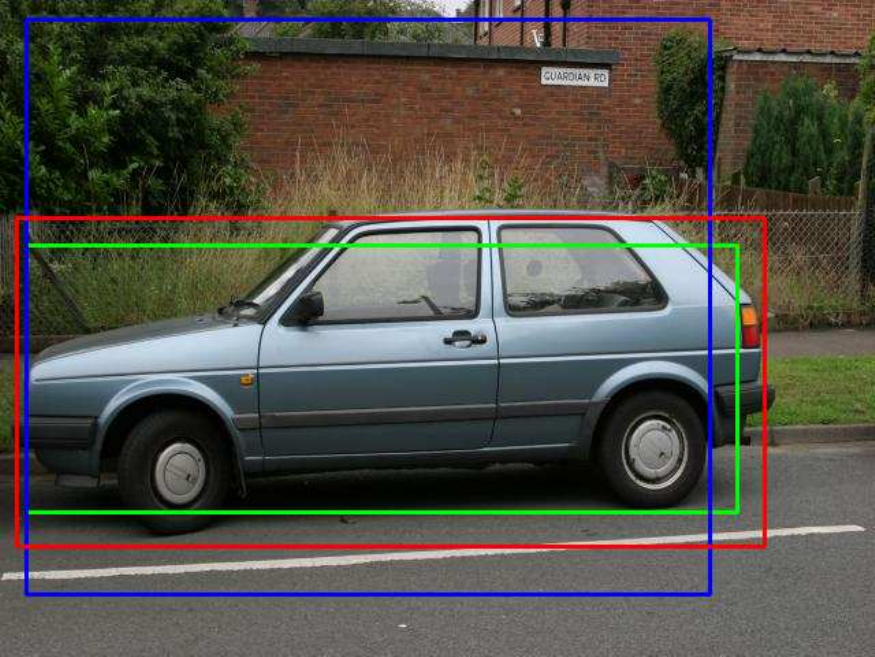} & \includegraphics[scale=0.2]{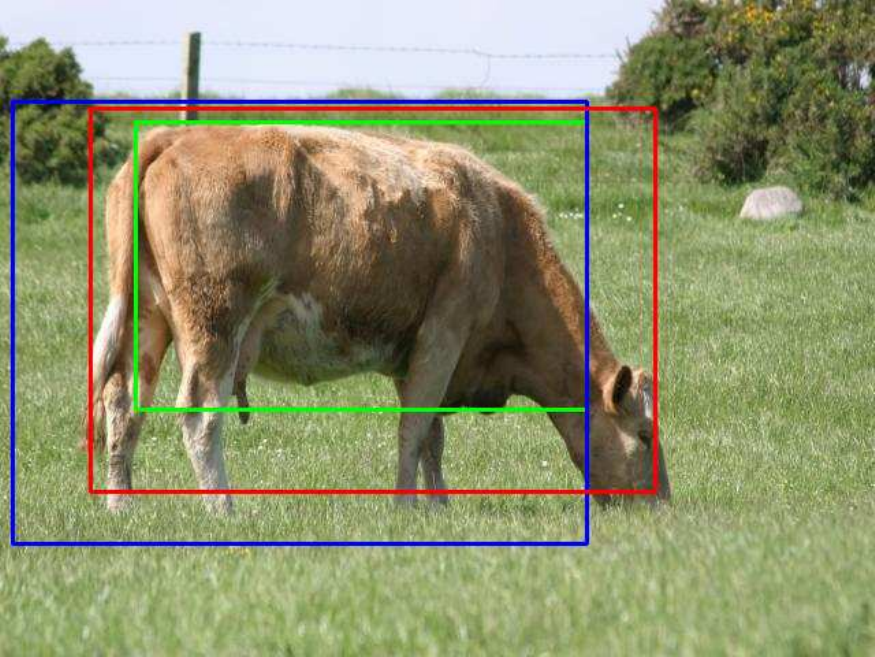} & \includegraphics[scale=0.2]{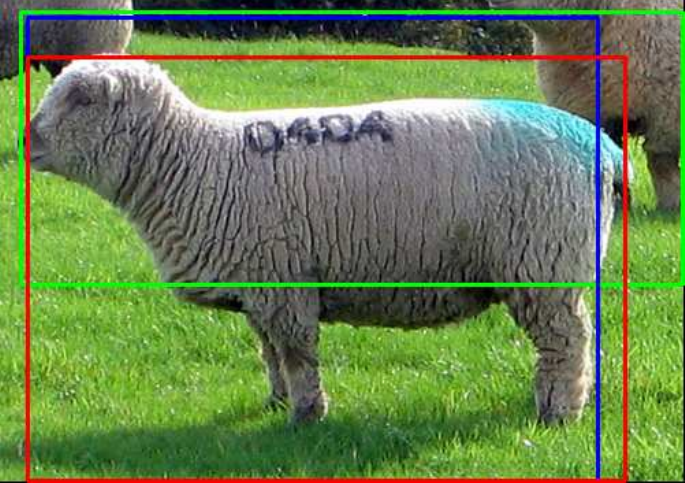} & & \includegraphics[scale=0.2]{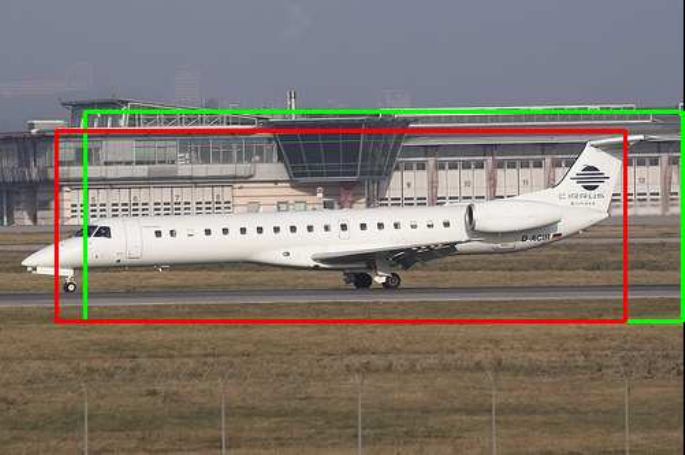} & \includegraphics[scale=0.2]{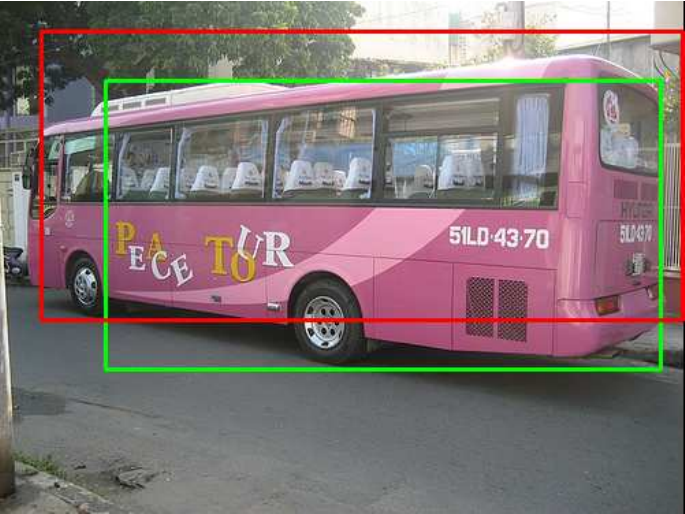} & \includegraphics[scale=0.2]{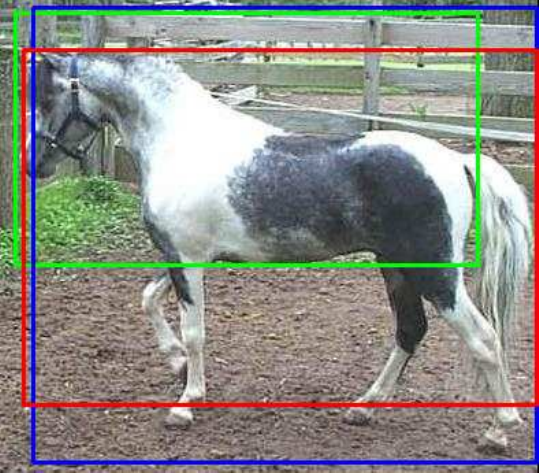} & \includegraphics[scale=0.2]{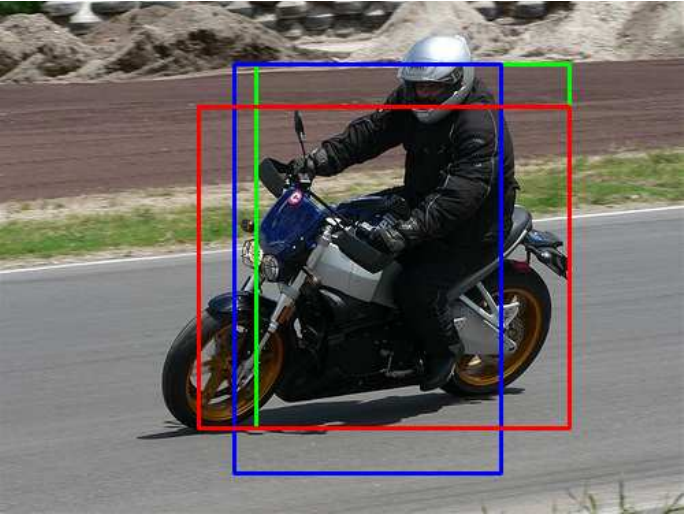} \\

\includegraphics[scale=0.2]{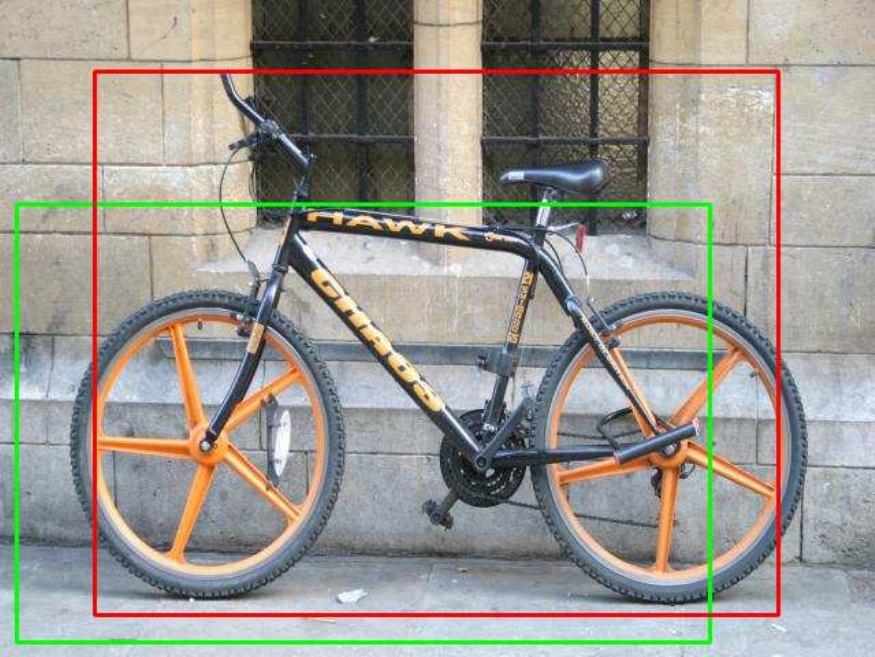} & \includegraphics[scale=0.2]{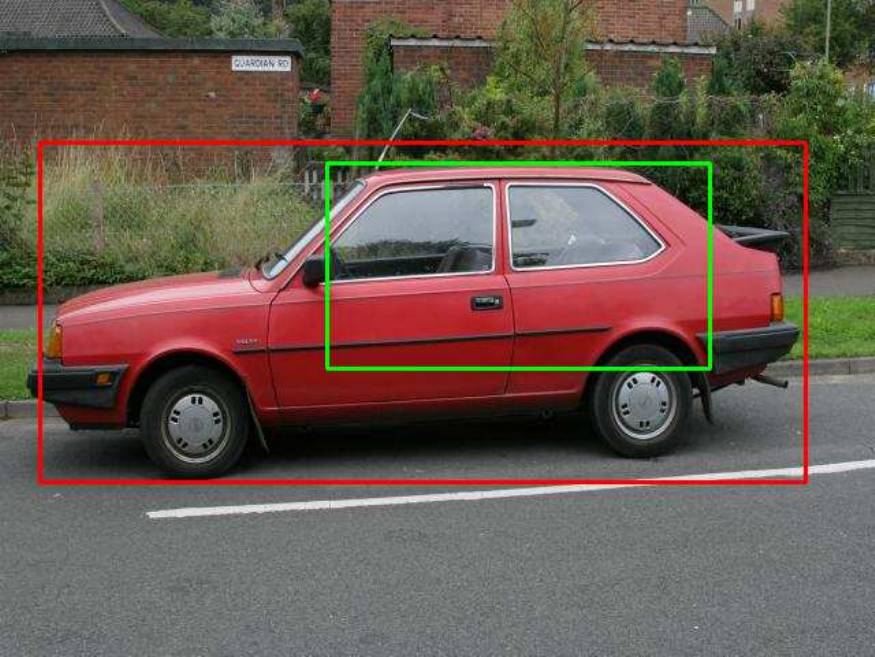} & \includegraphics[scale=0.2]{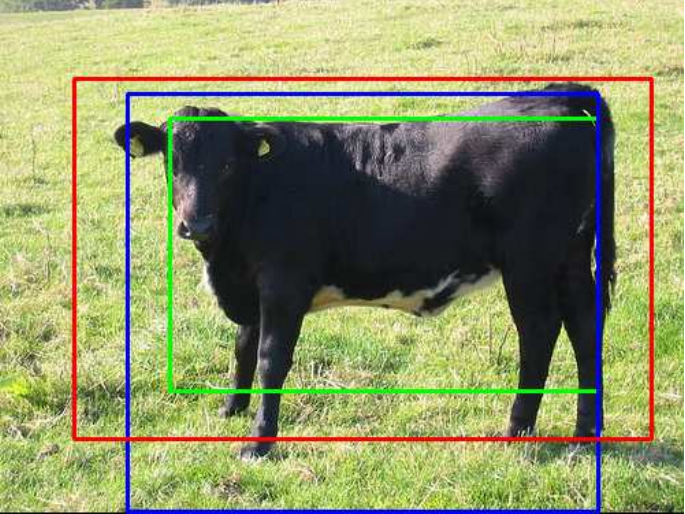} & \includegraphics[scale=0.2]{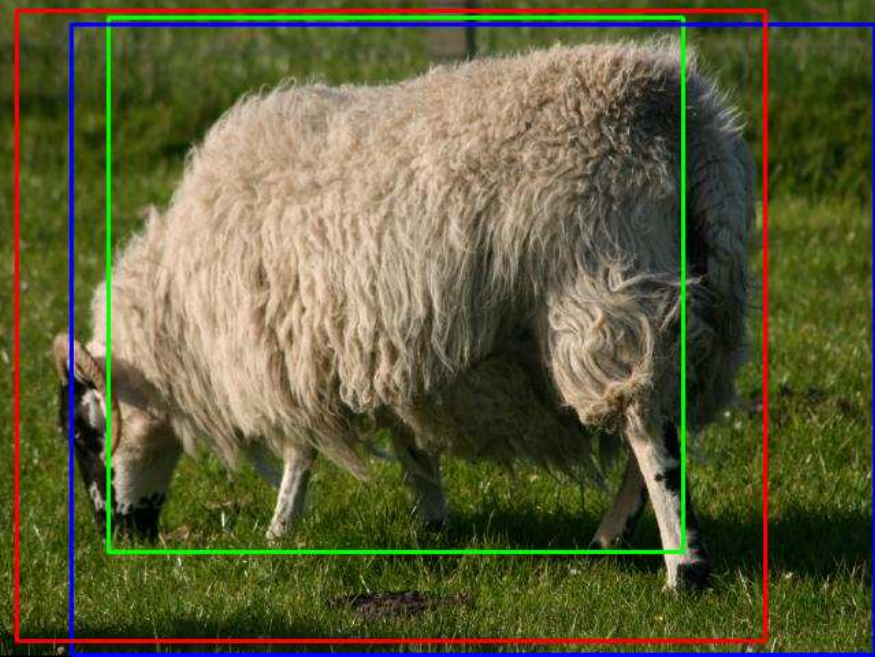} & & \includegraphics[scale=0.2]{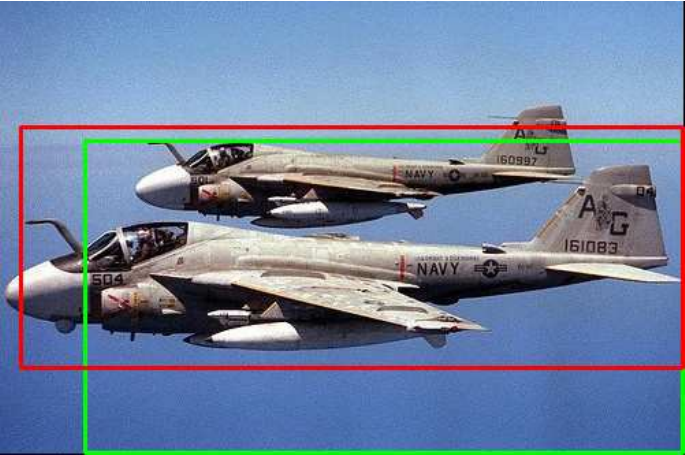} & \includegraphics[scale=0.2]{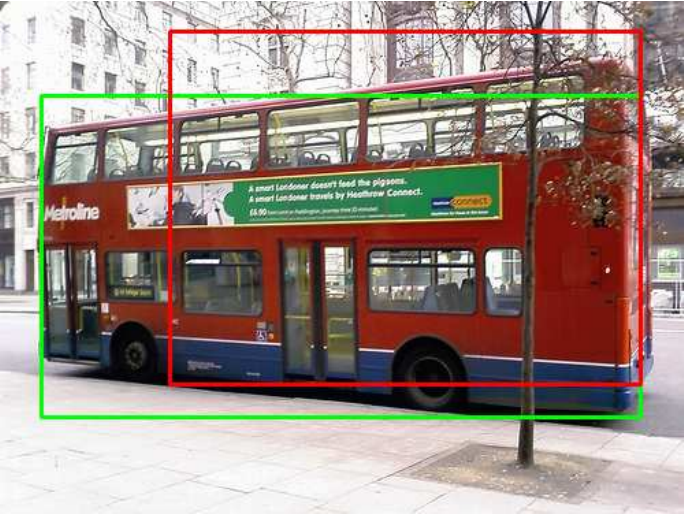} & \includegraphics[scale=0.2]{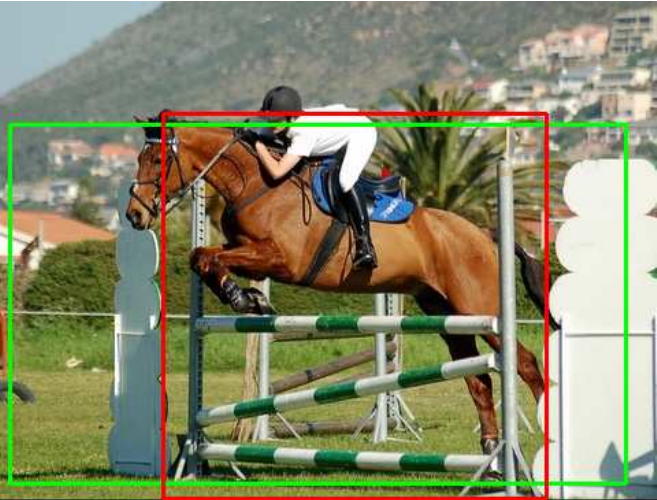} & \includegraphics[scale=0.2]{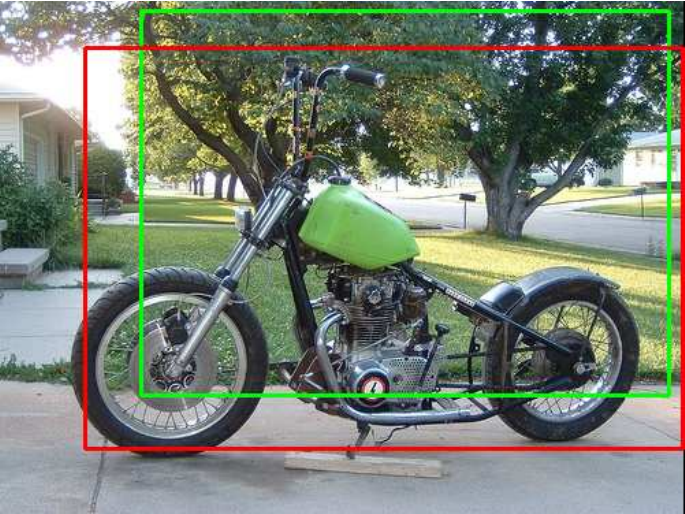} \\

\includegraphics[scale=0.2]{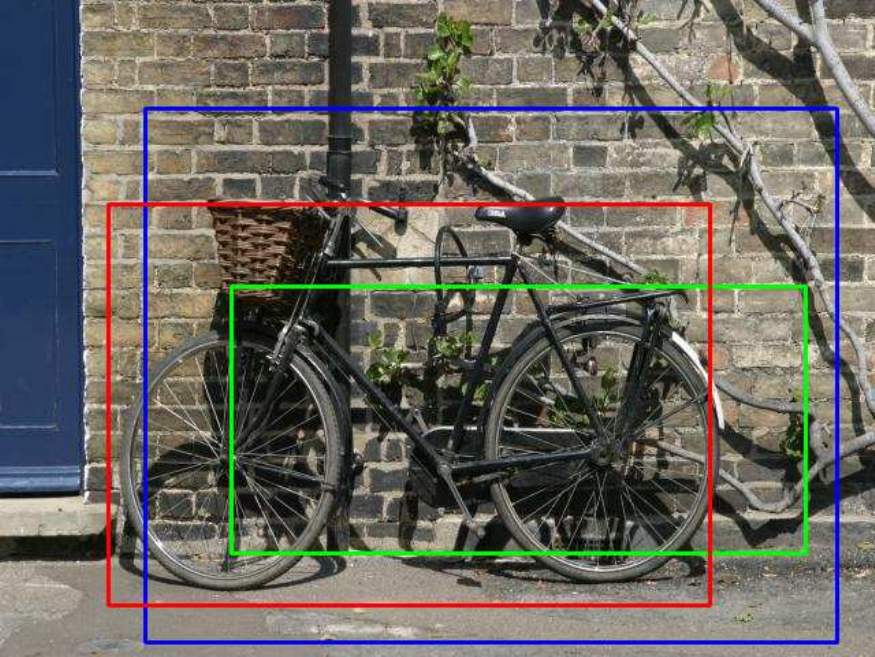} & \includegraphics[scale=0.2]{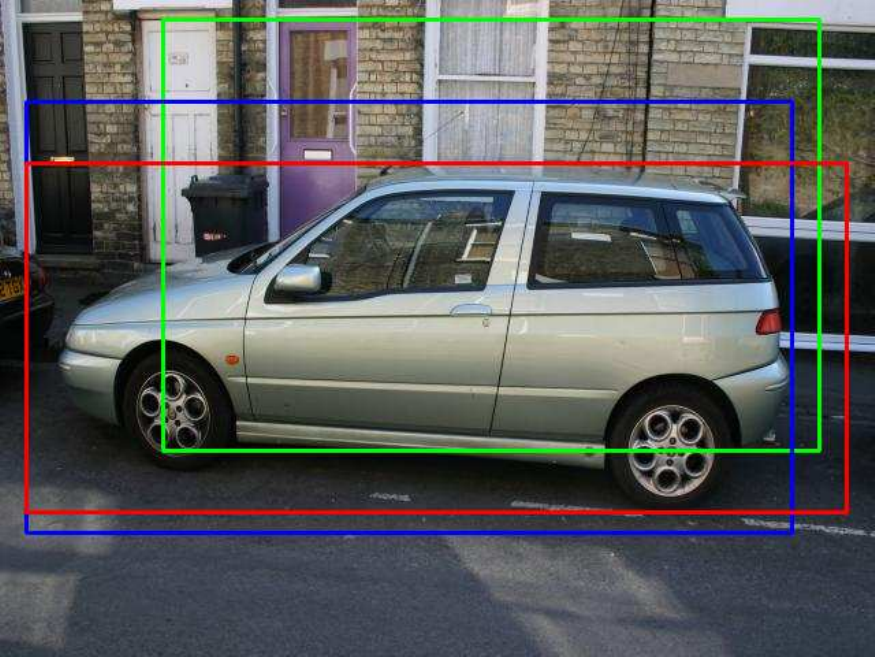} & \includegraphics[scale=0.2]{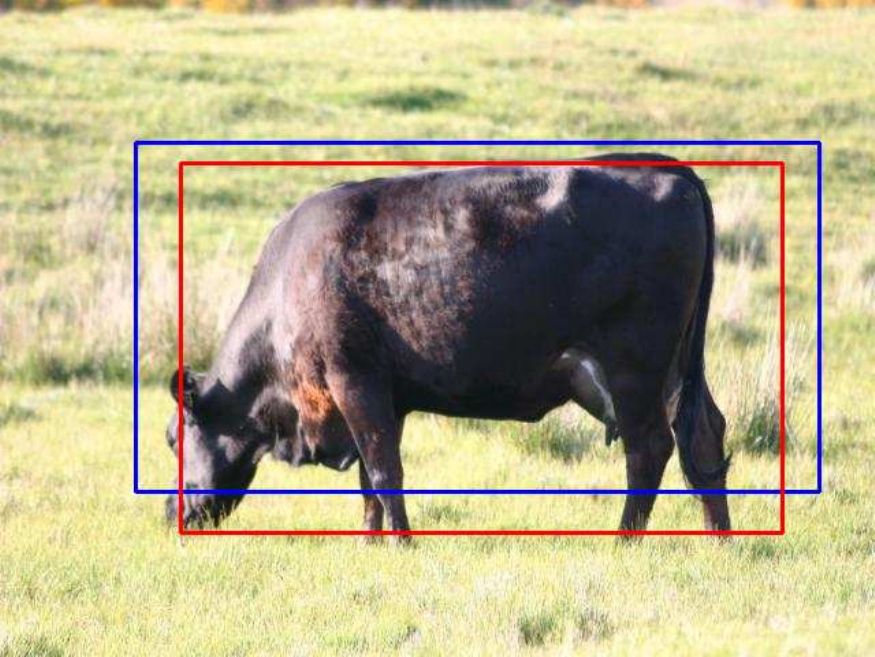} & \includegraphics[scale=0.2]{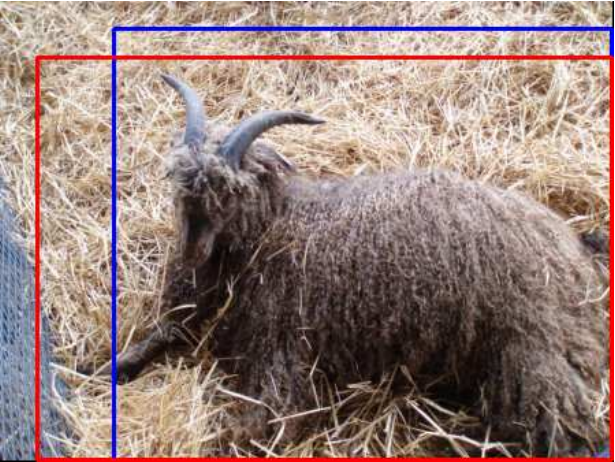} & & \includegraphics[scale=0.2]{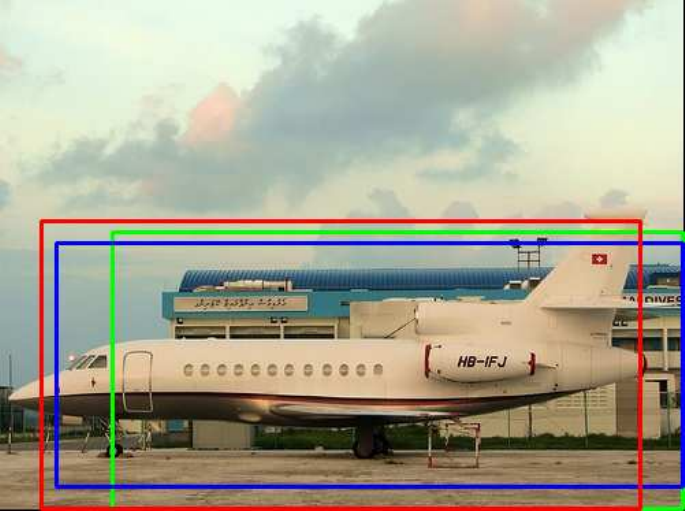} & \includegraphics[scale=0.2]{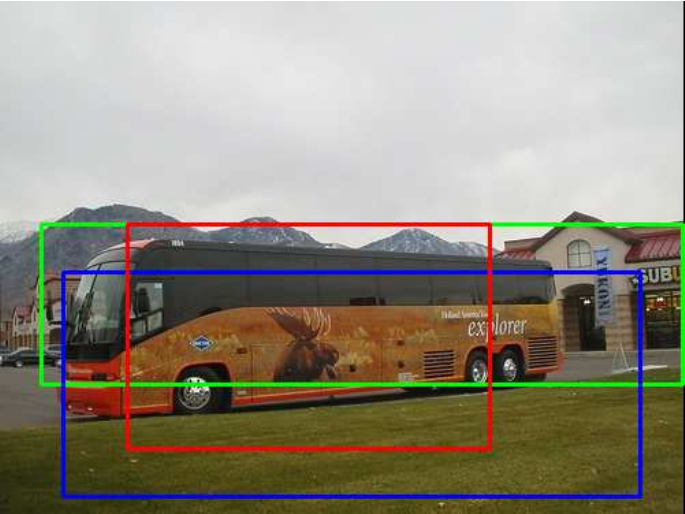} & \includegraphics[scale=0.2]{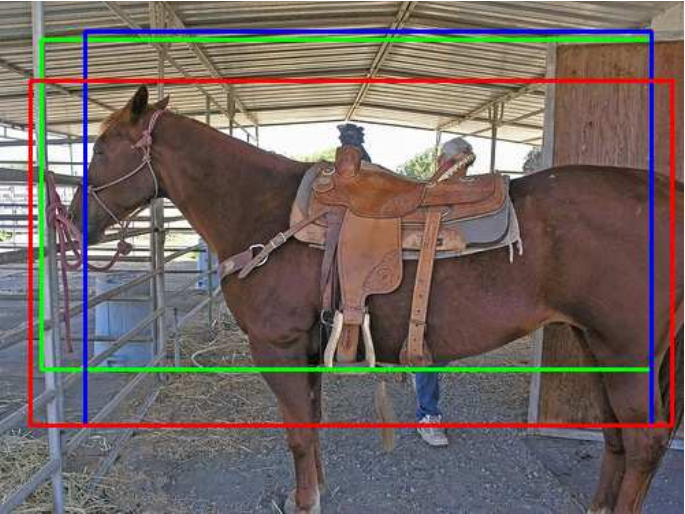} & \includegraphics[scale=0.2]{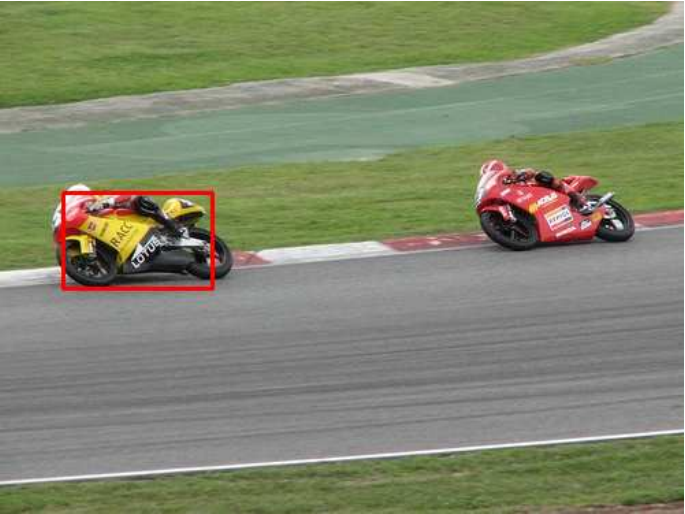} \\

\\

\includegraphics[scale=0.2]{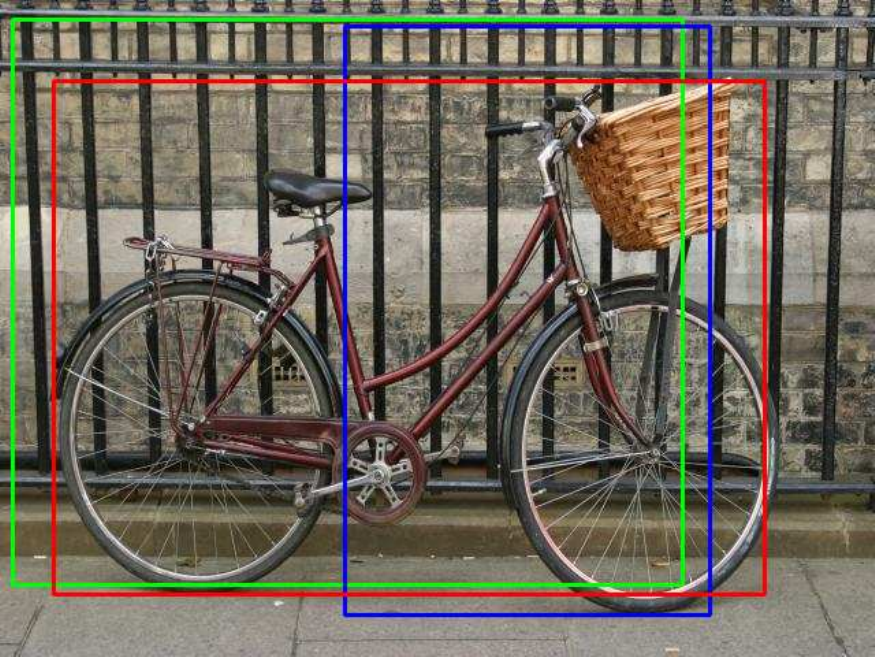} & \includegraphics[scale=0.2]{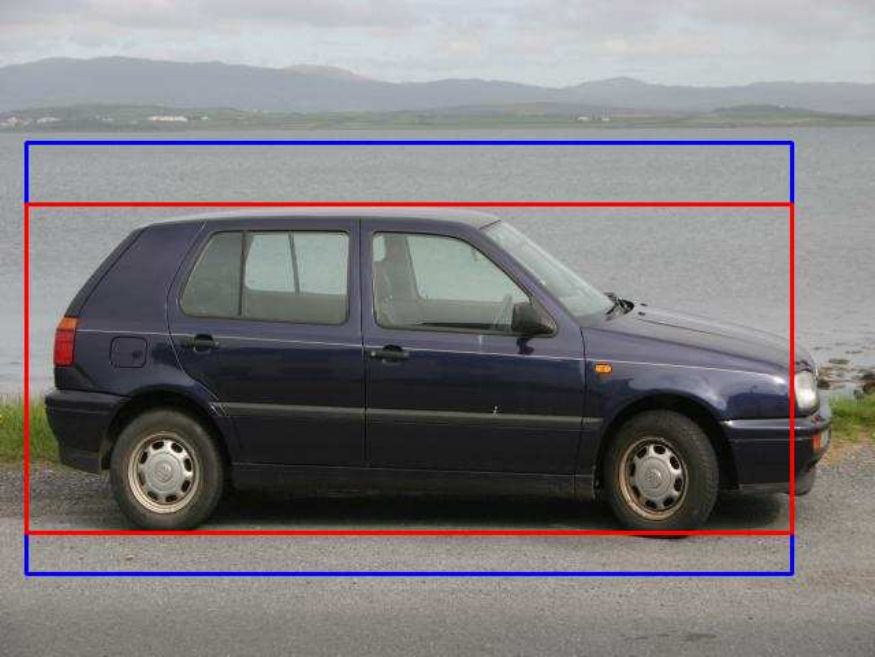} & \includegraphics[scale=0.2]{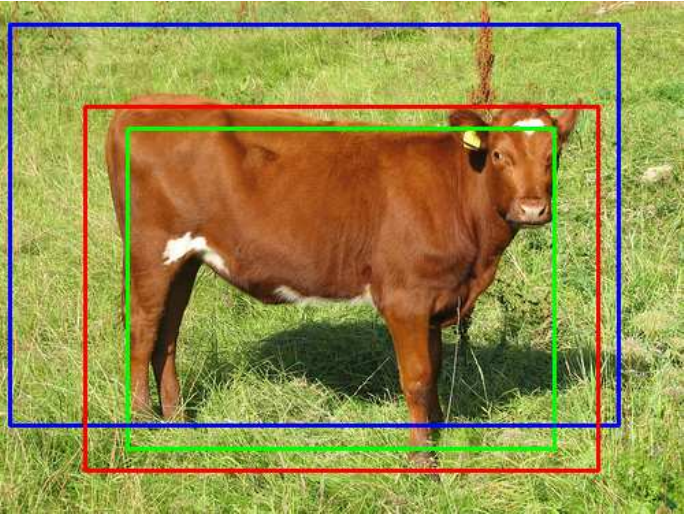} & \includegraphics[scale=0.2]{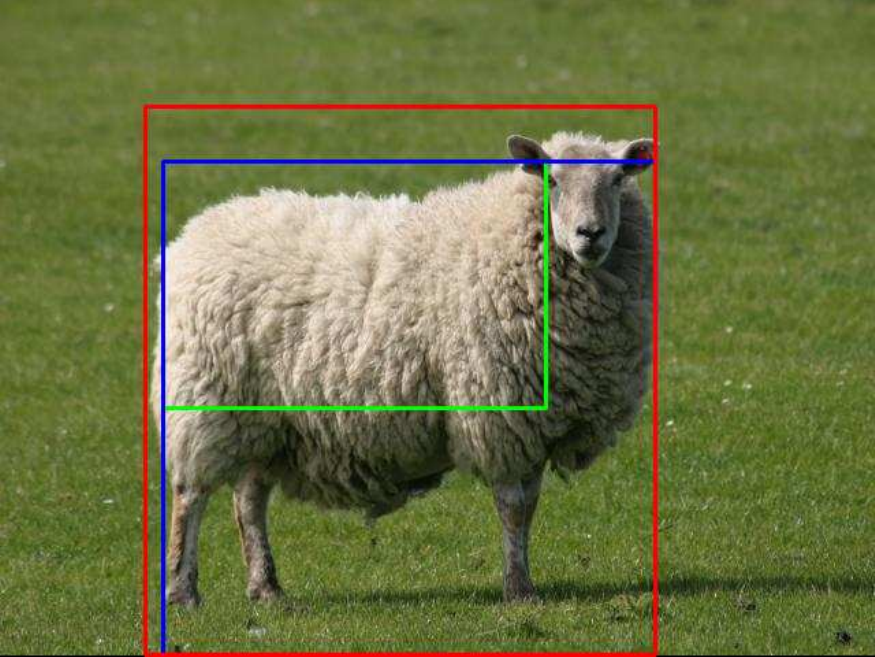} & & \includegraphics[scale=0.2]{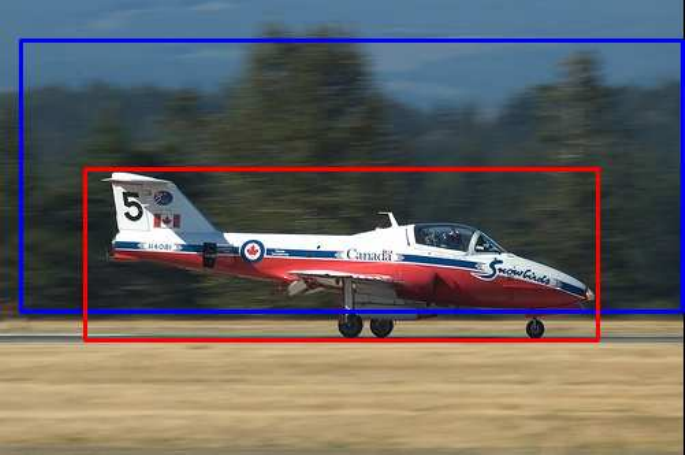} & \includegraphics[scale=0.2]{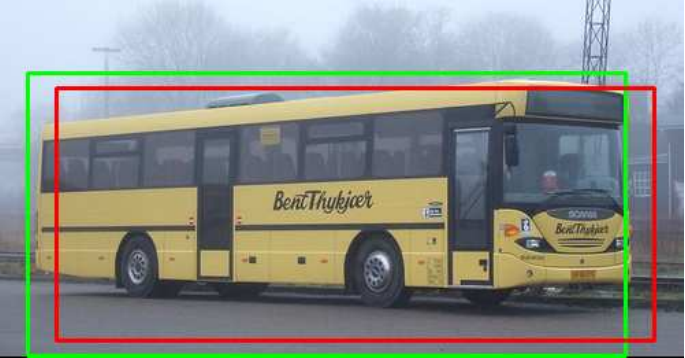} & \includegraphics[scale=0.2]{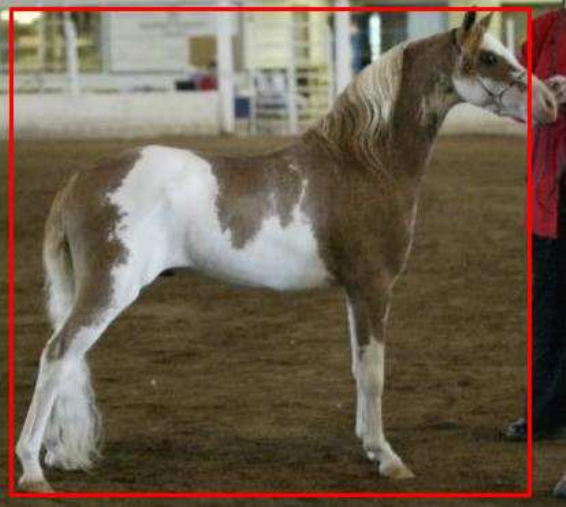} & \includegraphics[scale=0.2]{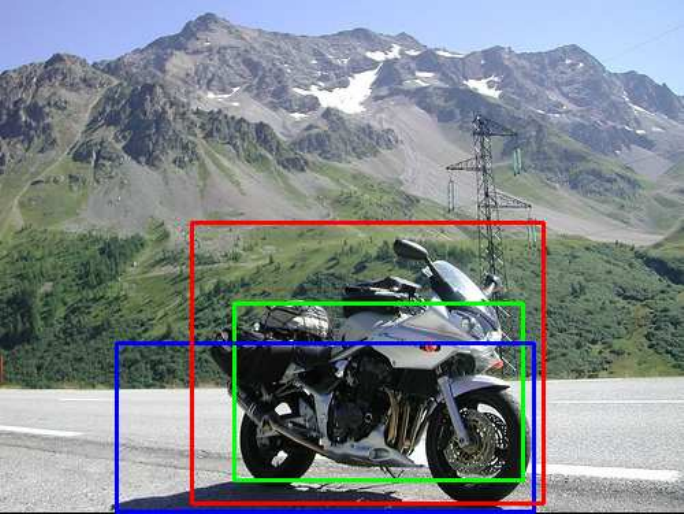} \\

\includegraphics[scale=0.2]{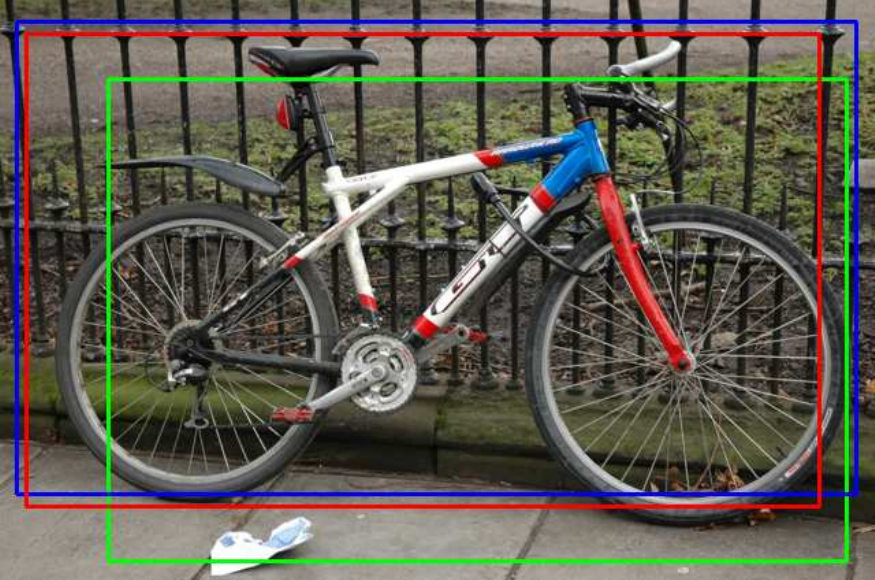} & \includegraphics[scale=0.2]{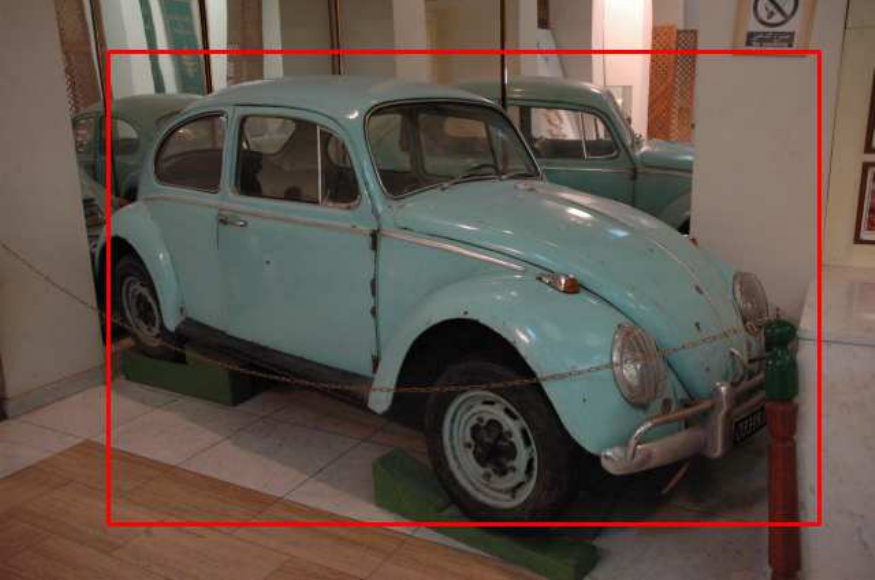} & \includegraphics[scale=0.2]{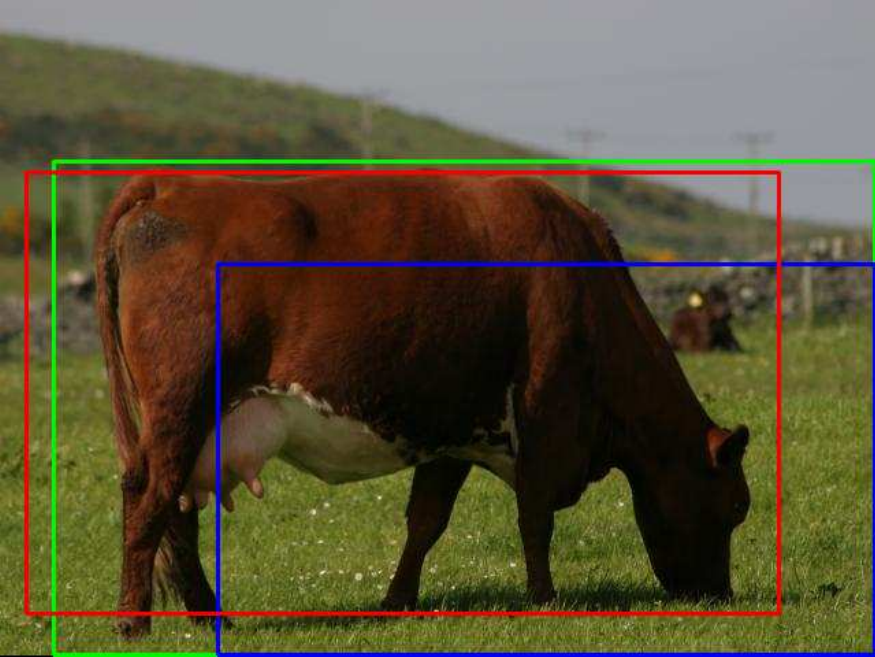} & \includegraphics[scale=0.2]{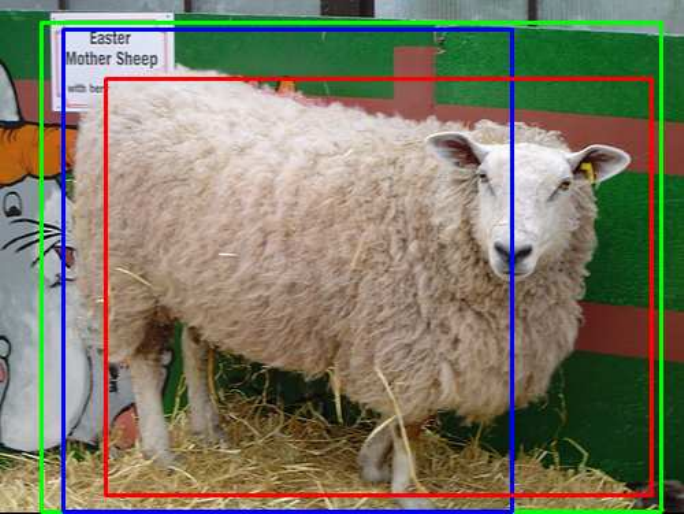} & & \includegraphics[scale=0.2]{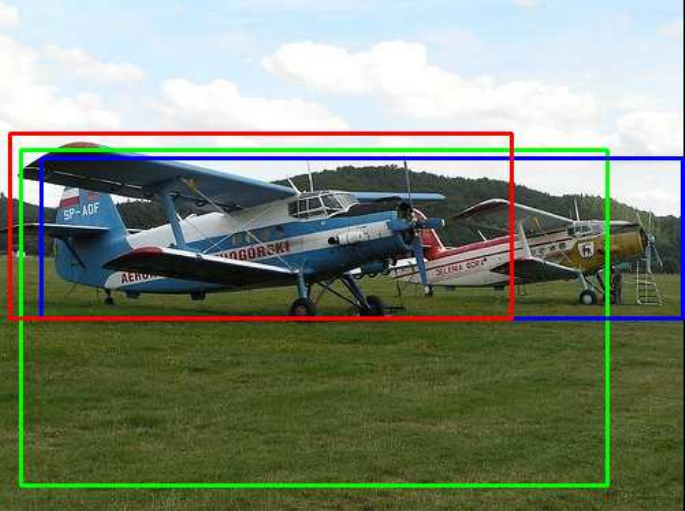} & \includegraphics[scale=0.2]{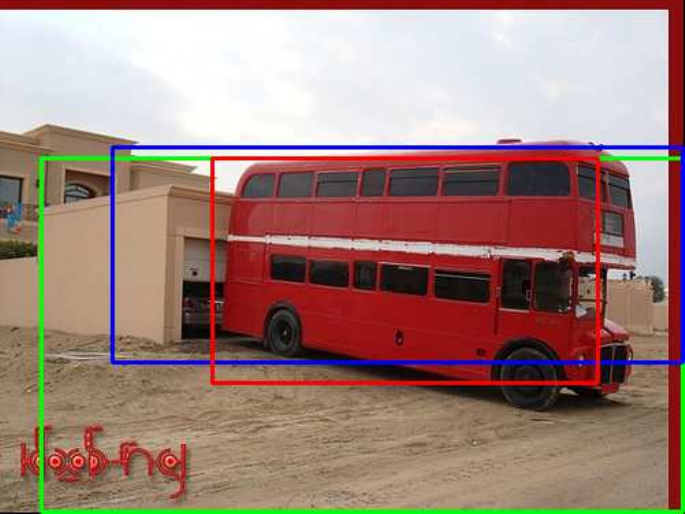} & \includegraphics[scale=0.2]{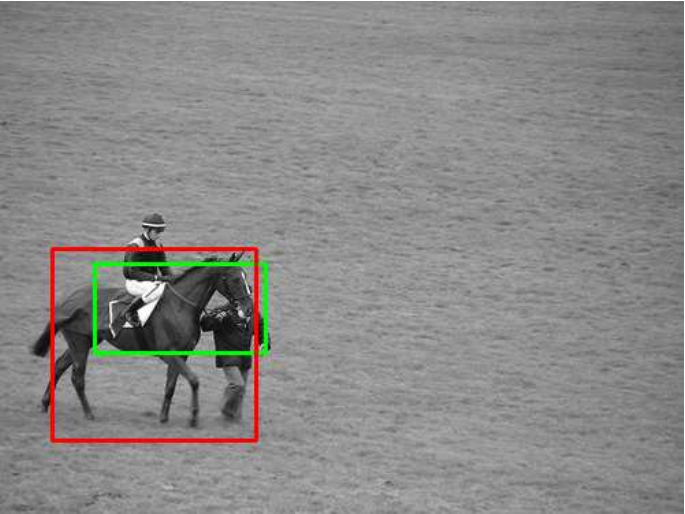} & \includegraphics[scale=0.2]{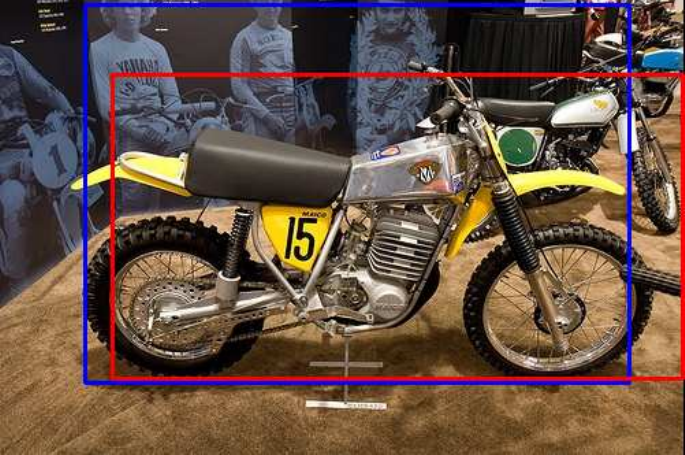} \\

\includegraphics[scale=0.2]{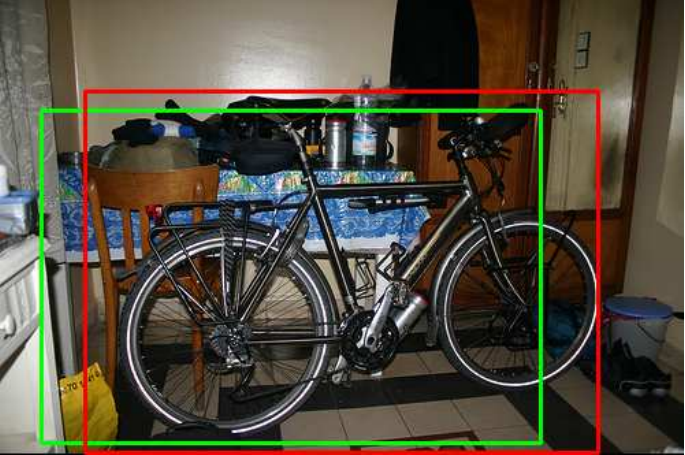} & \includegraphics[scale=0.2]{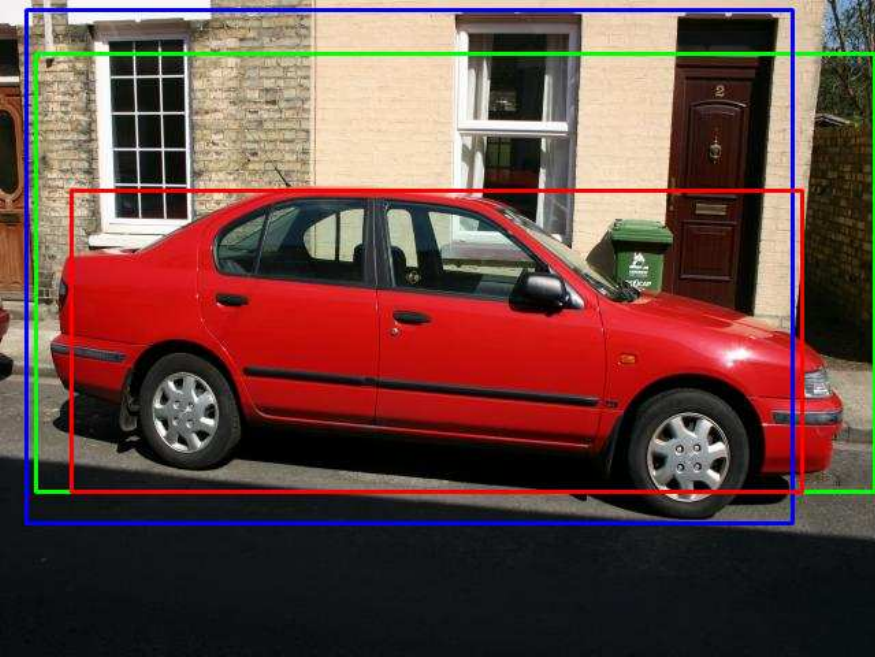} & \includegraphics[scale=0.2]{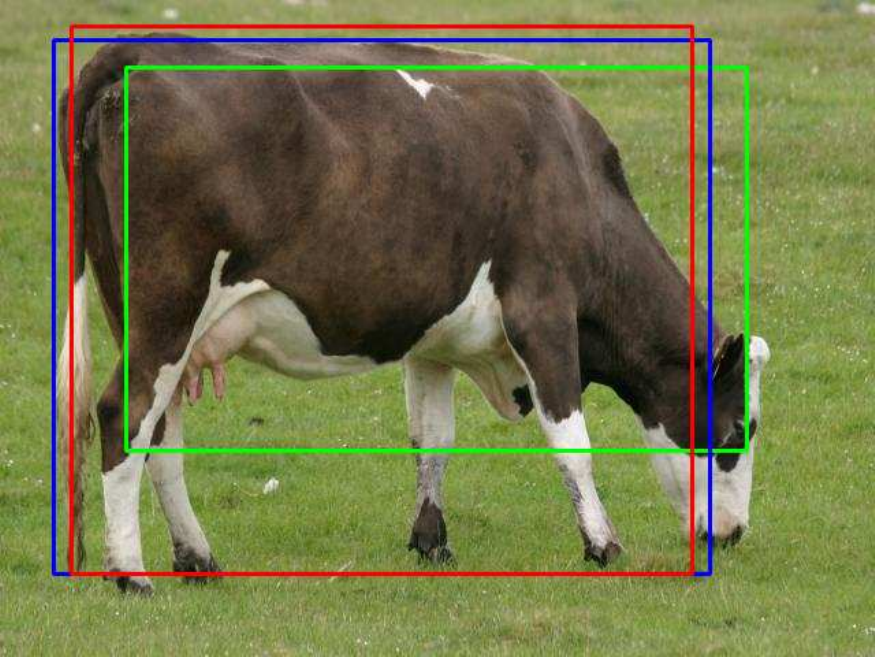} & \includegraphics[scale=0.2]{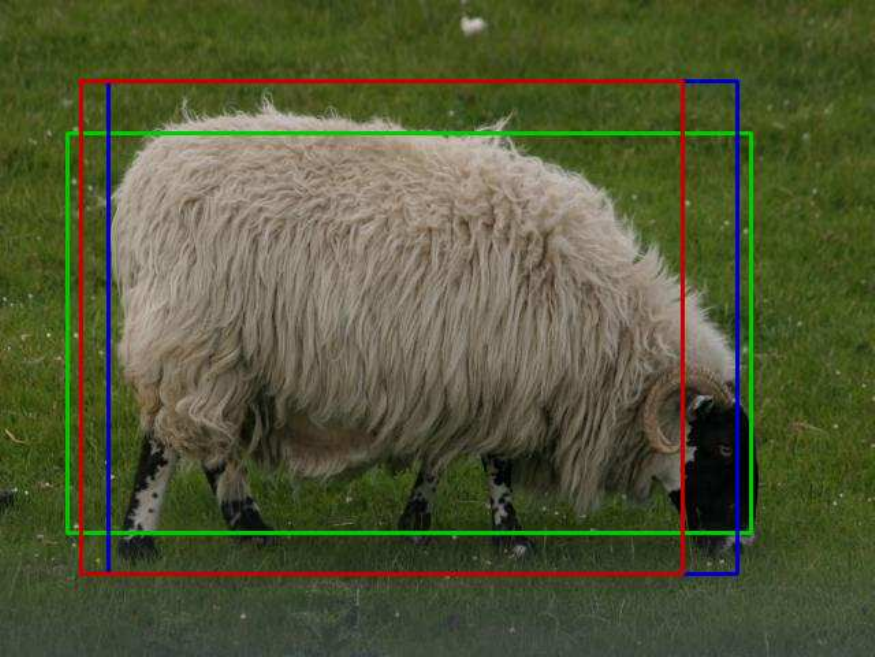} & & \includegraphics[scale=0.2]{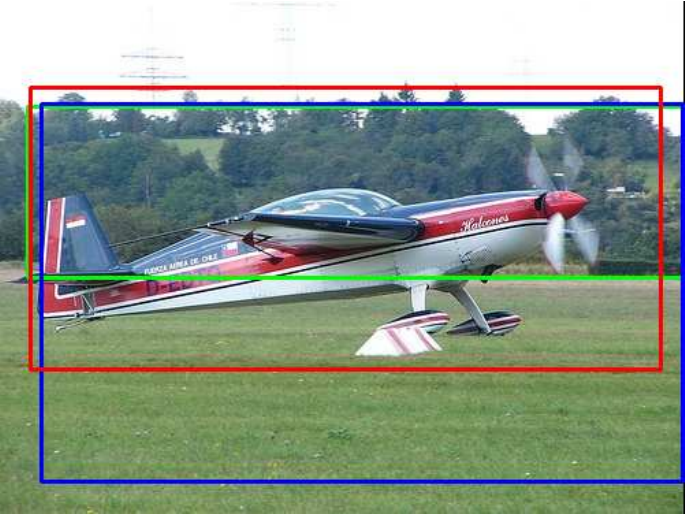} & \includegraphics[scale=0.2]{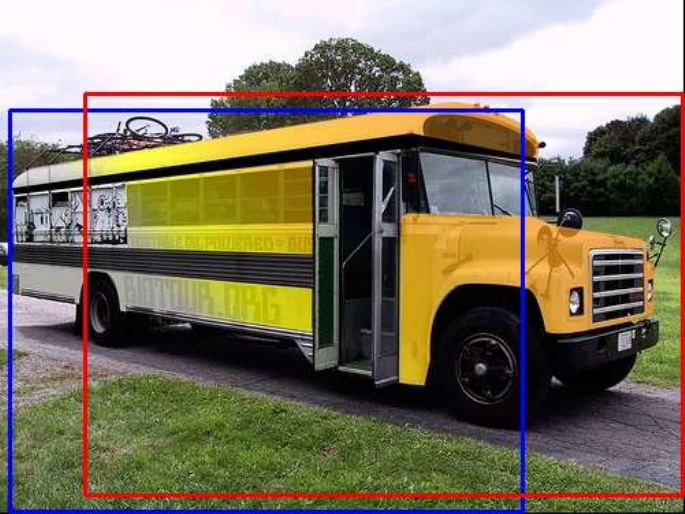} & \includegraphics[scale=0.2]{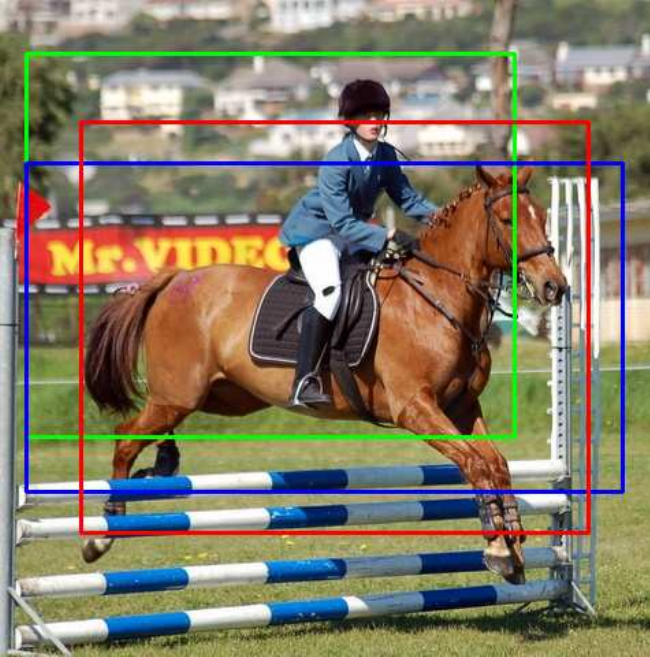} & \includegraphics[scale=0.2]{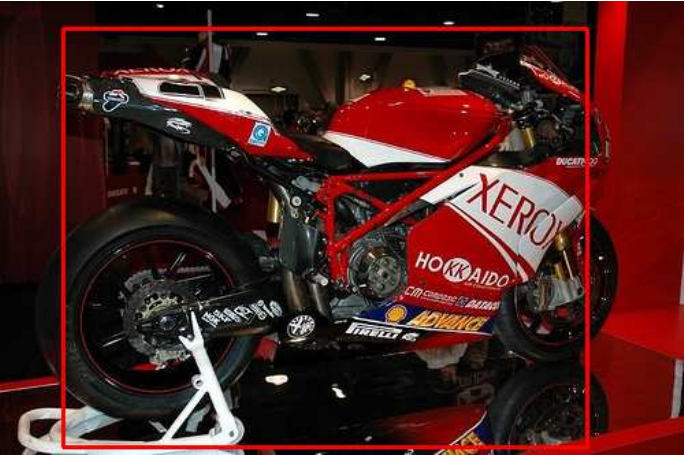} \\

\end{tabular}

}

\caption{ {\small Example images with estimated bounding boxes of different class-viewpoint combinations from the PASCAL06 and PASCAL07 datasets. In each image, results of Objectness \cite{Objectness} with the highest objectness score (green box),  and ROML in unsupervised (blue box) and weakly supervised (red box) settings are shown (they may coincide in some images where only one or two boxes are shown). Top part is for left viewpoint, and bottom part is for right viewpoint.} }\label{COLExamImgShowFigs}
\end{figure*}

\section{Conclusions}
\label{ConclusionSec}

In this paper, we propose a framework termed ROML, for robustly matching objects in a set of images. ROML is formulated as a rank and sparsity minimization problem to optimize a set of PPMs. The optimized PPMs identify inlier features from each image and establish their consistent correspondences across the image set. To solve ROML, we use the ADMM method, in which a subproblem associated with PPM updating is a difficult IQP. We prove that under widely applicable conditions, this IQP is equivalent to a formulation of LSAP, which can be efficiently solved by the Hungarian algorithm. Extensive experiments on rigid/non-rigid object matching, matching instances of a common object category, and common object localization show the efficacy of our proposed method.

In the present work, we have assumed for ROML that there is exactly one object instance contained in each of a set of images. This assumption is mainly to make an ideal problem setting for the difficult, combinatorial task of simultaneously matching inlier features of object instances across a set of images. However, in many practical problems such as unsupervised learning of object categories from an image collection, there could be multiple object instances, possibly of different categories, contained in one image. In these more challenging scenarios, ROML, by design, may at most identify and match one instance from each image, and ignore the other instances contained in the image collection. In order to learn object categories in these more challenging scenarios, one may need to extend the formulation of ROML so that multiple instances per image can be taken into account. For example, we have made an attempt of such kind in \cite{ZengAmbiguousLearning} for the task of learning object (face) categories from ambiguously labeled images, where each image in an image collection may contain multiple object instances of interest, and its associated caption has some labels of object category, with the true ones included, while the instance-label association is unknown. We take into account multiple instances per image in \cite{ZengAmbiguousLearning} by extending ROML to accommodate category-wise low-rank models and new constraints of PPMs. Nevertheless, extending ROML to unsupervised object learning remains an open question, and we are interested in pursuing this direction in future research.



\appendix

\section{Learning Features of Coordinates-Descriptor Combination}
\label{appendix_LowDimEmbedFeaLearning}

Local region descriptors alone could be ambiguous for feature matching when there exist repetitive textures or less discriminative local appearance in images. To improve the matching accuracy, it is necessary to exploit the geometric structure of inlier points that consistently appears in each of the set of images. In this work, we consider a simple method introduced in \cite{OneShot} to exploit such geometric constraints. The method derives an embedded feature representation that combines information of both spatial arrangement of feature points inside each image, and similarity of feature descriptors across images. We briefly summarize this method as follows.

Given a set of $K$ images, denote ${\bf A}_{spa.}^{k} \in \mathbb{R}^{n_{k}\times n_{k}}$ as an affinity matrix that measures the spatial proximity of any two of the $n_{k}$ extracted feature points in the $k^{th}$ image, where spatial proximity can be either measured based on Euclidean distances of image coordinates of feature points, which is invariant to translation and rotation, or made affine invariant \cite{OneShot}. In this work, we compute ${\bf A}_{spa.}^{k}$ using Gaussian kernel as ${\bf A}_{spa.}^{k} (i, j) = \mathrm{e}^{ - \| {\bf x}_{i}^{k} - {\bf x}_{j}^{k} \|^{2} / 2\sigma_{spa.}^{2} }$, where ${\bf x}^{k} = [x^{k}, y^{k}]^{\top}$ denotes image coordinates in the $k^{th}$ image, and $\sigma_{spa.}$ is a scaling parameter. Each feature point has an associated region descriptor. Denote ${\bf A}_{des.}^{pq} \in \mathbb{R}^{n_{p}\times n_{q}}$ as another affinity matrix, each entry of which measures the similarity of region descriptors between a pair of features selected from the $p^{th}$ and $q^{th}$ images respectively. ${\bf A}_{des.}^{pq}$ can be computed similar to ${\bf A}_{spa.}^{k}$ as ${\bf A}_{des.}^{pq} (i, j) = \mathrm{e}^{ - \| {\bf f}_{i}^{p} - {\bf f}_{j}^{q} \|^{2} / 2\sigma_{des.}^{2} }$, where ${\bf f}_{i}^{p}$ and ${\bf f}_{j}^{q}$ are feature descriptors from the $p^{th}$ and $q^{th}$ images respectively, and $\sigma_{des.}$ is a scaling parameter.

The method in \cite{OneShot} aims to learn embedded feature representations for all $N = \sum_{k=1}^{K}n_{k}$ points in the $K$ images so that in the embedded space: (1) spatial structure of the point set in each image should be preserved; (2) features from different images with high descriptor similarity should be close to each other. Let $\{ {\bf f}_{i}^{k} \in \mathbb{R}^{d} \}_{i=1}^{n_{k}}$, $k = 1, \dots, K$, be the new features to be learned \footnote{For consistency we use the same ${\bf f}$ for different feature types.}, the above objectives can be formalized as \begin{equation}\label{EqnOneShotEmbedObjective} \min\sum_{p,q}\sum_{i,j} \big\| {\bf f}_{i}^{p} - {\bf f}_{j}^{q} \big\|_{2}^{2} {\bf A}_{ij}^{pq} ,\end{equation} where the matrix ${\bf A} \in \mathbb{R}^{N\times N}$ is defined as: ${\bf A}^{pq} = {\bf A}_{spa.}^{k}$ when $p = q = k$, ${\bf A}^{pq} = {\bf A}_{des.}^{pq}$ when $p \neq q$, and ${\bf A}^{pq} \in \mathbb{R}^{n_{p}\times n_{q}}$ is the $(p,q)$ block of all the $K\times K$ blocks of ${\bf A}$. The objective function (\ref{EqnOneShotEmbedObjective}) turns to be a problem of Laplacian embedding \cite{LaplacianEmbed}. Let $\tilde{{\bf F}} = [{\bf f}_{1}^{1}, \dots, {\bf f}_{n_{1}}^{1}, \dots, {\bf f}_{1}^{K}, \dots, {\bf f}_{n_{K}}^{K}]^{\top} \in \mathbb{R}^{N\times d}$, (\ref{EqnOneShotEmbedObjective}) can be rewritten in matrix form as \begin{equation}\label{EqnOneShotEmbedObjMatrixForm} \min_{\tilde{{\bf F}}} \mathrm{trace}\big( \tilde{{\bf F}}^{\top} \tilde{{\bf L}}_{{\bf A}} \tilde{{\bf F}} \big) \ \mathrm{s.t.} \ \tilde{{\bf F}}^{\top} \tilde{{\bf D}}_{{\bf A}} \tilde{{\bf F}} = {\bf I}, \end{equation} where $\tilde{{\bf L}}_{{\bf A}} = \tilde{{\bf D}}_{{\bf A}} - {\bf A}$ is the Laplacian matrix of ${\bf A}$, and $\tilde{{\bf D}}_{{\bf A}}$ is a diagonal matrix with value of the $i^{th}$ diagonal entry as $\sum_{j}{\bf A}_{ij}$. (\ref{EqnOneShotEmbedObjMatrixForm}) is a generalized eigenvector problem: $\tilde{{\bf L}}_{{\bf A}} {\bf f} = \beta \tilde{{\bf D}}_{{\bf A}} {\bf f}$. Its optimal solution, i.e., the $N$ new features in the $d$-dimensional embedded space, can be obtained by the bottom $d$ nonzero eigenvectors.

\section{}
\label{appendix_LESolutionDerivation}

We present derivations of the solutions (\ref{EqnMainLagUpdateLSolution}) and (\ref{EqnMainLagUpdateESolution}) respectively for the problems (\ref{EqnMainLagUpdateL}) and (\ref{EqnMainLagUpdateE}) as follows.

Given updated variables $\mathbf{E}_t$, $\{ \mathbf{P}_t^k \}_{k=1}^K$, and $\mathbf{Y}_t$, and write ${\bf D}_{t} = [ \mathrm{vec}({\bf F}^{1}{\bf P}_{t}^{1}), \dots, \mathrm{vec}({\bf F}^{K}{\bf P}_{t}^{K}) ]$, the problem (\ref{EqnMainLagUpdateL}) can be written explicitly as
\begin{eqnarray}\label{AppendixEqnDerivationL}
\min_{\mathbf{L}} \| \mathbf{L} \|_* + \frac{\rho}{2} \| \mathbf{L} - ( \mathbf{D}_t - \mathbf{E}_t - \frac{1}{\rho}\mathbf{Y}_t ) \|_F^2 ,
\end{eqnarray}
which appears to be the form of a proximal operator associated with the nuclear norm. According to \cite{ZhouchenALM}, optimal solution of (\ref{AppendixEqnDerivationL}) can be written as (\ref{EqnMainLagUpdateLSolution}).

Given updated variables $\mathbf{L}_{t+1}$, $\mathbf{Y}_t$, and $\{ \mathbf{P}_t^k \}_{k=1}^K$ with ${\bf D}_{t} = [ \mathrm{vec}({\bf F}^{1}{\bf P}_{t}^{1}), \dots, \mathrm{vec}({\bf F}^{K}{\bf P}_{t}^{K}) ]$, the problem (\ref{EqnMainLagUpdateE}) can be written explicitly as
\begin{eqnarray}\label{AppendixEqnDerivationE}
\min_{\mathbf{E}} \lambda \| \mathbf{E} \|_1 + \frac{\rho}{2} \| \mathbf{E} - ( \mathbf{D}_t - \mathbf{L}_{t+1} - \frac{1}{\rho}\mathbf{Y}_t ) \|_F^2 ,
\end{eqnarray}
which appears to be the form of a proximal operator associated with the $\ell_1$-norm. According to \cite{ZhouchenALM}, optimal solution of (\ref{AppendixEqnDerivationE}) can be written as (\ref{EqnMainLagUpdateESolution}).

\bibliographystyle{spmpsci}      
\bibliography{Jia_RobustMatching_IJCVRevision}   

\end{document}